%% file: MonotonicityRatio.tex
\documentclass[11pt]{article}
\usepackage{etoolbox}
\usepackage{graphicx}
\usepackage{latexsym}
\usepackage{amssymb}
\usepackage{amsmath}
\usepackage{amsthm}
\usepackage{thmtools}
\usepackage{caption}
\usepackage{subcaption}
\usepackage{mwe}
\usepackage{derivative}
\usepackage{forloop}
\usepackage[paper=letterpaper,margin=1in]{geometry}
\usepackage[ruled,vlined,linesnumbered]{algorithm2e}
\usepackage{paralist}
\usepackage{xcolor}
\usepackage{nicefrac}
\usepackage{pgfplots}
\usepackage{siunitx}
\usepackage{hyperref}
\usepackage{tikz}

\usepgfplotslibrary{fillbetween}
\pgfplotsset{compat=1.17}

\definecolor{darkgreen}{rgb}{0,0.5,0}
\hypersetup{
    unicode=false,            
    colorlinks=true,          
    linkcolor=blue!70!black,  
    citecolor=darkgreen,      
    filecolor=magenta,        
    urlcolor=blue!70!black    
}

\newtheorem{theorem}{Theorem}[section]
\newtheorem{lemma}[theorem]{Lemma}

\newtheorem{corollary}[theorem]{Corollary}
\newtheorem{definition}{Definition}[section]
\newtheorem{proposition}[theorem]{Proposition}
\newtheorem{observation}[theorem]{Observation}

\newcommand{\defcal}[1]{\expandafter\newcommand\csname c#1\endcsname{{\mathcal{#1}}}}
\newcommand{\defbb}[1]{\expandafter\newcommand\csname b#1\endcsname{{\mathbb{#1}}}}
\newcommand{\defvec}[1]{\expandafter\newcommand\csname v#1\endcsname{{\mathbf{#1}}}}
\newcommand{\defmat}[1]{\expandafter\newcommand\csname m#1\endcsname{{\mathbf{#1}}}}
\newcounter{calBbCounter}
\forLoop{1}{26}{calBbCounter}{
    \edef\Letter{\Alph{calBbCounter}}
		\edef\letter{\alph{calBbCounter}}
    \expandafter\defcal\Letter
		\expandafter\defbb\Letter
		\expandafter\defvec\letter
		\expandafter\defmat\Letter
}

\newcommand{\eps}{\varepsilon}
\newcommand{\nnR}{{\bR_{\geq 0}}}
\newcommand{\nnRE}[1]{{\bR_{\geq 0}^{#1}}}
\newcommand{\characteristic}{{\mathbf{1}}}
\newcommand{\RSet}{{\mathtt{R}}}

\DeclareMathOperator*{\argmax}{\arg\max}
\DeclareMathOperator*{\argmin}{\arg\min}
\newcommand{\vzero}{\bar{0}}
\newcommand{\vone}{\bar{1}}
\newcommand{\quadprogIP}{{\textsc{quadprogIP}}}

\author{
	Loay Mualem\thanks{Computer Science Department, University of Haifa. Email: \href{mailto:loaymua@gmail.com}{loaymua@gmail.com}}
\and
	Moran Feldman\thanks{Computer Science Department, University of Haifa. Email: \href{mailto:moranfe@cs.haifa.ac.il}{moranfe@cs.haifa.ac.il}}
}
\title{Using Partial Monotonicity in Submodular Maximization}

\begin{document}

\maketitle

\input{abstract}

\pagenumbering{arabic}

\input{Loay_Introduction}
\input{Preliminaries}
\input{Unconstrained}
\input{Cardinality}
\input{Matroid}
\input{Experiments}
\input{Conclusion}

\appendix

\input{SymmetryGap}
\input{Experiments_Appendix}
\input{DR-submodular}

\bibliographystyle{plain}
\bibliography{MonotonicityRatio}

\end{document}

%% file: Abstract.tex
\begin{abstract}
Over the last two decades, submodular function maximization has been the workhorse of many discrete optimization problems in machine learning applications. Traditionally, the study of submodular functions was based on \emph{binary} function properties. However, such properties have an inherit weakness, namely, if an algorithm assumes functions that have a particular property, then it provides no guarantee for functions that violate this property, even when the violation is very slight. Therefore, recent works began to consider \emph{continuous} versions of function properties. Probably the most significant among these (so far) are the submodularity ratio and the curvature, which were studied extensively together and separately.

The monotonicity property of set functions plays a central role in submodular maximization. Nevertheless, and despite all the above works, no continuous version of this property has been suggested to date (as far as we know). This is unfortunate since submoduar functions that are almost monotone often arise in machine learning applications.
%
In this work we fill this gap by defining the \emph{monotonicity ratio}, which is a continues version of the monotonicity property. We then show that for many standard submodular maximization algorithms one can prove new approximation guarantees that depend on the monotonicity ratio; leading to improved approximation ratios for the common machine learning applications of movie recommendation, quadratic programming and image summarization.

%
%

\medskip

\end{abstract}

%% file: Loay_Introduction.tex
\section{Introduction}
Over the last two decades, submodular function maximization has been the workhorse of many discrete optimization problems in machine learning applications such as data summarization~\cite{dasgupta2013summarization,elhamifar2017online,kazemi2021regularized,kirchhoff2014submodularity,mitrovic2018data,tschiatschek2014learning}, social graph analysis~\cite{norouzi2018beyond}, adversarial attacks~\cite{lei2019discrete}, dictionary learning~\cite{das2011submodular}, sequence selection~\cite{mitrovic2019adaptive,tschiatschek2017selecting}, interpreting neural networks~\cite{elenberg2017streaming} and many more. 
Traditionally, the study of submodular functions was based on \emph{binary} properties of functions. A function can be either submodular or non-submodular, monotone or non-monotone, etc. Such properties are simple, but they have an inherit weakness---if an algorithm assumes functions that have a particular property, then it provides no guarantee for functions that violate this property, even if the violation is very slight.

Given the above situation, recent works began to consider \emph{continuous} versions of function properties. Probably the most significant among these continuous versions so far are the submodularity ratio and the curvature. The submodularity ratio (originally defined by Das and Kempe~\cite{das2019approximate}) is a parameter $\gamma \in [0, 1]$ replacing the binary submodularity property that a set function can either have or not have. A value of $1$ corresponds to a fully submodular function, and lower values of $\gamma$ represent some violation of submodularity (the worse the violation, the lower $\gamma$).  Similarly, the curvature (defined by Conforti and Cornu{\'{e}}jol~\cite{conforti1984submodular}) is a parameter $c \in [0, 1]$ replacing the binary linearity property that a set function can either have or not have. A value of $1$ corresponds to a fully linear function, and lower values of $c$ represent some violation of linearity.

A central conceptual contribution of Das and Kempe~\cite{das2019approximate} was that they were able to demonstrate that continuous function properties further extend the usefulness of submodular maximization to new machine learning applications (such as subset selection for regression and dictionary selection). This has motivated a long list of works on such properties (see~\cite{bian2017guarantees,ghadiri2020parametereized,ghadiri2021beyond,iyer2013curvature,kuhnle2018fast} for a few examples), including works that combine both the submodularity ratio and the curvature (see, e.g.,~\cite{bian2017guarantees}). However, to the best of our knowledge, no continuous version of the binary monotonicity property has been suggested so far.\footnote{Following the appearance of the pre-print version of this paper, we learned that Iyer defined in his Ph.D. thesis~\cite{iyer2015submodular} such a property, and moreover, this property is identical in name and definition to the one we define. However, Iyer only used this property to prove the result that appears below as Theorem~\ref{thm:greedy}; and thus, our work is the first to systematically study this property.}

We note that the monotonicity property of set functions plays a central role in submodular maximization, and basically every problem in this field has been studied for both monotone and non-monotone objective functions. Naturally, monotone objective functions enjoy improved approximation guarantees compared to general functions, and it is natural to ask how much of this improvement applies also to functions that are almost monotone (in some sense). Since such functions often arise in machine learning applications
when a diversity promoting component is added to a basic
monotone objective, obtaining better guarantees for them
should strongly enhance the usefulness of submodular maximization
as a tool for many machine learning applications.

Formally, a non-negative set function $f\colon 2^\cN \to \nnR$ over a ground set $\cN$ is (increasingly) \emph{monotone} if $f(S) \subseteq f(T)$ for every two sets $S \subseteq T \subseteq \cN$. Similarly, we define the \emph{monotonicity ratio} of such a function $f$ as the maximum value $m \in [0, 1]$ such that $m \cdot f(S) \leq f(T)$ for every two sets $S \subseteq T \subseteq \cN$. Equivalently, one can define the monotonicity ratio $m$ by
\[
	m \triangleq \min_{S \subseteq T \subseteq \cN} \frac{f(T)}{f(S)}
	\enspace,
\]
where the ratio $f(T) / f(S)$ is assumed to be $1$ whenever $f(S) = 0$. Intuitively, the monotonicity ratio measures how much of the value of a set $S$ can be lost when additional elements are added to $S$. One can view $m$ as a measure of the distance of $f$ from being monotone. In particular, $m$ is equal to $1$ if and only if $f$ is monotone.

Our main contribution in this paper is demonstrating the usefulness of the monotonicity ratio in machine learning applications, which we do in two steps.
\begin{compactitem}
	\item First, we show (in Sections~\ref{sec:unconstrained}, \ref{sec:cardinality} and~\ref{sec:matroid}) that for many standard submodular maximization algorithms one can prove new approximation guarantees that depend on the monotonicity ratio. These approximation guarantees interpolate between the known approximation ratios of these algorithms for monotone and non-monotone submodular functions.
	\item Then, using the above new approximation guarantees, we derive new approximation ratios for the standard applications of movie recommendation, quadratic programming and image summarization. Our guarantees improve over the state-of-the-art for most values of the problems' parameters. See Section~\ref{sec:experiments} for more detail.
\end{compactitem}

\paragraph{Remark.} Computing the monotonicity ratio $m$ of a given function seems to be a difficult task. Therefore, the algorithms we analyze avoid assuming access to $m$, and the value of $m$ is only used in the analyses of these algorithms. Nevertheless, in the context of particular applications, we are able to bound $m$, and plugging this bound into our general results yields our improved guarantees for these applications. 

\subsection{Our Results} \label{ssc:results}

Given a ground set $\cN$, a set function $f\colon 2^\cN \to \bR$ is submodular if $f(S \cup \{u\}) - f(S) \geq f(T \cup \{u\}) - f(T)$ for every two sets $S \subseteq T \subseteq \cN$ and element $u \in \cN \setminus T$. Submodular maximization problems ask to maximize such functions subject to various constraints. To allow for multiplicative approximation guarantees for these problems, it is usually assumed that the objective function $f$ is non-negative. Accordingly, we consider in this paper the following three basic problems.
\begin{compactitem}
	\item Given a non-negative submodular function $f\colon 2^\cN \to \bR$, find a set $S \subseteq \cN$ that (approximately) maximizes $f$. This problem is termed ``unconstrained submodular maximization'', and is studied in Section~\ref{sec:unconstrained}.
	\item Given a non-negative submodular function $f\colon 2^\cN \to \bR$ and an integer parameter $0 \leq k \leq |\cN|$, find a set $S \subseteq \cN$ of size at most $k$ that (approximately) maximizes $f$ among such sets. This problem is termed ``maximizing a submodular function subject to a cardinality constraint'', and is studied in Section~\ref{sec:cardinality}.
	\item Given a non-negative submodular function $f\colon 2^\cN \to \bR$ and a matroid $\cM$ over the same ground set, find a set $S \subseteq \cN$ that is independent in $\cM$ and (approximately) maximizes $f$ among such sets. This problem is termed ``maximizing a submodular function subject to a matroid constraint'', and is studied in Section~\ref{sec:matroid} (see Section~\ref{sec:matroid} also for the definition of matroids).
\end{compactitem}

We present both algorithmic and inapproximability results for the above problems. Our algorithmic results reanalyze a few standard algorithms, and are mostly proved using adaptations of the original analyses of these algorithms. However, as these adaptations are non-black box and often involve a technical challenge, it is quite surprising that for almost all algorithms we get an approximation ratio of $m \cdot \alpha_{\text{monotone}} + (1 - m) \cdot \alpha_{\text{non-monotone}}$, where $m$ is the monotonicity ratio, $\alpha_{\text{monotone}}$ is the approximation ratio known for the algorithm when $f$ is guaranteed to be monotone, and $\alpha_{\text{non-monotone}}$ is the approximation ratio known for the algorithm when $f$ is a general non-negative submodular function.

While the above mentioned algorithmic results lead to our improved guarantees for applications, our inapproximability results represent our main technical contribution. In general, these results are based on the symmetry gap framework of Vondr\'{a}k~\cite{vondrak2013symmetry}. The original version of this framework is able to deal both with the case of general (not necessarily monotone) submodular functions, and with the case of monotone submodular functions; which in our terms correspond to the cases of $m \geq 0$ and $m \geq 1$, respectively. However, to prove our inapproximability results, we had to show that the framework extends to arbitrary lower bounds on $m$, which was challenging because the original proof of the framework is highly based on derivatives of continuous functions. From this point of view, submodularity is defined as having non-positive second-order derivatives, and monotonicity is defined as having non-negative first-order derivatives. However, the definition of the monotonicity ratio cannot be easily restated in terms of derivatives; and thus, handling it required us to come up with a different proof approach.

Interestingly, our inapproximability result for unconstrained submodular maximization proves that the optimal approximation ratio for this problem does not exhibit a linear dependence on $m$. Thus, the nice linear dependence demonstrated by almost all our algorithmic results is probably an artifact of looking at standard algorithms rather than representing the true nature of the monotonicity ratio, and we expect future algorithms tailored to take advantage of the monotonicty ratio to improve over this linear dependence.
The reason that we concentrate in this work on reanalyzing standard submodular maximization algorithms rather than inventing new ones is that we want to stress the power obtained by using the new notion of monotonicity ratio, as opposed to power gained via new algorithmic innovations. This is in line with the research history of the submodularity ratio and the curvature. For both of these parameters, the original works concentrated on reanalyzed the standard greedy algorithm in view of the new suggested parameter; and the invention of algorithms tailored to the parameter was deferred to later works (see~\cite{sviridenko2017optimal} and~\cite{chen2018weakly} for examples of such algorithms for the curvature and submodularity ratio, respectively). 

Over the years, the standard submodular maximization algorithms have been extended and improved in various ways. Some works presented accelerated and/or parallelized versions of these algorithms, while other works generalized the algorithms beyond the realm of set functions (for example, to (DR-)submodular functions over lattices or continuous domains). Since our motivation in this paper is related to the monotonicity ratio, which is essentially independent of the extensions and improvements mentioned above, we mostly analyze the vanilla versions of all the algorithms considered. This keeps our analyses relatively simple. However, our experiments are based on more state-of-the-art versions of the algorithms. Similarly, many continuous properties (including the submodularity ratio) have weak versions that only depend on the behavior of the function for nearly feasible sets, and immediately enjoy most of the results that apply to the original strong property. The definition of such weak versions is useful for capturing additional application, but often add little from a theoretical perspective. Therefore, in the theoretical parts of this paper we consider only the monotonicity ratio as it is defined above; but for the sake of one of our applications we later define also the natural corresponding weak property.

\subsection{Additional Related Work}

Lin and Bilmes~\cite{lin2010multi} described an algorithm that takes advantages of a continuous partial monotonicity property, but unlike the monotonicity ratio, their property was defined in terms of the particular submodular objective they were interested in. More recently, Cui et al.~\cite{cui2021approximation} considered a weaker, but still binary, version of monotonicity called weak-monotonicity.

%% file: Preliminaries.tex
\section{Preliminaries and Basic Observations}

We begin this section by defining the notation that we use throughout the paper. Using this notation, we can then state some useful basic observations. Given an element $u \in \cN$ and a set $S \subseteq \cN$, we use $S + u $ and $S - u$ as shorthands for $S \cup \{u\}$ and $S \setminus \{u\}$. Additionally, given a set function $f\colon 2^\cN \to \bR$, we define $f(u \mid S) \triangleq f(S + u) - f(S)$ (this value is known as the marginal contribution of $u$ to $S$ with respect to $f$). Similarly, given an additional set $T \subseteq \cN$, we define $f(T \mid S) \triangleq f(S \cup T) - f(S)$. We also use $\characteristic_S$ to denote the characteristic vector of the set $S$, i.e., a vector in $[0, 1]^\cN$ that has $1$ in the coordinates corresponding to elements that appear in $S$ and $0$ in the rest of the coordinates. Finally, if $f$ is non-negative, then we say that it is $m$-monotone if its monotonicity ratio is at least $m$; and given an event $\cE$, we denote by $\characteristic[\cE]$ the indicator of this event, i.e., a random variable that takes the value $1$ when the event happens, and the value $0$ otherwise.

Next, we present two well-known continuous extensions of set functions. Given a set function $f\colon 2^\cN \to \bR$, its \emph{multilinear extension} is a function $F \colon [0, 1]^\cN \to \bR$ defined as follows. For every vector $\vx \in [0, 1]^\cN$, let $\RSet(\vx)$ to be a random subset of $\cN$ that includes every element $u \in \cN$ with probability $x_u$, independently. Then, $F(\vx) = \bE[f(\RSet(\vx))]$. The \emph{Lov\'{a}sz extension} of $f$ is a function $\hat{f}\colon [0, 1]^\cN \to \bR$ defined as follows. For every vector $\vx \in [0, 1]^\cN$,
\[
	\hat{f}(\vx)
	=
	\int_0^1 f(T_\lambda(\vx)) d\lambda
	\enspace,
\]
where $T_\lambda(\vx) \triangleq \{u \in \cN \mid x_u \geq \lambda\}$. The Lov\'{a}sz extension of a submodular function is known to be convex. More important for us is the following known lemma regarding this extension. This lemma stems from an equality, proved by Lov\'{a}sz~\cite{lovasz1983submodular}, between the Lov\'{a}sz extension of a submodular function and another extension known as the convex closure.
\begin{lemma} \label{lem:lovasz}
Let $f\colon 2^{\cN} \to \bR$ be a submodular function, and let $\hat{f}$ be its Lov\'{a}sz extension. For every $\vx \in [0, 1]^\cN$ and random set $D_{\vx} \subseteq \cN$ obeying $\Pr[u \in D_\vx] = x_u$ for every $u \in \cN$ (i.e., the marginals of $D_\vx$ agree with $\vx$), $\hat{f}(\vx) \leq \bE[f(D_\vx)]$.
\end{lemma}
As a consequence of the last lemma, we get the following useful corollary.\footnote{One can observe that Corollary~\ref{cor:sampling} is a generalization of the often used Lemma~2.2 of~\cite{buchbinder2014submodular}.}
\begin{corollary} \label{cor:sampling}
Let $f\colon 2^{\cN} \to \nnR$ be a non-negative $m$-monotone submodular function. For every deterministic set $O \subseteq \cN$ and random set $D \subseteq \cN$, $\bE[f(O \cup D)] \geq (1 - (1 - m) \cdot \max_{u \in \cN} \Pr[u \in D]) \cdot f(O)$.
\end{corollary}
\begin{proof}
Let $\vx$ be the vector of marginals of $O \cup D$, i.e., $x_u = \Pr[u \in O \cup D]$ for every $u \in \cN$. Then, by Lemma~\ref{lem:lovasz},
\begin{align*}
	\bE[f(O \cup D)]
	\geq{} &
	\hat{f}(\vx)
	=
	\int_0^1 f(T_\lambda(\vx)) d\lambda\\
	={} &
	\int_0^{\max_{u \in \cN} \Pr[u \in D]} f(T_\lambda(\vx)) d\lambda + \int_{\max_{u \in \cN} \Pr[u \in D]}^1 f(T_\lambda(\vx)) d\lambda\\
	={} &
	\int_0^{\max_{u \in \cN} \Pr[u \in D]} f(O \cup T_\lambda(\vx)) d\lambda + (1 - \max_{u \in \cN} \Pr[u \in D]) \cdot f(O)
	\enspace,
\end{align*}
where the last equality holds since the elements of $O$ appear in $T_\lambda(\vx)$ for every $\lambda \in [0, 1]$, and no other element appears in $T_\lambda(\vx)$ when $\lambda > \Pr[u \in D]$. Using the definition of the monotonicity ratio, the expression $f(O \cup T_\lambda(\vx))$ on the rightmost side of the previous equation can be lower bounded by $m \cdot f(O)$, which yields
\begin{align*}
	\bE[f(O \cup D)]
	\geq{} &
	\int_0^{\max_{u \in \cN} \Pr[u \in D]} m \cdot f(O) d\lambda + (1 - \max_{u \in \cN} \Pr[u \in D]) \cdot f(O)\\
	={} &
	m \cdot \max_{u \in \cN} \Pr[u \in D] \cdot f(O) + (1 - \max_{u \in \cN} \Pr[u \in D]) \cdot f(O)\\
	={} &
	(1 - (1 - m) \cdot \max_{u \in \cN} \Pr[u \in D]) \cdot f(O)
	\enspace.
	\qedhere
\end{align*}
\end{proof}

We conclude this section with the following observation, which immediately follows from the definition of the monotonicity ratio. We view this observation as evidence that the class of non-negative $m$-monotone functions is a natural class for every $m \in [0, 1]$, and not just for $m = 0$ (the class of all non-negative set functions) and $m = 1$ (the class of all non-negative monotone set functions).
\begin{observation} \label{obs:linearity}
For every two non-negative $m$-monotone functions $f, g \colon 2^\cN \to \nnR$ and constant $c \geq 0$, the following functions are also non-negative and $m$-monotone:
\begin{compactitem}
	\item $h(S) = f(S) + g(S)$,
	\item $h(S) = f(S) + c$, and
	\item $h(S) = c \cdot f(S)$.
\end{compactitem}
\end{observation}

%% file: Unconstrained.tex
\section{Unconstrained Maximization} \label{sec:unconstrained}

Recall that in the unconstrained submodular maximization problem, we are given a non-negative submodular function $f\colon 2^\cN \to \nnR$, and the objective is to find a set $S \subseteq \cN$ that (approximately) maximizes $f(S)$. Buchbinder et al.~\cite{buchbinder2015tight} gave the first $\nicefrac{1}{2}$-approximation algorithm for this problem, known as the double greedy algorithm. The $\nicefrac{1}{2}$-approximation guarantee of double greedy is known to be optimal in general due to a matching inapproximability result due to Feige et al.~\cite{feige2011maximizing}. Nevertheless, in this section we study the extent to which one can improve over this guarantee as a function of the monotonicity ratio $m$ of $f$. Formally, we prove the following theorem.

\begin{theorem} \label{thm:unconstrained}
There exists a polynomial time algorithm that obtains an approximation ratio of $\max\{m, (2 + m)/4\}$ for unconstrained submodular maximization, but no such algorithm obtains an approximation ratio of $1 / (2 - m) + \eps$ for any constant $\eps > 0$.\footnote{In the second part of Theorem~\ref{thm:unconstrained}, like in all the other inapproximability results in this paper, we make the standard assumption that the objective function $f$ can be accessed only through a value oracle that given a set $S \subseteq \cN$ returns $f(S)$.}
\end{theorem}

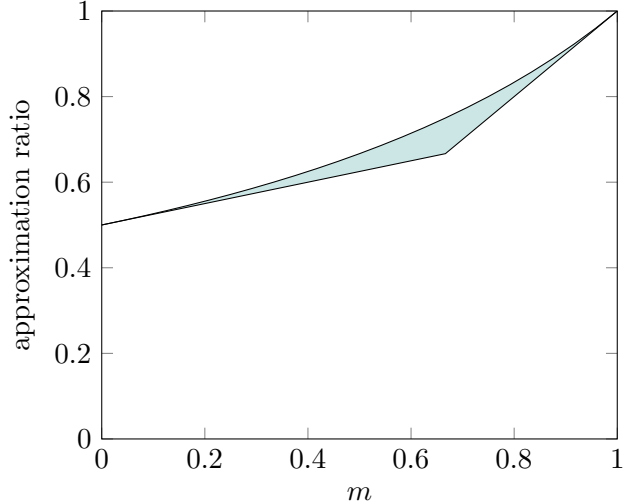
\begin{figure}
\begin{center}\input{UnconstrainedGraph.tikz}\end{center}
\caption{Graphical presentation of Theorem~\ref{thm:unconstrained}. The lower line represents the approximation guarantee (as a function of $m$) of the algorithm whose existence is guaranteed by the theorem; and the upper line represents the inapproximability result stated in the theorem. The shaded area between the lines includes the optimal approximation ratio (for any given $m$). Since this shaded area does not include a straight segment connecting the points $(0, 1/2)$ and $(1,1)$, the optimal approximation ratio does not have a linear dependence on $m$.} \label{fig:unconstrained}
\end{figure}

The guarantees stated in Theorem~\ref{thm:unconstrained} are plotted in Figure~\ref{fig:unconstrained}. This figure demonstrates that while Theorem~\ref{thm:unconstrained} does not settle the approximability of unconstrained submodular maximization as a function of $m$, the gap between its positive and negative results is quite small. More importantly, Figure~\ref{fig:unconstrained} also shows that the optimal approximation ratio for unconstrained submodular maximization cannot have a linear dependence on $m$.

The following simple proposition proves the first part of Theorem~\ref{thm:unconstrained}.
\begin{proposition}
There exists a polynomial time algorithm that obtains an approximation ratio of $\max\{m, (2 + m)/4\}$ for unconstrained submodular maximization.
\end{proposition}
\begin{proof}
The analysis of the double greedy algorithm by Buchinder et al.~\cite{buchbinder2015tight} shows that this algorithm outputs a solution set of expected value at least
\[
	\frac{2f(OPT) + f(\varnothing) + f(\cN)}{4}
	\enspace,
\]
where $OPT$ is an arbirary optimal solution. In general, this only shows $\nicefrac{1}{2}$-approximation, as is claimed by~\cite{buchbinder2015tight}. However, when the monotonicity ratio is taken into account,
\[
	\frac{2f(OPT) + f(\varnothing) + f(\cN)}{4}
	\geq
	\frac{2f(OPT) + f(\cN)}{4}
	\geq
	\frac{2f(OPT) + m \cdot f(OPT)}{4}
	=
	\frac{2 + m}{4} \cdot f(OPT)
	\enspace,
\]
where the first inequality follows from the non-negativity of $f$, and the second inequality holds since $OPT$ is a subset of $\cN$. Hence, the double greedy algorithm of Buchinder et al.~\cite{buchbinder2015tight} is a polynomial time algorithm guaranteeing $(2 + m)/4$-approximation.

To complete the proof of the proposition, we note that the trivial algorithm that always outputs the set $\cN$ has an approximation ratio of at least $m$ because $f(\cN) \geq m \cdot f(OPT)$. Hence, outputting the better solution among $\cN$ and the output of double greedy guarantees $\max\{m, (2 + m)/4\}$-approximation.
\end{proof}

We now would like to prove the second part of Theorem~\ref{thm:unconstrained}. We do this using a generalization of the symmetry gap framework of Vondr\'{a}k~\cite{vondrak2013symmetry} that is given as Theorem~\ref{thm:symmetry_gap}. To formally state Theorem~\ref{thm:symmetry_gap}, we first need to present some definitions from~\cite{vondrak2013symmetry}.

\begin{definition}[Strong symmetry]
Consider a non-negative submodular function $f$ and a collection $\cF \subseteq 2^\cN$ of feasible sets. The problem $\max\{f(S) \mid S \in \cF\}$ is strongly symmetric with respect to a group of permutations $\cG$ on $\cN$, if (1) $f(S) = f(\sigma(S))$ for all $S \subseteq \cN$ and $\sigma \in \cG$, and (2) $S \in \cF \iff S' \in \cF$ whenever $\bE_{\sigma \in \cG} [\characteristic_{\sigma(S)}] = \bE_{\sigma \in \cG} [\characteristic_{\sigma(S')}]$, where $\bE_{\sigma \in \cG}$ represents the expectation over picking $\sigma$ uniformly at random out of $\cG$.
\end{definition}

\begin{definition}[Symmetry gap]
Consider a non-negative submodular function $f$ and a collection $\cF \subseteq 2^\cN$ of feasible sets. Let $F(x)$ be the multilinear extension of $f$ and $P(\cF) \subseteq [0, 1]^\cN$ be the convex hull of $\cF$. Then, if the problem $\max\{f(S) \mid S \in \cF\}$ is strongly symmetric with respect to a graph $\cG$ of permutation, then its symmetry is defined as
\[
	\frac{\max\{F(\bar{\vx}) \mid \vx \in P(\cF)\}}{\max\{F(\vx) \mid \vx \in P(\cF)\}}
	\enspace,
\]
where $\bar{\vx} \triangleq \bE_{\sigma \in \cG}[\sigma(\vx)]$.
\end{definition}

\begin{definition}[Refinement]
Consider a set $\cF \subseteq 2^\cN$, and let $X$ be some set. We say that $\tilde{\cF} \subseteq 2^{\cN \times X}$ is a refinement
of $\cF$ if
\[
 \tilde{F}
	=
	\left\{S \subseteq \cN \times X ~\middle|~ \text{$\vx \in P(\cF)$, where $x_u = \tfrac{|S \cap (\{u\} \times X)|}{|X|}$ for all $u \in \cN$}\right\}
	\enspace.
\]
\end{definition}

\begin{restatable}{theorem}{thmSymmetryGap} \label{thm:symmetry_gap}
Consider a non-negative $m$-monotone submodular function $f$ and a collection $\cF \subseteq 2^\cN$ of feasible sets such that the problem $\max\{f(S) \mid S \in \cF\}$ is strongly symmetric with respect to some group $\cG$ of permutations over $\cN$ and has a symmetry gap $\gamma$. Let $\cC$ be the class of problems $\max\{\tilde{f}(S) \mid S \in \tilde{F}\}$ in which $\tilde{f}$ is a non-negative $m$-monotone submodular function, and $\tilde{F}$ is a refinement of $F$. Then, for every $\eps > 0$, any (even randomized) $(1 + \eps)\gamma$-approximation algorithm for the class $\cC$ would require exponentially many value queries to $\tilde{f}$.
\end{restatable}

While adapting the proof of Vondr\'{a}k~\cite{vondrak2013symmetry} to obtain Theorem~\ref{thm:symmetry_gap} is not trivial, as is discussed in Section~\ref{ssc:results}, the adaptation is quite technical and non-inspiring. Therefore, we defer the proof of Theorem~\ref{thm:symmetry_gap} to Appendix~\ref{app:symmetry_gap}.

To use Theorem~\ref{thm:symmetry_gap}, we need to define a submodular maximization problem with a significant symmetry gap. Specifically, let us choose $\cN = \{u, v\}$, $f(S) = m \cdot \characteristic[S \neq \varnothing] + (1 - m) \cdot (|S| \bmod 2)$ and $\cF = 2^\cN$, where $m$ is an arbitrary constant $m \in [0, 1]$. The function $f$ is submodular and non-negative since it a convex combination of the functions $\characteristic[S \neq \varnothing]$ and $|S| \bmod 2$ which are well-known to have these properties. One can also verify via the definitions that the monotonicity ratio of $f$ is exactly $m$, and that the problem $\max\{f(S) \mid S \in \cF\}$ is strongly symmetric with respect to the group $\cG$ of the two possible permutations of $\cN$. The following lemma calculates the symmetry gap of this problem.
\begin{lemma} \label{lem:symmetry_gap}
The problem $\max\{f(S) \mid S \in \cF\}$ has a symmetry gap of $\frac{1}{2 - m}$.
\end{lemma}
\begin{proof}
Observe that our definition of $\cF$ implies that $P(\cF) = [0, 1]^\cN$. Therefore,
\begin{equation} \label{eq:non_symmetric_maximum}
	\max\{F(\vx) \mid \vx \in P(\cF)\}
	=
	\max\{F(\vx) \mid \vx \in [0, 1]^\cN\}
	=
	\max\{f(S) \mid S \subseteq \cN\}
	=
	1
	\enspace,
\end{equation}
where the second equality holds since, for every vector $\vx$, $F(\vx)$ is a convex combination of values of $f$ for subsets of $\cN$; and on the other hand, for every set $S \subseteq \cN$, $f(S) = F(\characteristic_S)$.

Observe now that the definition of $f$ implies that
\begin{align*}
	F(\vx)
	={} &
	m[1 - (1 - x_u)(1 - x_v)] + (1 - m) \cdot [x_u(1 - x_v) + x_v(1 - x_u)]\\
	={} &
	x_u + x_v - x_u x_v(2 - m)
	\enspace.
\end{align*}
Since $\bar{\vx}$ is a vector that has the value $(x_u + x_v)/2$ in both its coordinates, if we we use the shorthand $y = (x_u + x_v)/2$, then we get
\[
	F(\bar{\vx})
	=
	2y - (2 - m)y^2
	\enspace.
\]
This expression is maximized for $y = 1 / (2 - m)$, and the maximum attained for this $y$ is
\[
	\frac{2}{2 - m} - \frac{(2 - m)}{(2 - m)^2}
	=
	\frac{1}{2 - m}
	\enspace.
\]
Since the value $y = 1 / (2 - m)$ is obtained, for example, when $\vx = (y, y) \in [0, 1]^\cN$, the above implies
\[
	\max\{F(\bar{\vx}) \mid \vx \in P(\cF)\}
	=
	\frac{1}{2 - m}
	\enspace.
\]
Together with Equation~\eqref{eq:non_symmetric_maximum}, this implies the lemma.
\end{proof}

The second part of Theorem~\ref{thm:unconstrained} now follows from Theorem~\ref{thm:symmetry_gap}, Lemma~\ref{lem:symmetry_gap} and the properties of $f$ discussed before the lemma because any refinement $\tilde{\cF}$ of $\cF$ is equal to the entire set $2^{\cN \times X}$ for some set $X$.

%% file: UnconstrainedGraph.tikz
\begin{tikzpicture}
\begin{axis}[
    xlabel = {$m$},
    ylabel = {approximation ratio},
    xmin=0, xmax=1,
    ymin=0, ymax=1]
 
\addplot [name path = A, domain = 0:1] {max(x, (2 + x)/4)};
 
\addplot [name path = B, domain = 0:1] {1 / (2 - x)};
 
\addplot [teal!20] fill between [of = A and B];
 
\end{axis}
\end{tikzpicture}

%% file: Cardinality.tex
\section{Maximization subject to a Cardinality Constraint} \label{sec:cardinality}

In this section we consider the problem of maximizing a non-negative submodular function $f\colon 2^\cN \to \nnR$ subject to a cardinality constraint. In other words, we are given an integer value $1 \leq k \leq |\cN|$, and the objective is to output a set $S \subseteq \cN$ of size at most $k$ (approximately) maximizing $f$ among such sets. When the objective function $f$ is guaranteed to be monotone, it is long known that a standard greedy algorithm (Algorithm~\ref{alg:greedy} below) guarantees $(1 - 1/e)$-approximation for the above problem~\cite{nemhauser1078analysis}, and that this is essentially the best possible for any polynomial time algorithm~\cite{nemhauser1978best}. Unfortunately, however, the greedy algorithm has no constant approximation guarantee when the objective function is not guaranteed to be monotone (see~\cite{buchbinder2018submodular} for a simple example demonstrating this). In Section~\ref{ssc:greedy} we prove Theorem~\ref{thm:greedy}, which generalizes the result of~\cite{nemhauser1078analysis}, and proves an approximation guarantee for the greedy algorithm that deteriorates gracefully from $1 - 1/e$ to $0$ as the monotonicity ratio $m$ goes from $1$ to $0$.

\begin{restatable}{theorem}{thmGreedy} \label{thm:greedy}
The Greedy algorithm (Algorithm~\ref{alg:greedy}) has an approximation ratio of at least $m(1 - 1/e)$ for the problem of maximizing a non-negative $m$-monotone submodular function subject to a cardinality constraint.
\end{restatable}

Following a long line of works~\cite{buchbinder2014submodular,ene2016constrained,feldman2011unified,lee2009non-monotone,oveisgharan2011submodular,vondrak2013symmetry}, the state-of-the-art approximation guarantee for the case in which the objective function $f$ is not guaranteed to be monotone is currently $0.385$~\cite{buchbinder2019constrained}. However, the algorithm obtaining this approximation ratio is quite involved, which limits its practicality. Arguably, the state-of-the-art approximation ratio obtained by a ``simple'' algorithm is the $1/e \approx 0.367$-approximation obtained by an algorithm called Random Greedy (Algorithm~\ref{alg:random_greedy} below). Furthermore, this algorithm has the nice property that for monotone objective functions it recovers the optimal $1 - 1/e$ approximation guarantee. In Section~\ref{ssc:random_greedy} we prove Theorem~\ref{thm:random_greedy}, which gives an approximation guarantee for Random Greedy that smoothly changes as a function of $m$ and recovers the above mentioned $1/e$ and $1 - 1/e$ guarantees in the cases of $m = 0$ and $m = 1$, respectively.

\begin{restatable}{theorem}{thmRandomGreedy} \label{thm:random_greedy}
The Random Greedy algorithm (Algorithm~\ref{alg:random_greedy}) has an approximation ratio of at least $m(1 - 1/e) + (1 - m) \cdot (1/e)$ for the problem of maximizing a non-negative $m$-monotone submodular function subject to a cardinality constraint.
\end{restatable}

As mentioned above, when $f$ is monotone, the optimal approximation ratio is known to be $1 - 1/e$. However, there is still gap between the state-of-the-art $0.385$-approximation for the non-monotone case and the state-of-the-art inapproximability result due to Oveis Gharan and Vondr\'{a}k~\cite{oveisgharan2011submodular}, which shows that no polynomial time algorithm can guarantee a better than roughly $0.491$-approximation. In Section~\ref{ssc:inapproximability_cardinality} we prove Theorem~\ref{thm:inapproximability_cardinality}, which proves an inapproximability result that smoothly depends on $m$ and recovers the above mentioned inapproximability results in the special cases of $m = 0$ and $m = 1$.

\begin{restatable}{theorem}{thmInapproximabilityCardinality} \label{thm:inapproximability_cardinality}
For any constant $\eps > 0$, no polynomial time algorithm can obtain an approximation ratio of
\[
	\min_{\alpha \in [0, 1]} \frac{\max_{x \in [0, 1]} \{\alpha(mx^2 + 2x - 2x^2) + 2(1 - \alpha)(1-e^{x - 1})(1 - (1 - m)x)\}}{\max\{1, 2(1 - \alpha)\}} + \eps
\]
for the problem of maximizing a non-negative $m$-monotone submodular function subject to a cardinality constraint.
\end{restatable}

\begin{figure}
\begin{center}\input{CardinalityGraph.tikz}\end{center}
\caption{Graphical representation of the results of Section~\ref{sec:cardinality}. This plots depicted the approximation guarantees we prove for the Greedy algorithm (Theorem~\ref{thm:greedy}) and the Random Greedy algorithm (Theorem~\ref{thm:random_greedy}), and our inapproximability result (Theorem~\ref{thm:inapproximability_cardinality}).} \label{fig:cardinality}
\end{figure}
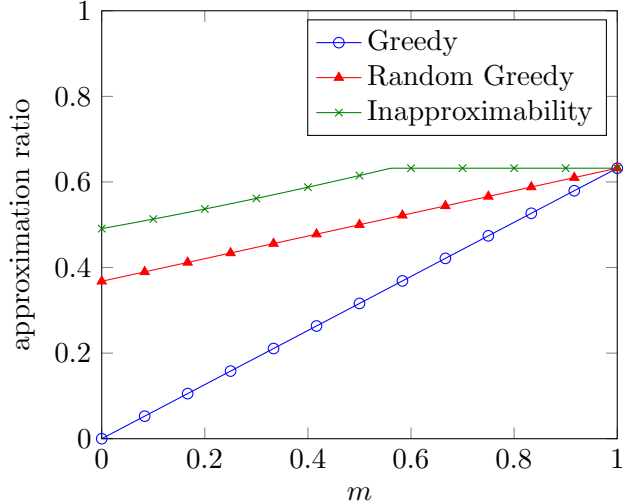

Unfortunately, the mathematical expression given in Theorem~\ref{thm:inapproximability_cardinality} is not very readable. To get an intuitive understanding of its behavior, we numerically evaluated it for various values of $m$. The plot obtained in this way appears in Figure~\ref{fig:cardinality}. For context, this figure also includes all the other results proved in this section. As is evident from Figure~\ref{fig:cardinality}, Theorem~\ref{thm:inapproximability_cardinality} improves over the $1 - 1/e$ inapproximability result of Nemhauser and Wolsey~\cite{nemhauser1978best} only for $m$ that is smaller than roughly $0.56$. This is surprising since, intuitively, one would expect the best possible approximation ratio to be strictly worse than $1 - 1/e$ for any $m < 1$. However, we were unable to prove an inapproximability that is even slightly lower than $1 - 1/e$ for any value $m > 0.56$. Understanding whether this is an artifact of our proof or a real phenomenon is an interesting question that we leave open.

\subsection{Analysis of the Greedy Algorithm} \label{ssc:greedy}

In this section we prove Theorem~\ref{thm:greedy}, which we repeat here for convenience.
\thmGreedy*

The greedy algorithm starts with an empty solution, and then augments this solution in $k$ iterations (recall that $k$ is the maximum cardinality allowed for a feasible solution). Specifically, in iteration $i$, the algorithm adds to the current solution the element $u_i$ with the best (largest) marginal contribution with respect to the current solution---but only if this addition does not decrease the value of the solution. A formal description of the greedy algorithm appears as Algorithm~\ref{alg:greedy}. Note that in this description the solution of the algorithm after $i$ iterations, for every integer $0 \leq i \leq n$, is denoted by $A_i$.

\begin{algorithm}
\DontPrintSemicolon
 Let $A_0\leftarrow \varnothing$.\;
\For{$i=1$ \KwTo $k$}{
Let $u_i$ be the element of $\mathcal{N} \setminus A_{i - 1}$ maximizing $f(u_i\mid A_{i-1})$.
\;
\lIf{$f(u_i\mid A_{i-1})\geq 0$}{Let $A_i\leftarrow A_{i-1}+u_i$.}
\lElse{Let $A_i \gets A_{i-1}$.}
}

\Return $A_k$.
\caption{The Greedy Algorithm $(f,k)$\label{alg:greedy}}
\end{algorithm}

Our first step towards proving Theorem~\ref{thm:greedy} is the following lemma, which lower bounds the increase in the value of $f(A_i)$ as a function of $i$. Specifically, the lemma shows that this increase is significant as long as there is a significant gap between between $f(A_{i - 1})$ and $m\cdot f(OPT)$, where $OPT$ is an arbitrary optimal solution.
\begin{lemma}
For every integer $1\leq i\leq k, f(A_i)-f(A_{i-1})\geq k^{-1}[m\cdot f(OPT)-f(A_{i-1})]$.
\end{lemma}
\begin{proof}
We need to distinguish between two cases. Consider first the case in which $f(u_i \mid A_{i - 1}) \geq 0$. In this case,
\begin{align*}
	f(A_i)-f(A_{i-1})
	={} &
	f(u_i\mid A_{i-1})
	\geq
	\frac{|OPT \setminus A_{i - 1}|}{k} \cdot f(u_i\mid A_{i-1})\\
	\geq{} &
	\frac{|OPT \setminus A_{i - 1}|}{k} \cdot \max_{u\in OPT \setminus A_{i - 1}} \mspace{-18mu} f(u\mid A_{i-1})
	\geq
	\frac{\sum_{u\in OPT \setminus A_{i - 1}} f(u\mid A_{i-1})}{k}\\
	\geq{} &
	\frac{f(OPT\cup A_{i-1})-f(A_{i-1})}{k}
	\geq
	\frac{m\cdot f(OPT)-f(A_{i-1})}{k}
	\enspace,
\end{align*}
where the first inequality holds since $|OPT \setminus A_{i - 1}| \leq |OPT| \leq k$ because $OPT$ is a feasible solution, the second inequality is due to the way used by the greedy algorithm to choose the element $u_i$, the penulatimate inequality follows from the submodularity of $f$, and  the last inequality holds since $f$ is $m$-monotone.

Consider now the case in which $f(u_i \mid A_{i - 1}) < 0$. In this case, $f(A_i) - f(A_{i - 1}) = 0$ because $A_i = A_{i - 1}$. Furthermore, repeating the arguments used to prove the above inequality yields
\begin{align*}
	m\cdot f(OPT)-f(A_{i-1})
	\leq{} &
	|OPT \setminus A_{i - 1}| \cdot \max_{u \in OPT \setminus A_{i - 1}} \mspace{-18mu} f(u \mid A_{i - 1})\\
	\leq{} &
	|OPT \setminus A_{i - 1}| \cdot \max_{u \in \cN \setminus A_{i - 1}} \mspace{-18mu} f(u \mid A_{i - 1})
	\leq
	0
	\enspace.
	\qedhere
\end{align*}
\end{proof}

Rearranging the last lemma, we get the following inequality.
\begin{equation}\label{greedeq}
    m\cdot f(OPT)-f(A_i) \leq (1-1/k) \cdot [m\cdot f(OPT)-f(A_{i-1})]
			\enspace.
\end{equation}
This inequality bounds the rate in which the gap between $m \cdot f(OPT)$ reduces as a function of $i$. This allows us to prove Theorem~\ref{thm:greedy}.
\begin{proof}[Proof of Theorem~\ref{thm:greedy}]
Combining Inequality~\eqref{greedeq} for every integer $1\leq i\leq k$ yields 
\[
    m\cdot f(OPT)-f(A_k)\leq (1-1/k)^k\cdot [m\cdot f(OPT)-f(A_0)]
		\enspace.
\]
Rearranging this inequality, we get
\[
    f(A_k)\geq m\cdot f(OPT)-m\cdot (1-1/k)^k\cdot \left[f(OPT) - f(A_0)\right]\geq m\cdot \left(1-\frac{1}{e}\right)\cdot f(OPT)
		\enspace,
\]
where the last inequality follows from the non-negativity of $f$ and the inequality $(1-1/k)^k\leq \frac{1}{e}.$
\end{proof}

\subsection{Analysis of Random Greedy} \label{ssc:random_greedy}

In this section we prove Theorem~\ref{thm:random_greedy}, which we repeat here for convenience.
\thmRandomGreedy*

Like the standard greedy algorithm from Section~\ref{ssc:greedy}, the Random Greedy algorithm starts with an empty solution, and then augments it in $k$ iterations. Specifically, in iteration $i$ the algorithm finds a set $M_i$ of at most $k$ elements whose total marginal contribution with respect to the current solution is maximal. Then, at most one element of $M_i$ is added to the algorithm's current solution in a random way guaranteeing that every element of $M_i$ is added to the solution with probability exactly $\nicefrac{1}{k}$. A formal presentation of the Random Greedy algorithm appears as Algorithm~\ref{alg:random_greedy}. Note that in this presentation the solution of the algorithm after $i$ iterations is denoted by $A_i$.

\SetKwIF{With}{OtherwiseWith}{Otherwise}{with}{do}{otherwise with}{otherwise}{}
\begin{algorithm}
\DontPrintSemicolon
Let $A_0\leftarrow \varnothing$.\;
\For{$i=1$ \KwTo $k$}{
Let $M_i \gets \argmax_{B \subseteq \cN \setminus A_{i - 1}, |B| \leq k} \{\sum_{u\in B}f(u\mid A_{i-1})\}$.\\
\With{probability $(1- |M_i| /k)$}{
	$A_i \leftarrow A_{i-1}$.
}
\Otherwise
{
	Let $u_i$ be a uniformly random element of $M_i$.\\
	Set $A_i\leftarrow A_{i-1} + u_i$.
}
}
\Return $A_k$.
\caption{Random Greedy $(f,k)$\label{alg:random_greedy}}
\end{algorithm}

We start the analysis of the Random Greedy algorithm with the following lemma.
\begin{lemma}
For every integer $0\leq i\leq k$ and element $u\in\mathcal{N}$, $\Pr[u\in A_i] \leq 1-(1-1/k)^i$.
\end{lemma}
\begin{proof}
Note that in each iteration $i$ of Algorithm~\ref{alg:random_greedy}, any element $u \in \cN \setminus A_{i - 1}$ is added to the current solution with probability of at most $\nicefrac{1}{k}$. Hence,
\[
	\Pr[u\in A_i]
	=
	1 - \Pr[u\notin A_i]
	=
	1 - \prod_{j = 1}^i \Pr[u \not \in A_j \mid u \not \in A_{j - 1}]
	\leq
	1-(1-1/k)^i
	\enspace.
	\qedhere
\]
\end{proof}

Plugging the guarantee of the last lemma into Corollary~\ref{cor:sampling} yields the following lower bound on the expected value of $A_i\cup OPT$.
\begin{corollary}\label{cor:RG}
For every integer $0\leq i\leq k$, $\mathbb{E}[f(A_i\cup OPT)] \geq [1 - (1 - m) \cdot (1 - (1-\frac{1}{k})^i)] \cdot f(OPT) = m \cdot f(OPT) + (1 - m)(1-\frac{1}{k})^i \cdot f(OPT)$.
\end{corollary}

Using the last corollary we are now ready to prove Theorem~\ref{thm:random_greedy}.

\begin{proof}[Proof of Theorem~\ref{thm:random_greedy}]
Let $\cE_{i-1}$ be an arbitrary possible choice for the random decisions of Random Greedy during its first $i-1$ iterations. Observe that, conditioned on $\cE_{i - 1}$ happening,
\begin{align*}
	\mathbb{E}[f(A_i) - f(A_{i-1})]
	={} &
	\frac{\sum_{u\in M_i}f(u\mid A_{i-1})}{k}\\
	\geq{} & \frac{\sum_{u\in OPT \setminus A_{i - 1}}f(u\mid A_{i-1})}{k} 
	\geq
	\frac{f(A_{i-1}\cup OPT)-f(A_{i-1})}{k}\enspace,
\end{align*}
where the first inequality follows from the choice of $M_i$ by the algorithm, and the second inequality follows from submodularity.  Taking now expectation over the choice $\cE_{i - 1}$ that realized, the last inequality yields
\begin{align} \label{eq:difference_bound}
\mathbb{E}[f(A_i) - f(A_{i - 1})]&\geq \frac{\mathbb{E}[f(A_{i-1}\cup OPT)]-\mathbb{E}[f(A_{i-1})]}{k}
\\\nonumber &\geq \frac{m \cdot f(OPT) + (1 - m)(1-\frac{1}{k})^{i - 1} \cdot f(OPT) - \mathbb{E}[f(A_{i-1})]}{k} \enspace,
\end{align}
where the second inequality is due to Corollary~\ref{cor:RG}.

The last inequality lower bounds the expected increase in the value of the solution of Random Greedy in every iteration. This implies also a lower bound on the expected value of $f(A_i)$. To complete the proof of the theorem, we need to prove a closed form for this implied lower bound, which we do by induction. Specifically, let us prove by induction on $i$ that
\begin{equation} \label{eq:random_greedy_induction}
	\mathbb{E}[f(A_i)]
	\geq
	\left[m\cdot \left(1-\left(1-\frac{1}{k}\right)^i\right) + (1-m)\cdot \frac{i}{k} \cdot \left(1-\frac{1}{k}\right)^{i-1} \right]\cdot f(OPT)
\end{equation}
for every integer $0\leq i\leq k$, which implies the theorem by plugging $i=k$ because $(1 - 1/k)^k \leq 1/e \leq (1 - 1/k)^{k - 1}$.

For $i=0$, Inequality~\eqref{eq:random_greedy_induction} holds since the non-negativity of $f$ guarantees that $f(A_0)\geq 0 = \left[(1-m)\cdot(\frac{0}{k})\cdot (1-\frac{1}{k})^{-1} + m\cdot (1-(1-\frac{1}{k})^0)\right]\cdot f(OPT)$. Consider now some integer $0<i\leq k$, and let us prove Inequality~\eqref{eq:random_greedy_induction} for this value of $i$ assuming that its holds for $i - 1$. By Inequality~\eqref{eq:difference_bound},
\begin{align*}
	\mathbb{E}[f(A_{i})]
	={} &
	\mathbb{E}[f(A_{i-1})]+\mathbb{E}[f(A_i) - f(A_{i-1})]\\
	\geq{} &
	\mathbb{E}[f(A_{i-1})]+\frac{m \cdot f(OPT) + (1 - m)(1-\frac{1}{k})^{i - 1} \cdot f(OPT) - \mathbb{E}[f(A_{i-1})]}{k}\\
	={} &
	\left(1-\frac{1}{k}\right)\cdot\mathbb{E}[f(A_{i-1})]+\frac{m + (1 - m)(1-\frac{1}{k})^{i - 1}}{k} \cdot f(OPT)
	\enspace.
\end{align*}
Plugging the induction hypothesis into the last inequality, we get
\begin{align*}
	\mathbb{E}[f(A_{i})]
	\geq{} &
	\left(1-\frac{1}{k}\right)\cdot \left[m\cdot\left(1-\left(1-\frac{1}{k}\right)^{i - 1}\right) + (1-m)\cdot\frac{i-1}{k}\cdot \left(1-\frac{1}{k}\right)^{i-2}\right] \cdot f(OPT)\\&\mspace{350mu}+\frac{m + (1 - m)(1-\frac{1}{k})^{i - 1}}{k}\cdot f(OPT)\\
	={} &
	\left[m\left(1-\left(1-\frac{1}{k}\right)^i\right) + (1-m)\cdot \frac{i}{k} \cdot \left(1-\frac{1}{k}\right)^{i-1}\right]\cdot f(OPT)
	\enspace.
	\qedhere
\end{align*}
\end{proof}

\subsection{Inapproximability for a Cardinality Constraint} \label{ssc:inapproximability_cardinality}

In this section we prove Theorem~\ref{thm:inapproximability_cardinality}, which we repeat here for convenience.
\thmInapproximabilityCardinality*

We prove Theorem~\ref{thm:inapproximability_cardinality} using the symmetry gap technique, and specifically, via our extension of this technique proved in Theorem~\ref{thm:symmetry_gap}. To use this theorem, we need to construct an instance of our problem in which there is a large gap between the values of the best (general) solution and the best symmetric solution. Our instance is based on an instance constructed by Oveis Gharan and Vondr\'{a}k~\cite{oveisgharan2011submodular}. However, the objective function in the original instance of~\cite{oveisgharan2011submodular} is not $m$-monotone for any $m > 0$, and therefore, we need to modify it so that it becomes $m$-monotone for a value $m \in [0, 1]$ of our choosing.

Fix some positive integer value $r$ to be determined later and some value $\alpha \in [0, 1]$. The ground set of the instance we construct is $\cN = \{a, b\} \cup \{a_i, b_i \mid i \in [r]\}$, and the constraint of the instance is a cardinality constraint allowing a feasible solution to include up to $2$ elements. The objective function of our instance is the function $f\colon 2^\cN \to \nnR$ defined by $f(S) = \alpha \cdot f_1(S) + (1 - \alpha)[f_2(S) + f_3(S)]$, where
\begin{gather*}
	f_1(S)
	=
	m \cdot \characteristic[S \cap \{a, b\} \neq \varnothing] + (1 - m) \cdot (|S \cap \{a, b\}| \bmod 2)
	\enspace,\\
	f_2(S)
	=
	\characteristic[S \cap \{a_i \mid i \in [r]\} \neq \varnothing] \cdot (1 - (1 - m) \cdot \characteristic[a \in S])
\end{gather*}
and
\[
	f_3(S)
	=
	\characteristic[S \cap \{b_i \mid i \in [r]\} \neq \varnothing] \cdot (1 - (1 - m) \cdot \characteristic[b \in S])
	\enspace.
\]

Let us denote the above described instance of submodular maximization subject to a cardinality constraint by $\cI$. We begin the analysis of $\cI$ by proving some properties of its objective function.
\begin{lemma} \label{lem:f_propeties_cardinality}
The objective function $f$ of $\cI$ is non-negative, $m$-monotone and submodular.
\end{lemma}
\begin{proof}
We prove below that the functions $f_1$, $f_2$ and $f_3$ have the properties stated in the lemma. This implies that $f$ also has these properties by Observation~\ref{obs:linearity} and the well-known closure of the class of submodular functions to multiplication by a non-negative constant and addition (see, e.g., Lemma~1.2 of~\cite{buchbinder2018submodular}). The function $f_1$ is identical to the function proved in Section~\ref{sec:unconstrained} to have the properties stated in the lemma, and the functions $f_2$ and $f_3$ are identical to each other up to switching the roles of $a$ with $b$ and $a_i$ with $b_i$. Therefore, to prove that both $f_2$ and $f_3$ have the properties stated by the lemma it suffices to show that $f_2$ has these properties, which we do in the rest of this proof.

Clearly, $f_2$ is non-negative. To see that $f_2$ is a submodular function, note that
\begin{compactitem}
	\item For every set $S \subseteq \cN - a$, $f_2(a \mid S) = -\characteristic[S \cap \{a_i \mid i \in [r]\} \neq \varnothing] \cdot (1 - m)$.
	\item For every integer $1 \leq i \leq r$ and set $S \subseteq \cN - a_i$, $f_2(a_i \mid S) = \characteristic[S \cap \{a_i \mid i \in [r]\} = \varnothing] \cdot (1 - (1 - m) \cdot \characteristic[a \in S])$.
	\item For every element $u \in (\cN - a) \setminus \{a_i \mid i \in [r]\}$ and set $S \subseteq \cN - u$, $f_2(u \mid S) = 0$.
\end{compactitem}
Since all the above marginal contributions are down-monotone functions of $S$ (i.e., functions whose value can only decrease when elements are added to $S$), the function $f_2$ is submodular.

It remains to argue why $f_2$ is $m$-monotone. Consider any two sets $S \subseteq T \subseteq \cN$. If $f_2(S) = 0$, then the inequality $m \cdot f(S) \leq f(T)$ follows from the non-negativity of $f_2$. Therefore, consider the case in which $f_2(S) > 0$, which implies that $S \cap \{a_i \mid i \in [r]\} \neq \varnothing$; and therefore, $f_2(S) = (1 - (1 - m) \cdot \characteristic[a \in S]) \leq 1$. Since $S$ is a subset of $T$, we also get $f_2(T) = (1 - (1 - m) \cdot \characteristic[a \in T]) \geq m$, and hence, $m \cdot f_2(S) \leq m \cdot 1 = m \leq f_2(T)$.
\end{proof}

A cardinality constraint is symmetric in the sense that the feasibility of a set depends only on the number of elements in it, and is completely independent of the identity of these elements. Let us now denote by $\cG$ the group of permutations of $\cN$ that are equivalent to applying any number of the following two steps: (1) switching $a$ with $b$ and $a_i$ with $b_i$ for every $i \in [r]$, or (2) switching $a_i$ with $a_j$ for two integers $i, j \in [r]$. The first step preserves the value of $f$ because it simply switches the values of $f_2$ and $f_3$, while leaving the value of $f_1$ unaffected; and the second step preserves the value of $f$ since it deals with elements that both $f_1$ and $f_3$ ignore, and $f_2$ treats in the same way. Hence, for every set $S \subseteq \cN$ and permutation $\sigma \in \cG$, we have $f(S) = f(\sigma(S))$, which implies the following observation.
\begin{observation} \label{obs:symmetry_cardinality}
The instance $\cI$ is strongly symmetric with respect to $\cG$.
\end{observation}

To use Theorem~\ref{thm:symmetry_gap}, we still need to bound the symmetry gap of $\cI$, which we do next.
\begin{lemma} \label{lem:gap_cardinality}
The symmetry gap of $\cI$ is at most
\begin{align*}
	&
	\frac{\max_{\substack{x \in [0, 1]}} \{\alpha(mx^2 + 2x - 2x^2) + 2(1 - \alpha)[1 - (1 - (1 - x)/r)^r](1 - (1 - m)x)\}}{\max\{1, 2(1 - \alpha)\}}\\
	\leq{} &
	\frac{\max_{\substack{x \in [0, 1]}} \{\alpha(mx^2 + 2x - 2x^2) + 2(1 - \alpha)(1 - e^{x - 1})(1 - (1 - m)x)\}}{\max\{1, 2(1 - \alpha)\}} + 2/r
	\enspace.
\end{align*}
\end{lemma}
\begin{proof}
Two possible feasible solutions for $\cI$ are the sets $\{a, b_1\}$ and $\{a_1, b_1\}$ whose values according to $f$ are $1$ and $2(1 - \alpha)$, respectively. Therefore, the value of the optimal solution for $\cI$ is at least $\max\{1, 2(1 - \alpha)\}$. Since the symmetry gap is the ratio between the value of the best symmetric solution and the value of the best solution, to prove the lemma it remains to argue that the best symmetric solution for $\cI$ has a value of $\max_{\substack{x \in [0, 1]}} \{\alpha(mx^2 + 2x - 2x^2) + 2(1 - \alpha)[1 - (1 - (1 - x)/r)^r](1 - (1 - m)x)\}$.

We remind the reader that a symmetric solution for $\cI$ is $\bar{\vy} = \bE_{\sigma \in \cG}[\vy]$ for some vector $\vy \in [0, 1]^\cN$ obeying $\|\vy\|_1 \leq 2$. Since $a$ and $b$ can be exchanged with each other by the permutations of $\cG$, the values of the coordinates of $a$ and $b$ in $\bar{\vy}$ must be equal to each other. Similarly, every two elements of $\{a_i, b_i \in i \in [r]\}$ can be exchanged by the permutations of $\cG$, and therefore, the values of the coordinates of these elements in $\bar{\vy}$ must all be identical. Thus, any symmetric solution $\bar{\vy}$ can be represented as
\[
	\bar{y}_u
	=
	\begin{cases}
		x & \text{if $u = a$ or $u = b$} \enspace,\\
		z & \text{if $u \in \{a_i, b_i \mid i \in [r]\}$}
	\end{cases}
\]
for some values $x, z \in [0, 1]$ obeying $2x + 2rz \leq 2$ (or equivalently, $z \leq (1 - x)/r$). The value of this solution (according to the multilinear extension $F$ of $f$) is
\begin{align*}
	&
	\alpha[m(1 - (1 - x)^2) + 2(1 - m)x(1 - x)] + 2(1 - \alpha)(1 - (1 - z)^r)(1 - (1 - m)x)\\
	={} &
	\alpha(mx^2 + 2x - 2x^2) + 2(1 - \alpha)(1 - (1 - z)^r)(1 - (1 - m)x)
	\enspace.
\end{align*}
Since this expression is a non-decreasing function of $z$, the maximum value of any symmetry solution for $\cI$ is
\begin{align*}
	&
	\max_{\substack{x, z \in [0, 1] \\ z \leq (1 - x)/r}} \{\alpha(mx^2 + 2x - 2x^2) + 2(1 - \alpha)(1 - (1 - z)^r)(1 - (1 - m)x)\}\\
	={} &
	\max_{\substack{x \in [0, 1]}} \{\alpha(mx^2 + 2x - 2x^2) + 2(1 - \alpha)[1 - (1 - (1 - x)/r)^r](1 - (1 - m)x)\}
	\enspace.
	\qedhere
\end{align*}
\end{proof}

Since any refinement of a cardinality constraint is a cardinality constraint over a larger ground set, plugging Lemma~\ref{lem:f_propeties_cardinality}, Observation~\ref{obs:symmetry_cardinality} and Lemma~\ref{lem:gap_cardinality} into Theorem~\ref{thm:symmetry_gap} yields the following corollary.
\begin{corollary}
For every constant $\eps' > 0$, no polynomial time algorithm for maximizing a non-negative $m$-monotone submodular function subject to a cardinality contraint obtains an approximation ratio of
\[
	\frac{\max_{\substack{x \in [0, 1]}} \{\alpha(mx^2 + 2x - 2x^2) + 2(1 - \alpha)(1 - e^{x - 1})(1 - (1 - m)x)\}}{\max\{1, 2(1 - \alpha)\}} + 2/r + \eps'
	\enspace.
\]
\end{corollary}

Theorem~\ref{thm:inapproximability_cardinality} now follows from the last corollary by choosing $\eps' = \eps / 2$, $r = \lceil 4/\eps \rceil$ and
\[
	\alpha
	=
	\argmin_{\alpha' \in [0, 1]} \frac{\max_{\substack{x \in [0, 1]}} \{\alpha'(mx^2 + 2x - 2x^2) + 2(1 - \alpha')(1 - e^{x - 1})(1 - (1 - m)x)\}}{\max\{1, 2(1 - \alpha')\}}
	\enspace.
\]

%% file: CardinalityGraph.tikz
\begin{tikzpicture}
\begin{axis}[
    xlabel = {$m$},
    ylabel = {approximation ratio},
    xmin=0, xmax=1,
    ymin=0, ymax=1,
		legend cell align=left]
 
\addplot [name path = Greedy, domain = 0:1, blue, mark=o, mark repeat=2] {x * (1 - exp(-1))};
\addlegendentry{Greedy}
 
\addplot [name path = RandomGreedy, domain = 0:1, red, mark=triangle*, mark repeat=2] {x * (1 - exp(-1)) + (1 - x) * exp(-1)};
\addlegendentry{Random Greedy}

\addplot [name path = Inapproximability, darkgreen, mark=x, mark repeat=10] coordinates {
      (0.000000, 0.490984)
      (0.010000, 0.493147)
      (0.020000, 0.495323)
      (0.030000, 0.497513)
      (0.040000, 0.499716)
      (0.050000, 0.501932)
      (0.060000, 0.504162)
      (0.070000, 0.506404)
      (0.080000, 0.508660)
      (0.090000, 0.510929)
      (0.100000, 0.513211)
      (0.110000, 0.515507)
      (0.120000, 0.517815)
      (0.130000, 0.520137)
      (0.140000, 0.522472)
      (0.150000, 0.524820)
      (0.160000, 0.527181)
      (0.170000, 0.529555)
      (0.180000, 0.531942)
      (0.190000, 0.534343)
      (0.200000, 0.536756)
      (0.210000, 0.539182)
      (0.220000, 0.541622)
      (0.230000, 0.544074)
      (0.240000, 0.546539)
      (0.250000, 0.549017)
      (0.260000, 0.551507)
      (0.270000, 0.554011)
      (0.280000, 0.556527)
      (0.290000, 0.559055)
      (0.300000, 0.561597)
      (0.310000, 0.564151)
      (0.320000, 0.566717)
      (0.330000, 0.569296)
      (0.340000, 0.571887)
      (0.350000, 0.574490)
      (0.360000, 0.577106)
      (0.370000, 0.579734)
      (0.380000, 0.582374)
      (0.390000, 0.585026)
      (0.400000, 0.587691)
      (0.410000, 0.590367)
      (0.420000, 0.593055)
      (0.430000, 0.595754)
      (0.440000, 0.598466)
      (0.450000, 0.601189)
      (0.460000, 0.603924)
      (0.470000, 0.606670)
      (0.480000, 0.609427)
      (0.490000, 0.612196)
      (0.500000, 0.614977)
      (0.510000, 0.617768)
      (0.520000, 0.620571)
      (0.530000, 0.623384)
      (0.540000, 0.626209)
      (0.550000, 0.629044)
      (0.560000, 0.631890)
      (0.570000, 0.632121)
      (0.580000, 0.632121)
      (0.590000, 0.632121)
      (0.600000, 0.632121)
      (0.610000, 0.632121)
      (0.620000, 0.632121)
      (0.630000, 0.632121)
      (0.640000, 0.632121)
      (0.650000, 0.632121)
      (0.660000, 0.632121)
      (0.670000, 0.632121)
      (0.680000, 0.632121)
      (0.690000, 0.632121)
      (0.700000, 0.632121)
      (0.710000, 0.632121)
      (0.720000, 0.632121)
      (0.730000, 0.632121)
      (0.740000, 0.632121)
      (0.750000, 0.632121)
      (0.760000, 0.632121)
      (0.770000, 0.632121)
      (0.780000, 0.632121)
      (0.790000, 0.632121)
      (0.800000, 0.632121)
      (0.810000, 0.632121)
      (0.820000, 0.632121)
      (0.830000, 0.632121)
      (0.840000, 0.632121)
      (0.850000, 0.632121)
      (0.860000, 0.632121)
      (0.870000, 0.632121)
      (0.880000, 0.632121)
      (0.890000, 0.632121)
      (0.900000, 0.632121)
      (0.910000, 0.632121)
      (0.920000, 0.632121)
      (0.930000, 0.632121)
      (0.940000, 0.632121)
      (0.950000, 0.632121)
      (0.960000, 0.632121)
      (0.970000, 0.632121)
      (0.980000, 0.632121)
      (0.990000, 0.632121)
			(1.000000, 0.632121)
};
\addlegendentry{Inapproximability}

\end{axis}
\end{tikzpicture}

%% file: Matroid.tex
\section{Maximization subject to a Matroid Constraint} \label{sec:matroid}

In this section we consider the problem of maximizing a non-negative submodular function subject to a matroid constraint. A matroid $\cM$ over the ground set $\cN$ is defined as a pair $\cM = (\cN, \cI)$, where $\cI$ is a non-empty subset of $2^\cN$ obeying two properties for every two sets $S, T \subseteq \cN$: (i) if $S \subseteq T$ and $T \in \cI$, then $S \in \cI$; and (ii) if $S, T \in \cI$ and $|S| < |T|$, then there exists an element $u \in T \setminus S$ such that $S + u \in \cI$. A set $S \subseteq \cN$ is called \emph{independent} with respect to a matroid $\cM$ if it belongs to $\cI$ (otherwise, we say that $S$ is dependent with respect to $\cM$); and the matroid constraint corresponding to a given matroid $\cM$ allows only independent sets with respect to this matroid as feasible solutions. Hence, we can restate the problem we consider in this section in the following more formal way. Given a non-negative submodular function $f\colon 2^\cN \to \nnR$ and a matroid $\cM = (\cN, \cI)$ over the same ground set, output an independent set $S \in \cI$ (approximately) maximizing $f$ among all such sets.

When the objective function $f$ is guaranteed to be monotone, a standard greedy algorithm guarantees $1/2$-approximation for the above problem~\cite{fisher1978analysis}. Our first result for this section (whose proof appears in Section~\ref{ssc:greedy_matroid}) shows how this approximation guarantee changes as a function of $m$ (note that since matroid constraints generalize cardinality constraints, the greedy algorithm has no constant guarantee for non-monotone functions in this case as well).
\begin{restatable}{theorem}{thmGreedyMatroid} \label{thm:greedy_matroid}
The Greedy algorithm (Algorithm~\ref{alg:greedy_matroid}) has an approximation ratio of at least $m/2$ for the problem of maximizing a non-negative $m$-monotone submodular function subject to a matroid constraint.
\end{restatable}

The approximation ratio of the greedy algorithm was improved over by the seminal work of C{\u{a}}linescu et al.~\cite{calinescu2011maximizing}, who described the Continuous Greedy algorithm whose approximation ratio is $1 - 1/e$ when $f$ is monotone; matching the inapproximability result proved by Nemhauser and Wolsey~\cite{nemhauser1978best} for the special case of a cardinality constraint. In contrast, when the objective $f$ is not guaranteed to be monotone, the approximability of the problem is less well-understood. On the one hand, after a long line of works~\cite{ene2016constrained,feldman2011unified,lee2009non-monotone,oveisgharan2011submodular,vondrak2013symmetry}, the state-of-the-art approximation ratio for the problem is $0.385$~\cite{buchbinder2019constrained}, but on the other hand, it is only known that no polynomial time algorithm for the problem can obtain an approximation ratio of $0.478$~\cite{oveisgharan2011submodular}.

Unfortunately, the above mentioned state-of-the-art $0.385$-approximation algorithm is quite involved. Therefore, we chose to consider in this work two other algorithms. The first algorithm is the Measure Continuous Greedy algorithm (due to~\cite{feldman2011unified}) which guarantees an approximation ratio of $1/e - o(1) \approx 0.367$. This algorithm performs only slightly worse than the above state-of-the-art, and is in fact a central component of all the currently known algorithms achieving better than $1/e$-approximation. Measured Continuous Greedy is also known to guarantee $(1 - 1/e - o(1))$-approximation when the objective $f$ is monotone, and the next theorem (which we prove in Section~\ref{ssc:continuous_greedy}) shows that the approximation guarantee of this algorithm changes smoothly with the monotonicity ratio of $f$ between the above two expressions.

\newtoggle{firstContinuousGreedy}
\begin{restatable}{theorem}{thmContinuousGreedy} \label{thm:continuous_greedy}
The Measured Continuous Greedy algorithm (Algorithm~\ref{alg:continuous_greedy}) has an approximation ratio of at least $m(1 - 1/e) + (1 - m) \cdot (1/e) - o(1)$ for the problem of maximizing a non-negative $m$-monotone submodular function subject to a matroid constraint, where the $o(1)$ term diminishes with the size of the ground set $\cN$.\iftoggle{firstContinuousGreedy}{}{\footnote{Technically, Measured Continuous Greedy is an algorithm for maximizing the multilinear extesnion of a non-negative submodular function subject to a general solvable down-closed convex body $P$ constraint, and we prove in Section~\ref{ssc:continuous_greedy} that it guarantees the approximation ratio stated in Theorem~\ref{thm:continuous_greedy} for this setting. However, this implies the result stated in Theorem~\ref{thm:continuous_greedy} using a standard reduction (see Section~\ref{ssc:continuous_greedy} for further detail).}}
\end{restatable}
\toggletrue{firstContinuousGreedy}

The other algorithm we consider is an algorithm called Random Greedy for Matroids (due to~\cite{buchbinder2014submodular}). Unlike Measured Continuous Greedy and almost all the other algorithms suggested for non-monotone objectives to date, this algorithm is combinatorial, which makes it appealing in practice. Buchbinder et al.~\cite{buchbinder2014submodular} proved an approximation ratio of roughly $(1 + e^{-2})/4$ for Random Greedy for Matroids. The next theorem shows how this approximation guarantee improves as a function of the monotonicity ratio. In this theorem we refer to the rank $k$ of the matroid constraint $\cM$, which is simply defined as the size of the largest independent set with respect to this matroid. We also note that the algorithm we analyze is identical to the algorithm of~\cite{buchbinder2014submodular} up to two modifications. The first modification is in the number of iterations that the algorithm makes. To get the result of Buchbinder et al.~\cite{buchbinder2014submodular}, it suffices to use $k$ iterations. However, the optimal number of iterations increases with $m$, and therefore, our version of the algorithm uses more than $k$ iterations. Furthermore, since we do not want to assume knowledge of $m$ in the algorithm, we use a number of iterations that is appropriate for $m = 1$, which requires us to make the second modification to the algorithm; namely, to allow an iteration to update the solution of the algorithm only when this update is beneficial. This guarantees that doing more iterations can never decrease the value of the algorithm's solution.

\begin{restatable}{theorem}{thmRandomGreedyMatroid} \label{thm:random_greedy_matroid}
For every $\eps \in (0, 1)$, the Random Greedy for Matroids algorithm (Algorithm~\ref{alg:random_greedy_matroid}) has an approximation ratio of at least $\frac{1 + m + e^{-2 / (1 - m)}}{4} - \eps - o_k(1)$ for the problem of maximizing a non-negative $m$-monotone submodular function subject to a matroid constraint (except in the case of $m = 1$ in which the approximation ratio is $1/2 - \eps - o_k(1)$). The notation $o_k(1)$ represents a term that diminishes with the rank $k$ of the matroid constraint.
\end{restatable}

Theorem~\ref{thm:random_greedy_matroid} is proved in Section~\ref{ssc:random_greedy_matroid}. We finish this section with the following theorem (proved in Section~\ref{ssc:inapproximability_matroids}), which generalizes the $0.478$ inapproximability result of Oveis Gharan and Vondr\'{a}k~\cite{oveisgharan2011submodular} for general non-negative submodular functions.

\begin{restatable}{theorem}{thmInapproximabilityMatroid} \label{thm:inapproximability_matroid}
For any constant $\eps > 0$, no polynomial time algorithm can obtain an approximation ratio of
\[
	\min_{\alpha \in [0, 1]} \max_{x \in [0, 1/2]} \{\alpha(mx^2 + 2x - 2x^2) + 2(1 - \alpha)(1-e^{-1/2})(1 - (1 - m)x)\} + \eps
\]
for the problem of maximizing a non-negative $m$-monotone submodular function subject to a matroid constraint.
\end{restatable}

Unfortunately, the mathematical expression given in Theorem~\ref{thm:inapproximability_matroid} is not very readable. Therefore, to get an intuitive understanding of its behavior, we numerically evaluated it for various values of $m$. The plot obtained in this way appears in Figure~\ref{fig:matroid}. For context, this figure also includes all the other results proved in this section. Somewhat surprisingly, Figure~\ref{fig:matroid} shows that Theorem~\ref{thm:inapproximability_matroid} does not generalize the $1-1/e$ inapproximability result of Nemhauser and Wolsey~\cite{nemhauser1978best} for monotone functions despite the fact that this inapproximability result holds for every monotonicity ratio $m \in [0, 1]$. This behavior resembles the inability of Theorem~\ref{thm:inapproximability_cardinality} to improve over the $1 - 1/e$ inapproximability of~\cite{nemhauser1978best} for large values of $m$.

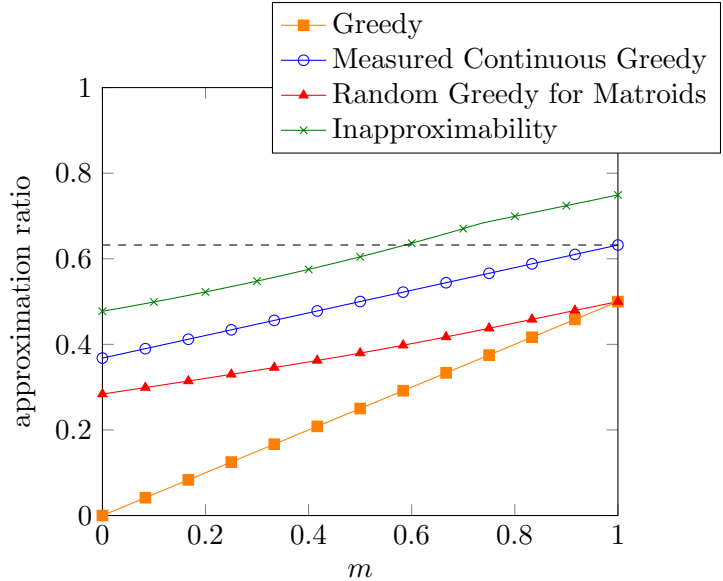
\begin{figure}
\begin{center}\input{MatroidGraph.tikz}\end{center}
\caption{Graphical representation of the results of Section~\ref{sec:matroid}. This plots depicted the approximation guarantees we prove for the greedy algorithm (Theorem~\ref{thm:greedy_matroid}), the Measured Continuous Greedy algorithm (Theorem~\ref{thm:continuous_greedy}) and the Random Greedy for Matroids algorithm (Theorem~\ref{thm:random_greedy_matroid}), and our inapproximability result (Theorem~\ref{thm:inapproximability_matroid}). The dashed line corresponds to an approximation ratio of $1 - 1/e$, which is another inapproximability result, inherited from monotone submodular functions~\cite{nemhauser1978best}.} \label{fig:matroid}
\end{figure}

\subsection{Analysis of the Greedy algorithm} \label{ssc:greedy_matroid}

A version of the greedy algorithm designed for matroid constraints appears as Algorithm~\ref{alg:greedy_matroid}. This algorithm starts with an empty solution, and then iteratively adds elements to this solution, where the element added in each iteration is the element with the largest marginal contrition with respect to the current solution among all the elements whose addition to the solution does not violate feasibility. The algorithm terminates when no additional elements can be added to the solution without decreasing its value.

\begin{algorithm}[ht]
\DontPrintSemicolon
 Let $A_0\leftarrow \varnothing$, and $i \gets 0$.\;
\While{true}{
	Let $u_{i + 1}$ be the element of $\{v \in \mathcal{N} \setminus A_i \mid A_i + v \in \cI\}$ maximizing $f(u_{i + 1} \mid A_i)$.\;
	\lIf{$f(u_{i + 1} \mid A_i)\geq 0$}{Let $A_{i + 1} \leftarrow A_i+u_{i + 1}$, and then, increase $i$ by $1$.}
	\lElse{\Return $A_i$.}
}
\caption{The Greedy Algorithm (for a Matroid Constraint) $(f,\cM = (\cN, \cI))$\label{alg:greedy_matroid}}
\end{algorithm}

\thmGreedyMatroid*
\begin{proof}
Lemma~3.2 of~\cite{gupta2010constrained} shows that the greedy algorithm outputs a solution $S$ of value at least $f(S \cup OPT)/2$ for the problem of maximizing a non-negative submodular function $f$ subject to a matroid constraint, where $OPT$ is an optimal solution for the problem.\footnote{In fact, Lemma~3.2 of~\cite{gupta2010constrained} proves a more general result for $p$-set systems, but it implies the stated result since matroids are $1$-set systems.} The theorem now follows since for an $m$-monotone function $f$ we are guaranteed to have $f(S \cup OPT) \geq m \cdot f(OPT)$.
\end{proof}

\subsection{Analysis of Measured Continuous Greedy} \label{ssc:continuous_greedy}

In this section, we reanalyze the Measured Continuous Greedy algorithm of \cite{feldman2011unified} in view of the monotonicity ratio. Given a non-negative submodular function $f\colon 2^\cN \to \nnR$ and a down-closed solvable\footnote{A body $P \subseteq [0, 1]^\cN$ is \emph{solvable} if one can efficiently optimize linear functions subject to it, and is \emph{down-closed} if $\vy \in P$ implies $\vx \in P$ for every vector $\vx \in [0, 1]^\cN$ obeying $\vx \leq \vy$ (this inequality should be understood to hold coordinate-wise).} convex body $P \subseteq [0, 1]^\cN$, Measured Continuous Greedy is an algorithm designed to find a vector $\vx \in P$ that approximately maximize $F(\vx)$, where $F$ is the multilinear extension of $f$. Specifically, we prove the following theorem.
\begin{theorem} \label{thm:measured_continuous_greedy}
Given a non-negative $m$-monotone submodular function $f\colon 2^\cN \to \nnR$, a solvable down-close convex body $P \subseteq [0, 1]^\cN$ and a parameter $T \geq 0$, Measured Continuous Greedy outputs a vector $\vx \in [0, 1]^\cN$ obeying $F(\vx) \geq [m(1 - e^{-T}) + (1 - m)Te^{-T}] \cdot f(OPT)$, where $F$ is the multilinear extension of $f$ and $OPT$ is the set maximizing $f$ among all sets whose characteristic vectors belong to $P$. Furthermore, $\vx \in P$ whenever $T \in [0, 1]$.
\end{theorem}

We note that Feldman et al.~\cite{feldman2011unified} discussed conditions that guarantee that $\vx$ belongs to $P$ also for some values of $T$ that are larger than $1$. However, the above stated form of Theorem~\ref{thm:measured_continuous_greedy} already suffices to prove Theorem~\ref{thm:continuous_greedy}. Let us explain why this is the case. When $P$ is the matroid polytope $P_\cM$ of a matroid $\cM$, there are algorithms called Pipage Rounding~\cite{calinescu2011maximizing} and Swap Rounding~\cite{chekuri2010dependent} that, given a vector $\vx \in P$ produce a set $S$ that is independent in $\cM$ and also obeys $\bE[f(S)] \geq F(\vx) - o(1) \cdot f(OPT)$. Therefore, one can obtain an algorithm for maximizing $f$ subject to the matroid $\cM$ by executing Measured Continuous Greedy with $P = P_\cM$ and $T = 1$, and then applying either Pipage Rounding or Swap Rounding to the resulting vector; which yields an algorithm with the properties specified by Theorem~\ref{thm:continuous_greedy}.

We now describe the version of Measured Continuous Greedy that we analyze (given as Algorithm~\ref{alg:continuous_greedy}). For simplicity, we chose to analyze a continuous version of this algorithm that assumes direct access to the multilinear extension $F$ of the objective function rather than just to the objective function itself. We refer the reader to~\cite{feldman2011unified} for details about discretizing the algorithm and avoiding the assumption of direct access to $F$. We also note that the $o(1)$ error term in the approximation guarantee stated in Theorem~\ref{thm:measured_continuous_greedy} is due to these issues.  Our description of Measured Continuous Greedy requires some additional notation, namely, given two vectors $\vx$ and $\vy$, we denote by $\vx \vee \vy$ their coordinate-wise maximum and by $\vx \odot \vy$ their coordinate-wise multiplication.

Measured Continuous Greedy starts at ``time'' $0$ with the empty solution, and improves this solution during the time interval $[0,T]$. We denote the solution of the algorithm at time $t$ by $\vy(t)$. At every time $t\in [0,T]$, the algorithm calculates a vector $\vw$ whose $u$-coordinate is the gain that can be obtained by increasing this coordinates in the solution $\vy(t)$ to be $1$ (i.e., $w_u(t) = F(\vy(t) \vee 1_{\{u\}}) - F(\vy(t))$). Then, the algorithm finds a vector $\vx(t)\in P$ that maximizes the objective function $\vw(t)\cdot \vx(t)$, and adds to the solution $\vy(t)$ an infinitesimal part of $(\characteristic_\cN-\vy(t)) \odot \vx(t)$ (to understand where the last expression comes from, we note that when $\vx$ is integral, fully adding $(\characteristic_\cN-\vy(t)) \odot \vx(t)$ to $\vy(t)$ sets to $1$ all the coordinates that are $1$ in $\vx(t)$, which matches the ``spirit'' of the definition of $\vw$).

\begin{algorithm}
\DontPrintSemicolon
Let $\vy(0)\leftarrow 1_{\varnothing}$.\;
\ForEach{$t \in [0,T)$}{
For each $u\in \mathcal{N}$, let $w_u(t) \leftarrow F(\vy(t) \vee 1_{\{u\}}) - F(\vy(t))$.\;
Let $\vx(t) \leftarrow \argmax_{\vx\in P}\{\vw(t)\cdot \vx\}$. \;
Increase $\vy(t)$ at a rate of $\frac{d\vy(t)}{dt} = (\characteristic_\cN-\vy(t)) \odot \vx(t)$. \label{line:growth} \;
}
\Return $y(T)$.
\caption{Measured Continuous Greedy($f,P,T$)\label{alg:continuous_greedy}}
\end{algorithm}

The first step in the analysis of Measured Continuous Greedy is bounding the maximum value of the coordinates of the solution $\vy(t)$.
\begin{lemma} \label{lem:y_upper_bound}
For every $t \in [0, T]$, $\|\vy(t)\|_\infty \leq 1 - e^{-t}$.
\end{lemma}
\begin{proof}
Fix an arbitrary element $u \in \cN$, and let us explain why $y_u(t) \leq 1 - e^{-t}$. By Line~\ref{line:growth} of Algorithm~\ref{alg:continuous_greedy}, $y_u(t)$ obeys the differential inequality
\[
	\frac{dy_u(t)}{dt}
	=
	(1-y_u(t)) \cdot x_u(t)
	\leq
	1 - y_u(t)
	\enspace,
\]
and the solution of this differential inequality for the initial condition $y_u = 0$ is
\[
	y_u(t) \leq 1 - e^{-t}
	\enspace.
	\qedhere
\]
\end{proof}

We are now ready to prove Theorem~\ref{thm:measured_continuous_greedy}
\begin{proof}[Proof of~\ref{thm:measured_continuous_greedy}]
Recall that $\vx(t)$ is a vector inside $P$ for every time $t\in [0,T]$, and since $P$ is down-closed, $(\characteristic_\cN-\vy(t))\odot \vx(t)$ and $\characteristic_{\varnothing}$ both belong to $P$ as well. This means that for $T \leq 1$ the vector $\vy(T) = (1 - T) \cdot \characteristic_\varnothing + \int_{0}^{T} (\characteristic_\cN - y(t)) \odot \vx(t) dt$ is a convex combination of vectors in $P$, and therefore belongs to $P$ by the convexity of $P$.

It remains to lower bound the value of $F(\vy(T))$.
By the chain rule,
\[\frac{dF(\vy(t))}{dt}=\sum_{u\in\mathcal{N}}\left(\frac{dy_u(t)}{dt}\cdot\frac{\partial F(\vy)}{\partial y_u}\Bigr|_{\vy=\vy(t)}\right)=\sum_{u\in\mathcal{N}}\left((1-y_u(t))\cdot x_u(t)\cdot\frac{\partial F(\vy)}{\partial y_u}\Bigr|_{\vy=\vy(t)}\right) \enspace.\]
Since $F$ is multilinear, its partial derivative with respect to a single coordinate is equal to the difference between the value of the function for two different values of this coordinate over the difference between these values.
Plugging this observation into the previous inequality yields
\[
\frac{dF(\vy(t))}{dt}
=\sum_{u\in\mathcal{N}}\left((1-y_u(t))\cdot x_u(t)\cdot\frac{F(\vy(t)\vee 1_{\{u\}})-F(\vy(t))}{1- y_u(t)}\right)
=\vx(t)\cdot \vw(t) \enspace.
\]
One possible candidate to be $\vx(t)$ is $\characteristic_{OPT}$. Hence, by the definition of $\vx(t)$, $\vx(t)\cdot \vw(t) \geq \characteristic_{OPT}\cdot \vw(t)$.
Combining this inequality with the previous one, we get
\begin{align*}
\frac{dF(\vy(t))}{dt} &\geq
	\characteristic_{OPT}\cdot \vw(t) =
	\sum_{u \in OPT} \left[F(\vy(t)\vee 1_{\{u\}}) - F(\vy(t))\right]\\
	&\geq F(\vy(t)\vee \characteristic_{OPT}) - F(\vy(t)) \geq [1 - (1 - m) \cdot \|\vy(t)\|_\infty] \cdot f(OPT) - F(\vy(t))\\
	&\geq
	[1 - (1 - m)(1 - e^{-t})] \cdot f(OPT) - F(\vy(t))
	=
	[m + (1 - m)e^{-t}] \cdot f(OPT) - F(\vy(t))
\enspace,
\end{align*}
where the second inequality holds by the submodularity of $f$, the penultimate inequality holds by Corollary~\ref{cor:sampling}, and the last inequality follows from Lemma~\ref{lem:y_upper_bound}. 

Solving the differential inequality that we got for the initial condition $F(\vy(0)) \geq 0$ (which holds by the non-negativity of $f$) yields
\[ F(y(t))\geq \left[m(1-e^{-t})+(1-m)te^{-t}\right]\cdot f(OPT) \enspace,\]
and the theorem now follows by plugging $t = T$.
%
\end{proof}

\subsection{Analysis of Random Greedy for Matroids} \label{ssc:random_greedy_matroid}

In this section we prove Theorem~\ref{thm:random_greedy_matroid}, which we repeat here for convenience.
\thmRandomGreedyMatroid*

To prove the theorem, we first need to state the algorithm it refers to. Towards this goal, let us assume that the ground set $\mathcal{N}$ contains a set $D$ of $2k$ ``dummy" elements that are known to the algorithm and have the following two properties.
\begin{compactitem}
	\item $f(S)=(S\setminus D)$ for every set $S\subseteq\mathcal{N}$.
	\item $S\in\mathcal{I}$ if and only if $S\setminus D\in\mathcal{I}$ and $|S|\leq k$.
\end{compactitem}
This assumption is useful since it allows us to assume that the optimal solution $OPT$ is a base of $\mathcal{M}$, and thus, simplifies the description of our algorithm (Random Greedy for Matroids). We can justify our assumption using the following procedure: (i) add $2k$ dummy elements to the ground set, (ii) extend $f$ and $\mathcal{I}$ according to the above properties, (iii) execute Random Greedy for Matroids on the resulting instance, and (iv) remove from the output of the algorithm any dummy elements that end up in it. This procedure guarantees that any approximation guarantee obtained by Random Greedy for Matroids using our assumption can be obtained also without the assumption.

Our version of the Random Greedy for Matroids algorithm is given as Algorithm~\ref{alg:random_greedy_matroid}. Like the original version of the algorithm (due to~\cite{buchbinder2014submodular}), our version starts with a base of $\cM$ consisting only of dummy elements, and then modifies it in a series of iterations. In each iteration $i$, the algorithm starts with a solution $S_{i - 1}$, and then identifies a base $M_i$ of $\cM$ whose elements have the largest total marginal contribution with respect to $S_{i - 1}$ ($M_i$ is also required to be disjoint from $S_{i - 1}$). The algorithm then picks a uniformly random element $u_i \in S_{i - 1}$, and adds it to the solution $S_{i - 1}$ at the expense of an element $g_i(u_i)$ of $S_{i - 1}$ given by a function $g_i$ that is chosen carefully (the existence of such a function follows, for example, from Corollary~39.12a of~\cite{schrijver2003combinatorial}).

As mentioned above, our version of Random Greedy for Matroids differs compared to the version of~\cite{buchbinder2014submodular} in two respects. First, we check whether replacing $g(u_i)$ with $u_i$ is beneficial, and make the swap only if this is indeed the case. Second, our algorithm gets a parameter $\eps \in (0, 1)$ and makes $k/\eps$ iterations instead of $k$ (we assume without loss of generality that $k/\eps$ is integral; otherwise, we can replace $\eps$ with a value which is smaller than $\eps$ by at most a factor of $2$ and has this property). 

\begin{algorithm}
\DontPrintSemicolon
Initialize $S_0$ to be an arbitrary base containing only elements of $D$.\;
\For{$i=1$ \KwTo $\nicefrac{k}{\varepsilon}$}{
Let $M_i\subseteq N$ be a base of $\mathcal{M}$ that contains only elements of $\cN \setminus S_{i - 1}$ and maximizes $\sum_{u\in M_i}f(u \mid S_{i-1})$ among all such bases.
\;
Let $g_i$ be a function mapping each element of $M_i$ to an element of $S_{i-1}$ obeying $S_{i-1} - g_i(u) + u \in \mathcal{I}$ for every $u\in S_{i-1}$. \;
Let $u_i$ be a uniformly random element from $M_i$. \;
\lIf{$f(S_{i - 1} - g_i(u_i) + u_i) > f(S_{i - 1})$}{Let $S_i \leftarrow S_{i-1} - g_i(u_i) + u_i$.}
\lElse{Let $S_i \gets S_{i - 1}$.}
}
\Return $S_{k / \eps}$.
\caption{Random Greedy for Matroids($f,\cM = (\cN, \cI),\varepsilon$)} \label{alg:random_greedy_matroid}
\end{algorithm}

Since Theorem~\ref{thm:random_greedy_matroid} is trivial for a constant $k$, we can assume in the analysis of Algorithm~\ref{alg:random_greedy_matroid} that $k$ is larger than any given constant. The first step in this analysis is proving the following lower bound on the expected value of $OPT\cup S_i$.
\begin{observation}\label{obs:randgreedmat}
For every integer $0\leq i\leq k/\eps$, $\mathbb{E}[f(OPT\cup S_i)]\geq \tfrac{1}{2}(1+m+(1-m)(1-2/k)^i) \cdot f(OPT)$.
\end{observation}
\begin{proof}
For every integer $0\leq i\leq k/\eps$ and element $u \in \cN \setminus D$, let $p_{u,i}$ denote the probability $u$ belongs to $S_i$. We would like to argue that when $i > 0$, we have $p_{u, i} \leq p_{u,i-1}(1-2/k)+1/k$. To see why this is the case, note that $u$ belongs to $S_i$ only if one of the following happens: (i) $u$ belongs to $S_{i - 1}$ and is not removed from the solution (happens with probability $p_{u, i - 1}(1 - 1/k)$ since $g_i(u_i)$ is a uniformly random element of $S_{i - 1}$), or (ii) $u$ belongs to $M_{i - 1}$ and is chosen as $u_i$ (happens with probability at most $(1 - p_{u, i})/k$). Therefore,
\begin{equation} \label{eq:p_step}
	p_{u, i}
	\leq
	p_{u, i - 1} \cdot (1 - 1/k) + (1 - p_{u, i - 1}) / k
	=
	p_{u, i - 1} \cdot (1 - 2/k) + 1/k
	\enspace.
\end{equation}

Next, we aim to prove by induction that $p_{u,i}\leq \tfrac{1}{2}(1-(1-2/k)^i)$ for every integer $0 \leq i \leq k/\eps$. For $i=0$, this is true since $u \in \cN \setminus D$ implies that $p_{u, 0} = 0 = \tfrac{1}{2}(1-(1-2/k)^0)$. Assume now that the claim holds for $i - 1$, and let us prove it for $i \geq 1$. By the induction hypothesis and Inequality~\eqref{eq:p_step},
\[
    p_{u,i}\leq p_{u,i-1}(1-2/k)+1/k\leq\tfrac{1}{2}(1-(1-2/k)^{i-1})(1-2/k)+1/k
    =\tfrac{1}{2}(1-(1-2/k)^i)
		\enspace.
\]

The observation now follows since Corollary~\ref{cor:sampling} guarantees that $\mathbb{E}[f(OPT\cup S_i)] = \mathbb{E}[f(OPT\cup (S_i \setminus D))]\geq (1-(1-m)\cdot \max_{u\in\mathcal{N} \setminus D}p_{i,u})\cdot f(OPT)$.
\end{proof}

Below we prove a lower bound on the value of the solution of Algorithm~\ref{alg:random_greedy_matroid} after any number of iterations. However, to prove this lower bound we first need to prove the following technical observation.
\begin{observation}\label{obs:randgreedmat2}
For every positive integer $i$,
 \[\left(1-\frac{2}{k}\right)^{i-1} \geq e^{-\frac{2i}{k}} - \frac{k}{i} \cdot o_k(1) \enspace.\]
\end{observation}
\begin{proof}
Note that
\begin{align*}
    e^{-\frac{2i}{k}}&=\left(e^{-\frac{2}{k}}\right)^i\leq \left(1-\frac{2}{k}+\frac{4}{k^2}\right)^i\leq\left(1-\frac{2}{k}\right)^i+\frac{4}{k^2}\cdot i\left(1-\frac{2}{k}+\frac{4}{k^2}\right)^{i-1}\\
    &\leq \left(1-\frac{2}{k}\right)^i+\frac{4i}{k^2}\left(1-\frac{1}{k}\right)^{i-1}\leq \left(1-\frac{2}{k}\right)^i +\frac{4i}{k^2}\cdot e^{-\frac{i-1}{k}}\enspace,
\end{align*}
where the third inequality holds for $k \geq 4$, and the second inequality holds since the derivative of the function $(1-2/k + x)^i$ is $i(1 - 2/k + x)^{i - 1}$, which implies
\[
	\left(1-\frac{2}{k}+\frac{4}{k^2}\right)^i
	=
	\left(1-\frac{2}{k}\right)^i + \int_0^{4/k^2} i(1 - 2/k + x)^{i - 1}dx
	\leq
	\left(1-\frac{2}{k}\right)^i + \frac{4i}{k^2}(1 - 2/k + 4/k^2)^{i - 1}dx
	\mspace{-2mu} \enspace.
\]

To complete the proof of the observation, it remains to note that, since the maximum of the function $x^2e^{-x}$ for $x \geq 0$ is $4e^{-2}$,
\[
    \frac{4i}{k^2}\cdot e^{-\frac{i-1}{k}}
    \leq
		\frac{16e^{-2}}{i}\cdot e^{\frac{1}{k}}\\
    =
		\frac{k}{i} \cdot o_k(1)
		\enspace.
		\qedhere
\]
\end{proof}

We are now ready to prove the promised lower bound on the value of the solution $S_i$ of Algorithm~\ref{alg:random_greedy_matroid} after any number of iterations.
\begin{lemma} \label{lem:S_i_lowerbound}
For every integer $0 \leq i \leq k/\eps$, \[\bE[f(S_i)] \geq \left[\frac{1+m}{4}\cdot\left(1-e^{-\frac{2i}{k}}\right)+\frac{(1-m)i}{2k}\cdot e^{-\frac{2i}{k}}-o_k(1)\right]\cdot f(OPT)\enspace.\]
\end{lemma}
\begin{proof}
For $i = 0$ the lemma follows from the non-negativity of $f$ since the right hand side of the inequality that we need to prove is non-positive for $i = 0$. Together with Observation~\ref{obs:randgreedmat2}, this implies that it suffices to prove the following inequality
\begin{equation} \label{eq:unit_inequality}
	\bE[f(S_i)] \geq \left[\frac{1+m}{4}\cdot\left(1 - \left(1-\frac{2}{k}\right)^i\right)+\frac{(1-m)i}{2k}\cdot \left(1 - \frac{2}{k}\right)^{i - 1}\right]\cdot f(OPT)
	\enspace,
\end{equation}
and the rest of the proof is devoted to this goal.

Fix an arbitrary integer $1 \leq i \leq k/\eps$. We would like to derive a lower bound on the expected marginal contribution of the element $u_i$ to the set $S_{i-1}$, and an upper bound on the expected marginal contribution of the element $g(u_i)$ to the set $S_{i-1}\setminus g(u_i)$. Let $A_{i - 1}$ be an event fixing all random choices of Algorithm~\ref{alg:random_greedy_matroid} up to iteration $i - 1$ (including), and let $\mathcal{A}_{i - 1}$ be the set of all possible $A_{i - 1}$ events. Conditioned on any event $A_{i - 1} \in \cA_{i - 1}$, the sets $S_{i - 1}$ and $M_i$ becomes deterministic, and we can define $M'_i$ as a set containing the elements of $OPT\setminus S_{i-1}$ plus enough dummy elements of $D \setminus S_{i - 1}$ to make the size of $M'_i$ exactly $k$. Then,
\begin{align*}
\mathbb{E}[f(u \mid S_{i-1}) \mid A_{i - 1}]&=\frac{\sum_{u\in M_i}f(u \mid S_{i-1})}{k}\geq  \frac{\sum_{u\in M'_i}f(u \mid S_{i-1})}{k}\\
&=\frac{\sum_{u\in OPT\setminus S_{i-1}}f(u \mid S_{i-1})}{k}\geq \frac{f(OPT\cup S_{i-1})-f(S_{i-1})}{k}\enspace,
\end{align*}
where $S_i$, $M_i$ and $M'_i$ represent here their values conditioned on $A_i$, the first inequality follows from the definition of $M_i$ and the second inequality holds by the submodularity of $f$. Similarly,
\[
    \mathbb{E}[f(g(u_i) \mid S_{i-1} - g(u_i)) \mid A_{i - 1}] = \frac{\sum_{u\in M_i}f(g(u_i) \mid S_{i-1} - g(u))}{k}\leq \frac{f(S_{i-1})-f(\varnothing)}{k}\leq \frac{f(S_{i-1})}{k}
		\enspace,
\]
where the first inequality follows from the submodularity of $f$. Taking expectation over the event $A_{i-1}$, we get
\begin{align*}
\mathbb{E}[f(u_i \mid S_{i-1})]&\geq \frac{\mathbb{E}[f(OPT\cup S_{i-1})]-\mathbb{E}[f(S_{i-1})]}{k}\nonumber\\
&\geq \frac{\tfrac{1}{2}(1+m+(1-m)(1-2/k)^{i - 1})\cdot f(OPT)-\mathbb{E}[f(S_{i-1})]}{k}
\enspace,
\end{align*}
where the last inequality is due to Observation~\ref{obs:randgreedmat}, and
\[
    \mathbb{E}[f(g(u_i) \mid S_{i-1} - g(u_i))]\leq \frac{\mathbb{E}[f(S_{i-1})]}{k} \enspace.
\]
Combing the last two inequalities now yields
\begin{align} \label{eq:step_random_greedy_matroid}
	\bE[f(S_i)]
	\geq{} &
	\bE[f(S_{i - 1} - g(u_i) + u_i)]\\\nonumber
	={} &
	\bE[f(S_{i - 1})] + \bE[f(u_i \mid S_{i - 1} - g(u_i))] - \bE[f(g(u_i) \mid S_{i - 1} - g(u_i)]\\\nonumber
	\geq{} &
	\bE[f(S_{i - 1})] + \bE[f(u_i \mid S_{i - 1})] - \bE[f(g(u_i) \mid S_{i - 1} - g(u_i)]\\\nonumber
	\geq{} &
	\left(1 - \frac{2}{k}\right) \cdot \bE[f(S_{i - 1})] + \frac{\tfrac{1}{2}(1+m+(1-m)(1-2/k)^{i - 1})\cdot f(OPT)}{k}
	\enspace,
\end{align}
where the first inequality follows from the submodularity of $f$ since $g(u_i) \neq u_i$ because $g(u_i) \in S_{i - 1}$ and $u_i \in M_i$.

Since Inequality~\eqref{eq:step_random_greedy_matroid} holds for every integer $1 \leq i \leq k / \eps$, we can use it repeatedly to get, for every integer $0 \leq i \leq k / \eps$,
\[
	\bE[f(S_i)]
	\geq
	\frac{1}{2k} \mspace{-1mu} \left[(1+m)\sum_{j=1}^{i}\left(1-\frac{2}{k}\right)^{i-j}+(1-m)\sum_{j=1}^{i}\left(1-\frac{2}{k}\right)^{i-1}\right] \mspace{-1mu}\cdot f(OPT) + \left(1 - \frac{2}{k}\right)^i \mspace{-1mu} \cdot f(S_0)
	\mspace{-2mu} \enspace.
\]
Since the non-negativity of $f$ guarantees that $f(S_0) \geq 0$, the last inequality implies Inequality~\eqref{eq:unit_inequality}, and therefore, completes the proof of the lemma.
\end{proof}

One can show that the lower bound for $f(S_i)$ proved by Lemma~\ref{lem:S_i_lowerbound} is maximized when $i = k / (1 - m)$. Unfortunately, we cannot simply plug this $i$ value into the lower bound due to two issues: this value of $i$ might not be integral, and this value of $i$ might be larger than the number $k/\eps$ of iterations. The following two lemmata prove the approximation guarantee of Theorem~\ref{thm:random_greedy_matroid} despite these issues, and together they complete the proof of the theorem.

\begin{lemma}
When $m \leq 1 - \eps$, the approximation ratio of Algorithm~\ref{alg:random_greedy_matroid} is at least \[\frac{1 + m + e^{-2 / (1 - m)}}{4} - o_k(1) \enspace.\]
\end{lemma}
\begin{proof}
Let $i' = \lfloor k / (1 - m) \rfloor$. Due to the condition of the lemma, Algorithm~\ref{alg:random_greedy_matroid} makes at least $i'$ iterations. Furthermore, since Algorithm~\ref{alg:random_greedy_matroid} makes a swap in its solution only when this swap is beneficial, the expected value of the output of the algorithm is at least
\begin{align*}
	\bE[f(S_{i'})]
	\geq{} &
	\left[\frac{1+m}{4}\cdot\left(1-e^{-\frac{2i'}{k}}\right)+\frac{(1-m)i'}{2k}\cdot e^{-\frac{2i'}{k}}-o_k(1)\right]\cdot f(OPT)\\
	\geq{} &
	\left[\frac{1+m}{4}\cdot\left(1-e^{\frac{2}{k}-\frac{2}{1 - m}}\right)+\frac{k - 1}{2k}\cdot e^{-\frac{2}{1 - m}}-o_k(1)\right]\cdot f(OPT)\\
	\geq{} &
	\left[\frac{1+m}{4}\cdot\left(1-e^{-\frac{2}{1 - m}}\right) - \frac{e^{\frac{2}{k}} - 1}{2} + \frac{1}{2}\cdot e^{-\frac{2}{1 - m}} - \frac{1}{2k} - o_k(1)\right]\cdot f(OPT)
	\enspace,
\end{align*}
where the first inequality follows from Lemma~\ref{lem:S_i_lowerbound}, and the second inequality holds since $k / (1 - m) - 1 \leq i' \leq k / (1 - m)$. Since the terms $\frac{e^{2/k} - 1}{2}$ and $\frac{1}{2k}$ are both diminishing with $k$ (and therefore, can be replaced with $o_k(1)$), the last inequality implies the lemma.
\end{proof}

\begin{lemma}
When $1 - \eps \leq m < 1$, the approximation ratio of Algorithm~\ref{alg:random_greedy_matroid} is at least \[\frac{1 + m + e^{-2 / (1 - m)}}{4} - \eps - o_k(1) \enspace,\] and when $m = 1$ the approximation ratio of this algorithm is at least $1/2 - \eps - o_k(1)$.
\end{lemma}
\begin{proof}
The output set of Algorithm~\ref{alg:random_greedy_matroid} is $f(S_{k / \eps})$. By Lemma~\ref{lem:S_i_lowerbound}, the expected value of this set is at least
\begin{align*}
	\bE[f(S_{k/\eps})&]
	\geq
	\left[\frac{1+m}{4}\cdot\left(1-e^{-\frac{2}{\varepsilon}}\right)+\frac{1-m}{2\varepsilon}\cdot e^{-\frac{2}{\varepsilon}} - o_1(k)\right]\cdot f(OPT)\\
	\geq{} &
	\left[\frac{1+m}{4}\cdot\left(1-e^{-\frac{2}{\varepsilon}}\right) - o_k(1)\right]\cdot f(OPT)
	\geq
	\left[\frac{1+m}{4}\cdot\left(1-\frac{1}{1 + 2 / \eps}\right) - o_k(1)\right]\cdot f(OPT)\\
	={} &
	\left[\frac{1 + m}{2(\eps + 2)}- o_k(1)\right]\cdot f(OPT)
	\geq
	\left[\frac{1 + m}{4} - \frac{\eps}{4} - o_k(1)\right]\cdot f(OPT)
	\enspace,
\end{align*}
where the third inequality holds since for every $x \geq 0$, $\ln(1 / (1 + x)) = \ln(1 - x / (1 + x)) \geq - \frac{x / (1 + x)}{1 - x / (1 + x)} = -x$.

The above inequality completes the proof for the case of $m = 1$. To complete the proof also for the case of $1 - \eps \leq m < 1$, it suffice to observe that in this case
\[
	e^{-2 / (1 - m)}
	\leq
	e^{-2 / \eps}
	\leq
	\frac{1}{1 + 2/\eps}
	\leq
	\frac{\eps}{2}
	\enspace.
	\qedhere
\]
\end{proof}

\subsection{Inapproximability for a Matroid Constraints} \label{ssc:inapproximability_matroids}

In this section we prove Theorem~\ref{thm:inapproximability_matroid}, which we repeat here for convenience.
\thmInapproximabilityMatroid*

The proof of Theorem~\ref{thm:inapproximability_matroid} is very similar to the proof of Theorem~\ref{thm:inapproximability_cardinality}. Recall that in Section~\ref{ssc:inapproximability_cardinality}, we proved Theorem~\ref{thm:inapproximability_cardinality} by constructing an instance $\cI$ of submodular maximization subject to a cardinality constraint, and then applying Theorem~\ref{thm:symmetry_gap} to this instance. The proof of Theorem~\ref{thm:inapproximability_matroid} is based on an instance $\cI'$ of submoduar maximization subject to a matroid constrained that is identical to $\cI$ except for the following difference. In $\cI$, the constraint is a cardinality constraint allowing the selection of up to $2$ elements from the ground set $\cN = \{a, b\} \cup \{a_i, b_i \mid i \in [r]\}$. In $\cI'$, we have instead a (simplified) partition matroid constraint allowing the selection of up to $1$ element from $\{a, b\}$ and up to $1$ element from $\{a_i, b_i \mid i \in [r]\}$.

Since the instances $\cI$ and $\cI'$ have the same objective function, the properties of this function stated in Lemma~\ref{lem:f_propeties_cardinality} apply to both of them. Furthermore, one can verify that $\cI'$ is strongly symmetric with respect to the group $\cG$ of permutation defined in Section~\ref{ssc:inapproximability_cardinality}. Therefore, we concentrate on analyzing the symmetry gap of $\cI'$.
\begin{lemma} \label{lem:gap_matroid}
The symmetry gap of $\cI'$ is at most
\begin{align*}
	&
	\max_{\substack{x \in [0, 1/2]}} \{\alpha(mx^2 + 2x - 2x^2) + 2(1 - \alpha)[1 - (1 - 1/(2r))^r](1 - (1 - m)x)\}\\
	\leq{} &
	\max_{\substack{x \in [0, 1/2]}} \{\alpha(mx^2 + 2x - 2x^2) + 2(1 - \alpha)(1 - e^{-1/2})(1 - (1 - m)x)\} + 1/(2r)
	\enspace.
\end{align*}
\end{lemma}
\begin{proof}
One possible feasible solution for $\cI'$ is the set $\{a, b_1\}$ whose value according to $f$ is $1$. Therefore, the value of the optimal solution for $\cI'$ is at least $1$. Since the symmetry gap is the ratio between the value of the best symmetric solution and the value of the best solution, to prove the lemma it remains to argue that the best symmetric solution for $\cI$ has a value of at most $\max_{\substack{x \in [0, 1/2]}} \{\alpha(mx^2 + 2x - 2x^2) + 2(1 - \alpha)[1 - (1 - 1/(2r))^r](1 - (1 - m)x)\}$.

We remind the reader that a symmetric solution for $\cI'$ is $\bar{\vy} = \bE_{\sigma \in \cG}[\vy]$ for some vector $\vy \in [0, 1]^\cN$ obeying $y_a + y_b \leq 1$ and $\sum_{i = 1}^r y_{a_i} + y_{b_i} \leq 1$. Since $a$ and $b$ can be exchanged with each other by the permutations of $\cG$, the values of the coordinates of $a$ and $b$ in $\bar{\vy}$ must be equal to each other. Similarly, every two elements of $\{a_i, b_i \in i \in [r]\}$ can be exchanged by the permutations of $\cG$, and therefore, the values of the coordinates of all these elements in $\bar{\vy}$ must be all identical. Thus, any symmetric solution $\bar{\vy}$ can be represented as
\[
	\bar{y}_u
	=
	\begin{cases}
		x & \text{if $u = a$ or $u = b$} \enspace,\\
		z & \text{if $u \in \{a_i, b_i \mid i \in [r]\}$}
	\end{cases}
\]
for some values $x \in [0, 1/2]$ and $z \in [0, 1/(2r)]$. The value of this solution (according to the multilinear extension $F$ of $f$) is
\begin{align*}
	&
	\alpha[m(1 - (1 - x)^2) + 2(1 - m)x(1 - x)] + 2(1 - \alpha)(1 - (1 - z)^r)(1 - (1 - m)x)\\
	={} &
	\alpha(mx^2 + 2x - 2x^2) + 2(1 - \alpha)(1 - (1 - z)^r)(1 - (1 - m)x)\\
	\leq{} &
	\alpha(mx^2 + 2x - 2x^2) + 2(1 - \alpha)(1 - (1 - 1/(2r))^r)(1 - (1 - m)x)
	\enspace.
\end{align*}
Therefore, one can obtain an upper bound on the value of the best symmetric solution for $\cI'$ by taking the maximum of the last expression over all the values that $x$ can take, which completes the proof of the lemma.
\end{proof}

Since any refinement of a (simplified) partition matroid constraint is a (generalized) partition matroid constraint on its own right, plugging Lemmata~\ref{lem:f_propeties_cardinality} and Lemma~\ref{lem:gap_matroid} into Theorem~\ref{thm:symmetry_gap} yields the following corollary.
\begin{corollary}
For every constant $\eps' > 0$, no polynomial time algorithm for maximizing a non-negative $m$-monotone submodular function subject to a matroid contraint obtains an approximation ratio of
\[
	\max_{\substack{x \in [0, 1/2]}} \{\alpha(mx^2 + 2x - 2x^2) + 2(1 - \alpha)(1 - e^{-1/2})(1 - (1 - m)x)\} + 1/(2r) + \eps'
	\enspace.
\]
\end{corollary}

Theorem~\ref{thm:inapproximability_matroid} now follows from the last corollary by choosing $\eps' = \eps / 2$, $r = \lceil \eps^{-1} \rceil$ and
\[
	\alpha
	=
	\argmin_{\alpha' \in [0, 1]} \max_{\substack{x \in [0, 1/2]}} \{\alpha'(mx^2 + 2x - 2x^2) + 2(1 - \alpha')(1 - e^{-1/2})(1 - (1 - m)x)\}
	\enspace.
\]

%% file: MatroidGraph.tikz
\begin{tikzpicture}
\begin{axis}[
    xlabel = {$m$},
    ylabel = {approximation ratio},
    xmin=0, xmax=1,
    ymin=0, ymax=1,
		legend cell align=left,
		legend style={at={(1.2,1.2)}}]

\addplot [name path = Greedy, domain = 0:1, orange, mark=square*, mark repeat=2] {x / 2};
\addlegendentry{Greedy}

\addplot [name path = ContinuousGreedy, domain = 0:1, blue, mark=o, mark repeat=2] {x * (1 - exp(-1)) + (1 - x) * exp(-1)};
\addlegendentry{Measured Continuous Greedy}
 
\addplot [name path = RandomGreedyMatroid, domain = 0:1, red, mark=triangle*, mark repeat=2] {(1 + x + exp(- min(2 / (1 - x), 100))) / 4};
\addlegendentry{Random Greedy for Matroids}

\addplot [name path = Inapproximability, darkgreen, mark=x, mark repeat=10] coordinates {
      (0.000000, 0.477302)
      (0.010000, 0.479402)
      (0.020000, 0.481518)
      (0.030000, 0.483650)
      (0.040000, 0.485799)
      (0.050000, 0.487964)
      (0.060000, 0.490145)
      (0.070000, 0.492343)
      (0.080000, 0.494558)
      (0.090000, 0.496790)
      (0.100000, 0.499039)
      (0.110000, 0.501305)
      (0.120000, 0.503589)
      (0.130000, 0.505890)
      (0.140000, 0.508208)
      (0.150000, 0.510545)
      (0.160000, 0.512900)
      (0.170000, 0.515272)
      (0.180000, 0.517663)
      (0.190000, 0.520073)
      (0.200000, 0.522501)
      (0.210000, 0.524947)
      (0.220000, 0.527413)
      (0.230000, 0.529897)
      (0.240000, 0.532400)
      (0.250000, 0.534923)
      (0.260000, 0.537465)
      (0.270000, 0.540027)
      (0.280000, 0.542608)
      (0.290000, 0.545209)
      (0.300000, 0.547830)
      (0.310000, 0.550471)
      (0.320000, 0.553132)
      (0.330000, 0.555814)
      (0.340000, 0.558516)
      (0.350000, 0.561238)
      (0.360000, 0.563982)
      (0.370000, 0.566746)
      (0.380000, 0.569531)
      (0.390000, 0.572338)
      (0.400000, 0.575166)
      (0.410000, 0.578015)
      (0.420000, 0.580885)
      (0.430000, 0.583778)
      (0.440000, 0.586691)
      (0.450000, 0.589627)
      (0.460000, 0.592585)
      (0.470000, 0.595565)
      (0.480000, 0.598567)
      (0.490000, 0.601592)
      (0.500000, 0.604639)
      (0.510000, 0.607708)
      (0.520000, 0.610800)
      (0.530000, 0.613915)
      (0.540000, 0.617052)
      (0.550000, 0.620213)
      (0.560000, 0.623396)
      (0.570000, 0.626603)
      (0.580000, 0.629833)
      (0.590000, 0.633086)
      (0.600000, 0.636362)
      (0.610000, 0.639662)
      (0.620000, 0.642985)
      (0.630000, 0.646332)
      (0.640000, 0.649702)
      (0.650000, 0.653096)
      (0.660000, 0.656514)
      (0.670000, 0.659955)
      (0.680000, 0.663421)
      (0.690000, 0.666910)
      (0.700000, 0.670423)
      (0.710000, 0.673960)
      (0.720000, 0.677521)
      (0.730000, 0.681106)
      (0.740000, 0.684259)
      (0.750000, 0.686751)
      (0.760000, 0.689242)
      (0.770000, 0.691734)
      (0.780000, 0.694225)
      (0.790000, 0.696717)
      (0.800000, 0.699208)
      (0.810000, 0.701699)
      (0.820000, 0.704191)
      (0.830000, 0.706682)
      (0.840000, 0.709174)
      (0.850000, 0.711665)
      (0.860000, 0.714157)
      (0.870000, 0.716648)
      (0.880000, 0.719140)
      (0.890000, 0.721631)
      (0.900000, 0.724122)
      (0.910000, 0.726614)
      (0.920000, 0.729105)
      (0.930000, 0.731597)
      (0.940000, 0.734088)
      (0.950000, 0.736580)
      (0.960000, 0.739071)
      (0.970000, 0.741563)
      (0.980000, 0.744054)
      (0.990000, 0.746545)
      (1.000000, 0.749037)
};
\addlegendentry{Inapproximability}

\addplot [name path = OldInapproximability, domain = 0:1, dashed] {1 - exp(-1)};

\end{axis}
\end{tikzpicture}

%% file: Experiments.tex
\section{Applications and Experiment Results} \label{sec:experiments}
Many machine learning applications require optimization of non-monotone submodular functions subject to some constraint. Unfortunately, such functions enjoy relatively low approximation guarantees. Nevertheless, in many cases the non-monotone objective functions have a significant monotone component that can be captured by the monotonicity ratio. In this section, we discuss three concrete applications with non-monotone submodular objective functions. For each application we  provide a lower bound on the monotonicity ratio $m$ of the objective function, which translates via our results from the previous sections into an improved approximation guarantee for the application.

To demonstrate the value of our improved guarantees in experiments, we took the following approach. The output of an approximation algorithm provides an upper bound on the value of the optimal solution for the problem (formally, this upper bound is the value of the output over the approximation ratio of the algorithm). Thus, we plot in each experiment the upper bound on the value of the optimal solution obtained with and without taking into account the monotonicity ratio, which gives a feeling of how the magnitude of our improvements compare to other values of interest (such as the gaps between the performances of the algorithms considered). 


\subsection{Personalized Movie Recommendation} \label{ssc:movie_recomendation}

The first application we consider is Personalized Movie Recommendation. Consider a movie recommendation system where each user specifies what genres she is interested in, and the system has to provide a representative subset of movies from these genres. Assume that each movie is represented by a vector consisting of users' ratings for the corresponding movie. One challenge here is that each user does not necessarily rate all the movies, hence, the vectors representing the movies do not necessarily have similar sizes. To overcome this challenge, a low-rank matrix completion techniques~\cite{candes2009exact} can be performed on the matrix with missing values in order to obtain a complete rating matrix. Formally, given few ratings from $k$ users to $n$ movies we obtain in this way a rating matrix $\mM$ of size $k \times n$. Following~\cite{mirzasoleiman2016fast}, to score the quality of a selected subset of movies, we use the following function.
\begin{equation}\label{movies}
    f(S)=\sum_{u \in \cN}\sum_{v\in S} s_{u,v}-\lambda\sum_{u\in S}\sum_{v\in S}s_{u,v}
		\enspace.
\end{equation}
Here, $\cN$ is the set of $n$ movies, $\lambda \in [0, 1]$ is a parameter and $s_{u,v}$ denotes the similarity between movies $u$ and $v$ (the similarity $s_{u, v}$ can be calculated based on the matrix $\mM$ in multiple ways: cosine similarity, inner product, etc). Note that the first term in the definition of $f$ captures the coverage, while the second term captures diversity. Thus, the parameter $\lambda$ denotes the importance of diversity in the returned subset.

One can verify that the above defined function $f$ is non-negative and submodular. The next theorem analyzes the monotonicity ratio of this function. In this theorem we assume that the similarity scores $s_{u, v}$ are non-positive and  obey the equality $s_{u, v} = s_{v, u}$ for every $u, v \in \cN$. Note that the above mentioned ways to define these scores have these properties. Interestingly, it turns out that the function $f$ is monotone when $\lambda$ is small enough despite the fact that this function is traditionally treated as non-monotone (e.g., in~\cite{feldman2017greed,mirzasoleiman2016fast}). This is a welcomed unexpected result of the use of the monotonicity ratio, which required us to really understand the degree of non-monotonicity represented by the objective function.
\begin{theorem} \label{thm:movie_recommendation_monotonicity}
The objective function $f$ is monotone for $0 \leq \lambda \leq \nicefrac{1}{2}$ and $2(1 - \lambda)$-monotone for $\nicefrac{1}{2} \leq \lambda \leq 1$.
\end{theorem}
\begin{proof}
We first prove the first part of the theorem. Thus, we assume $\lambda \leq \nicefrac{1}{2}$, and we need to show that for arbitrary set $S \subseteq \cN$ and element $u \in \cN \setminus S$ the marginal contribution $f(u \mid S)$ is non-negative. This holds because
\begin{align*}
	f(u \mid S)
	={} &
	\sum_{v \in \cN} s_{v, u} - \lambda \left[\sum_{v \in S} s_{u,v} + \sum_{v \in S} s_{v,u} + s_{u, u} \right]\\
	={} &
	\sum_{v \in \cN} s_{v, u} - \lambda \left[2 \sum_{v \in S} s_{v,u} + s_{u, u} \right]
	\geq
	\sum_{v \in \cN} s_{v, u} - \sum_{v \in S} s_{v,u} - s_{u, u}
	\geq
	0
	\enspace,
\end{align*}
where the second equality holds because $s_{u, v} = s_{v, u}$, and the first inequality holds since $\lambda \leq \nicefrac{1}{2}$ in the case we consider and the $s_{u, v}$ values are non-negative.

It remains to prove the second part of the theorem. Thus, we assume from now on $\lambda \in [\nicefrac{1}{2}, 1]$, and we consider two sets $S \subseteq T \subseteq \cN$. To prove the theorem we need to show that $f(T) \geq 2(1 - \lambda) \cdot f(S)$. The first step towards showing this is to prove the following lower bound on $f(S)$.
\begin{align} \label{eq:f_S_lower_bound}
	f(S)
	={} &
	2(1 - \lambda) \cdot f(S) + (2\lambda - 1) \cdot \left[\sum_{u \in \cN} \sum_{v \in S} s_{u,v} - \lambda \sum_{u \in S} \sum_{v \in S} s_{u,v}\right]\\\nonumber
	\geq{} &
	2(1 - \lambda) \cdot f(S) + (2\lambda - 1) \cdot  \sum_{u \in T \setminus S} \sum_{v \in S} s_{u,v}
	=
	2(1 - \lambda) \cdot f(S) + (2\lambda - 1) \cdot  \sum_{u \in S} \sum_{v \in T \setminus S} \mspace{-9mu} s_{u,v}
	\enspace,
\end{align}
where the inequality holds since $\lambda \leq 1$, and the second equality holds since $s_{u,v} = s_{v, u}$. Using this lower bound, we now get
\begin{align*}
	f(T)
	={} &
	f(S) + \sum_{u \in \cN} \sum_{v \in T \setminus S} \mspace{-9mu} s_{u,v} - \lambda \left[\sum_{u \in S} \sum_{v \in T \setminus S} \mspace{-9mu} s_{u, v} + \sum_{u \in T \setminus S} \sum_{v \in S} s_{u, v} + \sum_{u \in T \setminus S} \sum_{v \in T \setminus S} \mspace{-9mu} s_{u, v}\right]\\
	={} &
	f(S) + \sum_{u \in \cN} \sum_{v \in T \setminus S} \mspace{-9mu} s_{u,v} - \lambda \left[2\sum_{u \in S} \sum_{v \in T \setminus S} \mspace{-9mu} s_{u, v} + \sum_{u \in T \setminus S} \sum_{v \in T \setminus S} \mspace{-9mu} s_{u, v}\right]\\
	\geq{} &
	f(S) + (1 - 2\lambda) \cdot \sum_{u \in S} \sum_{v \in T \setminus S} \mspace{-9mu} s_{u, v}
	\geq
	2(1 - \lambda) \cdot f(S)
	\enspace,
\end{align*}
where the first inequality holds since $\lambda \leq 1$, and the second inequality holds by Inequality~\eqref{eq:f_S_lower_bound}.
\end{proof}

To demonstrate the value of our lower bound on the monotonicity ratio, we followed the experimental setup of the prior work~\cite{mirzasoleiman2016fast} and used a subset of movies from the MovieLens data set~\cite{harper2015movielens} which includes \num[group-separator={,}]{10437} movies. Each movie in this data set is represented by a $25$ dimensional feature vector calculated using user ratings, and we used the inner product similarity to obtain the similarity values $s_{u,v}$ based on these vectors.

In our experiment we employed accelerated versions of the algorithms analyzed in Section~\ref{sec:cardinality} for a cardinality constraint. Specifically, instead of the Greedy algorithm we used Threshold Greedy~\cite{badanidiyuru2014fast} and Sample Greedy~\cite{mirzasoleiman2015lazier}; and instead of Random Greedy we used a threshold based version of this algorithm due to~\cite{buchbinder2017comparing} that we refer to as Threshold Random Greedy.\footnote{Algorithm~6 in~\cite{buchbinder2017comparing}.} All three algorithms had almost identical performance in our experiments (see Appendix~\ref{sec:extra}), and therefore, to avoid confusion, in the plots of this section we draw only the output of Threshold Random Greedy.

Each plot of Figure~\ref{moviecard} depicts the outputs of Threshold Random Greedy and a scarecrow algorithm called Random that simply outputs a random subset of movies of the required size. Each point in the plots represents the average value of the outputs of $10$ executions of these algorithms. We also depict in each plot the upper bound on the value of the optimal solution based on the general approximation ratio of Random Greedy and the improved approximation ratio implied by Theorems~\ref{thm:random_greedy} and~\ref{thm:movie_recommendation_monotonicity}---the area between the two upper bounds is shaded. In Figure~\ref{fig:f1} we plot these values for the case in which we asked the algorithms to pick at most $10$ movies, and we vary the parameter $\lambda$. In Figures~\ref{fig:f2} amd~\ref{fig:f3} we plotted the same values for a fixed parameter $\lambda$, while varying the maximum cardinality (number of movies) allowed for the output set. Since the height of the shaded area is on the same order of magnitude as the values of the solutions produced by Threashold Random Greedy (especially when $\lambda$ is close to $1/2$), our results demonstrate that the improved upper bound we are able to prove is much tighter than the state-of-the-art. Furthermore, our improved upper bound shows that the gap between the empirical outputs of Threshold Random Greedy and Random is much more significant as a percentage of the value of the optimal solution than one could believe based on the weaker bound.

\begin{figure}[tb]
  \begin{subfigure}[t]{0.3\textwidth}
    \includegraphics[width=\textwidth]{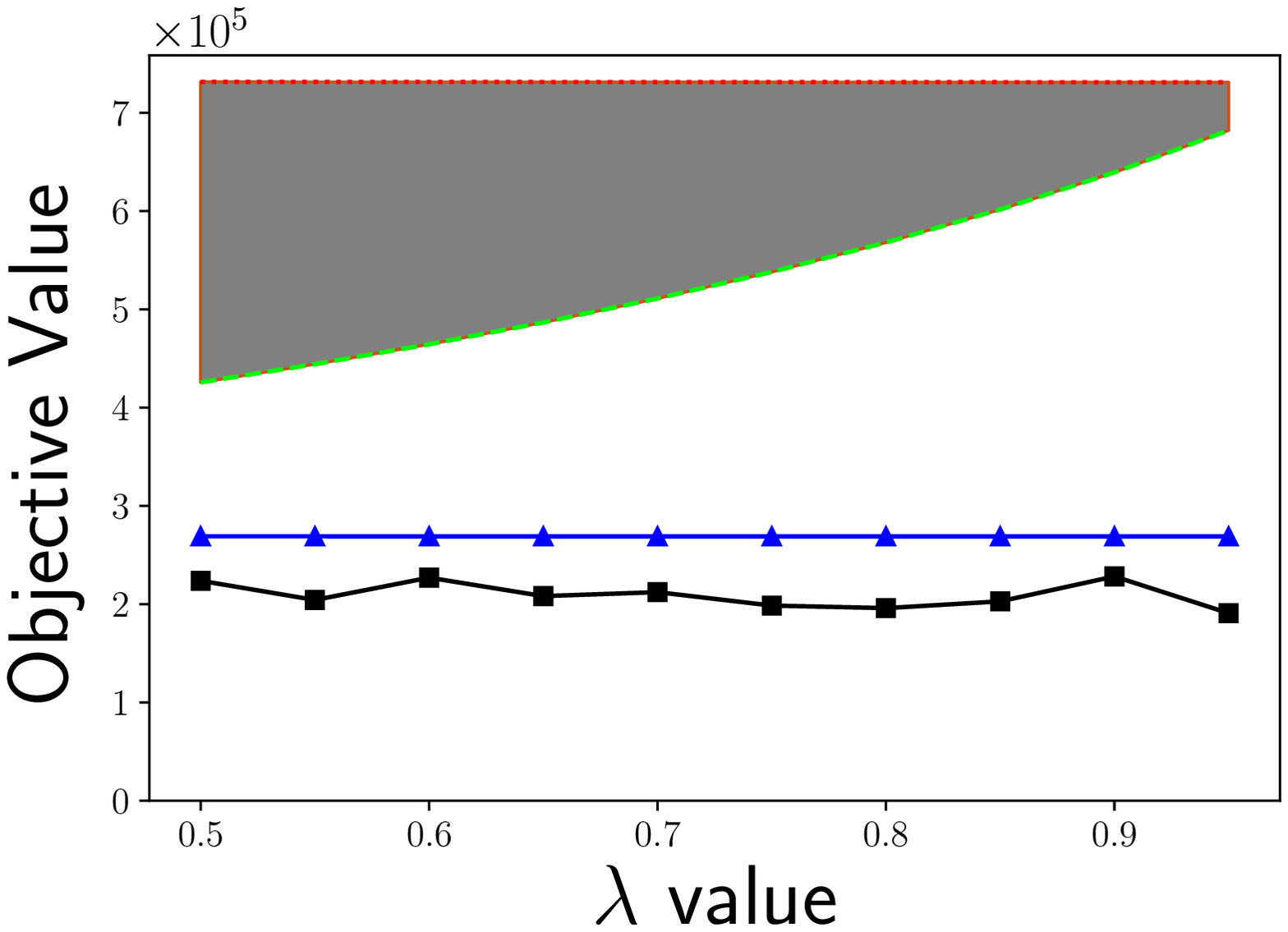}
    \caption{Results when up to $10$ movies can be selected for varying $\lambda$.}
    \label{fig:f1}
  \end{subfigure}
  \hfill
  \begin{subfigure}[t]{0.3\textwidth}
    \includegraphics[width=\textwidth]{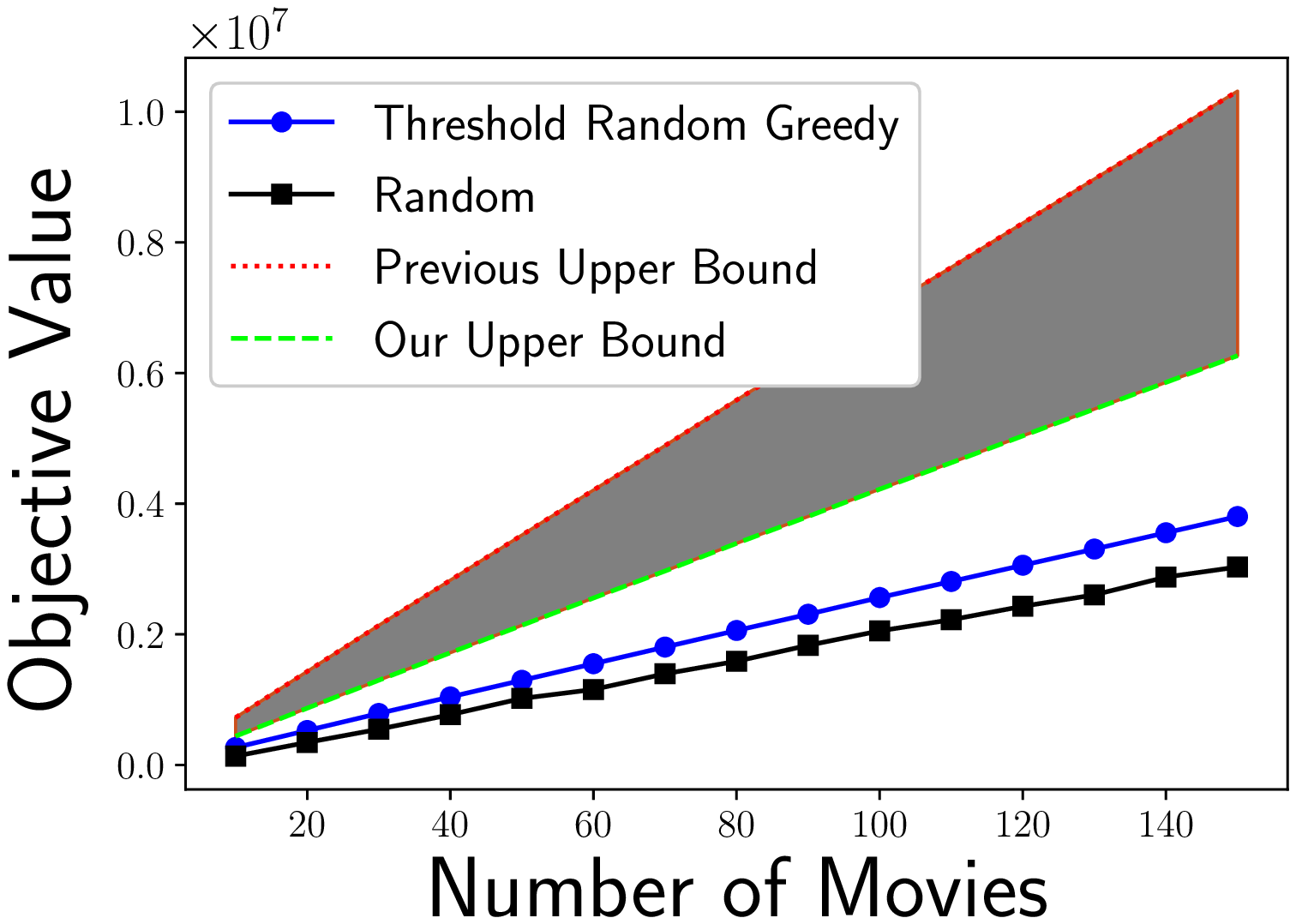}
    \caption{Results for $\lambda = 0.55$ when the maximum number movies in the solution varies.}
    \label{fig:f2}
  \end{subfigure}
  \hfill
    \begin{subfigure}[t]{0.3\textwidth}
    \includegraphics[width=\textwidth]{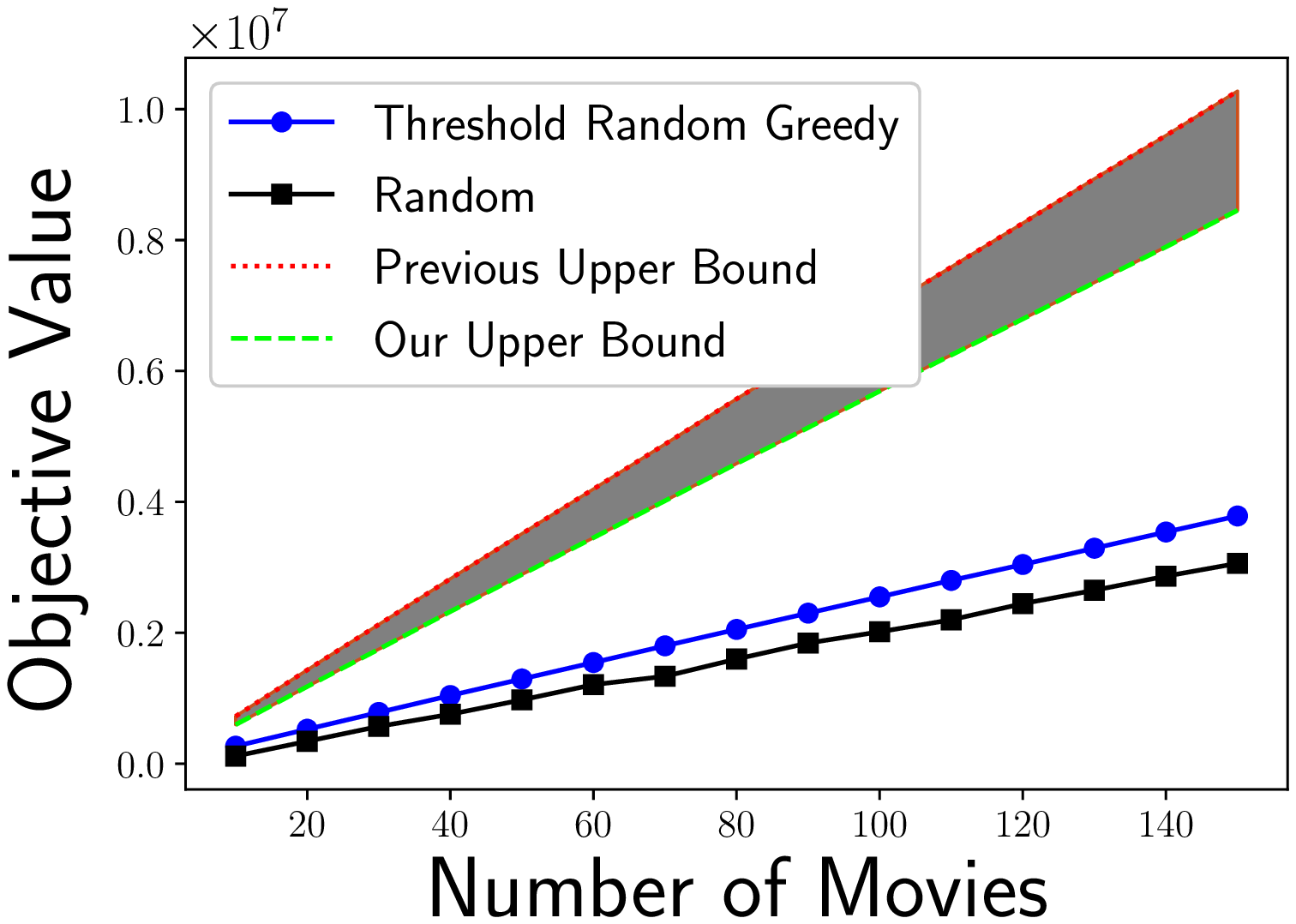}
    \caption{Results for $\lambda = 0.85$ when the maximum number movies in the solution varies.}
    \label{fig:f3}
  \end{subfigure}
  \caption{Experimental results for Personalized Movie Recommendation. Each plot includes the output of the algorithms we consider as well the previous and improved upper bounds on the optimal value (the area between these two bounds is shaded).}\label{moviecard}
\end{figure}

\subsection{Personalized Image Summarization}
Consider a setting in which we get as input a collection $\cN$ of images from $\ell$ disjoint categories (e.g., birds, dogs, cats) and the user specifies $r \in [\ell]$ categories, and then demands a subset of the images in these categories that summarize all the images of the categories. Following~\cite{mirzasoleiman2016fast} again, we use the following function to evaluate a given subset of images.
\begin{align}\label{func:mov}
    f(S)=\sum_{u\in \cN}\max_{v\in S}s_{u,v} - \frac{1}{|\cN|}\sum_{u\in S}\sum_{v\in S}s_{u,v}
		\enspace,
\end{align}
where $s_{u,v}$ is a non-negative similarity between images $u$ and $v$.

One can verify that the above function $f$ is non-negative and submodular. Unfortunately, however, this function can have a very low monotonicity ratio. To compensate for this, we observe that most the analyses we described in the previous sections use the monotonicity ratio only to show that $f(S \cup T) \geq m \cdot f(S)$ for sets $S$ and $T$ that are feasible. This motivates the following weak version of the monotonicity ratio. We note that many continuous properties of set functions have such weak versions. For example, the original paper presenting the submodularity-ratio~\cite{das2019approximate} presented in fact the weak version of this property, and the non-weak version was only formulated at a later point.
\begin{definition}
Consider the problem of maximizing a non-negative function $f$ subject to some constraint. In the context of this problem, we say that $f$ is \emph{$m$-weakly monotone} if for every two feasible sets $S$ and $T$ it holds that $f(S \cup T) \geq m \cdot f(S)$. Furthermore, the \emph{weak monotonicity ratio} of the problem is the maximum value $m$ for which the above holds.
\end{definition}

\begin{theorem} \label{thm:image_summerization}
The objective function $f$ given by Equation~\eqref{func:mov} is $1-\frac{2k}{|\cN|}$-weakly monotone when the size of feasible solutions is at most $k$ for some $1 \leq i \leq |\cN|$.
\end{theorem}
\begin{proof}
When $k \geq |\cN| / 2$, the theorem is trivial. Thus, we can assume below $k < |\cN| / 2$.
Consider two feasible sets $S,T\in\mathcal{N}$, and let us lower bound $f(S\cup T)$.
\begin{align*}
    f(S\cup T) &=\sum_{u\in \cN}\max_{v\in S\cup T} s_{u,v}-\frac{1}{|\cN|}\sum_{u\in S\cup T}\sum_{v\in S\cup T}s_{u,v}\\
    &\geq \sum_{u\in \cN}\max_{v\in S\cup T}s_{u,v}-\frac{|S\cup T|}{|\cN|}\sum_{u\in S\cup T}\max_{v\in S\cup T}s_{u,v}
    \geq \sum_{u\in \cN}\max_{v\in S\cup T}s_{u,v}-\frac{2k}{|\cN|}\sum_{u\in S\cup T}\max_{v \in S\cup T}s_{u,v}\\
    &=\left(1-\frac{2k}{|\cN|}\right)\sum_{u\in \cN}\max_{v\in S\cup T}s_{u,v}
    \geq \left(1-\frac{2k}{|\cN|}\right)\sum_{u \in \cN}\max_{v\in S}s_{u,v}\enspace.
\end{align*}
Using this lower bound, we now get
\[
    f(S) =\sum_{u\in E}\max_{v\in S}s_{u,v}-\frac{1}{|\cN|}\sum_{u\in S}\sum_{v\in S}s_{u,v}
    \leq \sum_{u\in E}\max_{v\in S}s_{u,v}
		\leq
		\frac{f(S\cup T)}{1 - 2k/|\cN|}
		\enspace,
\]
which completes the proof of the theorem since $S$ and $T$ have been chosen as arbitrary feasible sets. 
\end{proof}

Our experiments for this setting are based on a subset of the CIFAR-$10$ dataset~\cite{krizhevsky2009learning} which includes \num[group-separator={,}]{10000} Tiny Images. These images belong to $10$ classes, with $1000$ images per class. Each image consists of $32\times 32$ RGB pixels represented by a \num[group-separator={,}]{3072} dimensional vector. To compute the similarity $s_{u,v}$ between images, we used the dot product.

In our first experiment, we simply looked for a summary consisting of a limited number of images. Since this is a cardinality constraint, we again used the scarecrow algorithm Random and the accelerated versions mentioned in Section~\ref{ssc:movie_recomendation} of the algorithms from Section~\ref{sec:cardinality}. In Figure~\ref{img:f1} we depict the outputs of Threshold Random Greedy and Random for various limits on the number of images in the summary (like in Section~\ref{ssc:movie_recomendation} we omit the other non-scarecrow algorithms since their performance is essentially identical to the one of Threshold Random Greedy, and we refer the reader to Appendix~\ref{sec:extra} for more detail). Figure~\ref{img:f1} also includes the upper bounds on the optimal solution obtained via the previous approximation ratio for Random Greedy and our improved approximation ratio (the area between the two upper bounds is shaded). We can see that the upper bound obtained via our improved approximation ratio is much tighter, and this upper bound also demonstrates that the gap between the non-scarecrow and the scarecrow algorithms is significant compared to the optimal solution.

In our second experiment, we looked for a summary containing up to $k$ images from each category selected by the user for some parameter $k$ (we assumed in the experiment that the user chose the categories: ``airplane'', ``automobile'' and ``bird''). Since this is a (generalized partition) matroid constraint, in this experiment we used versions of the algorithms from Section~\ref{sec:matroid}. Specifically, we used Random Greedy for Matroids and an accelerated version of Measured Continuous Greedy based on the acceleration technique underlying the Accelerated Continuous Greedy of~\cite{badanidiyuru2014fast}. Additionally, we used in this experiment a scarecrow algorithm called Random that outputs a set containing a random selection of $k$ images from each one of the chosen categories. The values of the outputs of all these algorithms are depicted in Figure~\ref{img:f2} (values shown are averaged over $10$ executions).

Figure~\ref{img:f2} also includes two upper bounds on the value of the optimal solution. The previous upper bound is an upper bound computed based on the previously known approximation ratios of Random Greedy for Matroids and Measured Continuous Greedy. In contrast, our upper bound is computed based on the approximation ratios proved in Theorems~\ref{thm:continuous_greedy} and~\ref{thm:random_greedy_matroid} and the weak monotonicity ratio proved in Theorem~\ref{thm:image_summerization}.\footnote{From a purely formal point of view this upper bound is not fully justified since Measured Continuous Greedy is a rare example of an algorithm whose analysis cannot use in a black box fashion the weak monotonicity ratio instead of the monotonicity ratio. However, due to probabilistic concentration, we expect the upper bound to still hold up to at most a small error.} As is evident from the similarity between Figures~\ref{img:f1} and~\ref{img:f2}, our observations from the first experiment extend also the more general constraint considered in the current experiment.

\begin{figure}[tb]
  \begin{subfigure}[t]{0.45\textwidth}
    \includegraphics[width=\textwidth]{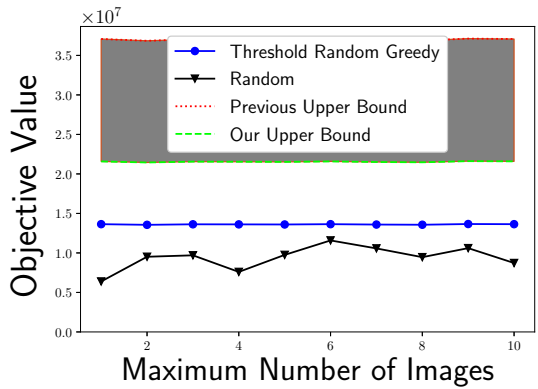}
    \caption{Results for Personalized Image Summarization with a cardinality constraint for varying number of images in the summary produced.}
    \label{img:f1}
  \end{subfigure}
  \hfill
  \begin{subfigure}[t]{0.45\textwidth}
    \includegraphics[width=\textwidth]{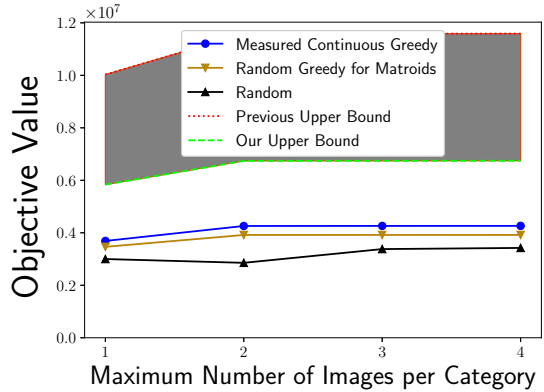}
    \caption{Results for Personalized Image Summarization with a matroid constraint. The $x$-axis gives the maximum number $k$ of images allowed from each category.}
    \label{img:f2}
  \end{subfigure}
 \caption{Personalized Image Summerization Results} \label{image}
\end{figure}

\subsection{Quadratic Programming} \label{ssc:quadratic_programming}

Consider the function
\begin{align}
\label{exp:quad}
F(\vx)=\frac{1}{2}\vx^T\mH\vx+\vh^T\vx+c
\enspace. 
\end{align}
By choosing appropriate matrix $\mH$, vector $\vh$ and scalar $c$, this function can be made to have various properties. Specifically, we would like to make it non-negative and DR-submodular (DR-submodularity is an extension of submodularity to continuous functions---see Appendix~\ref{sec:DR-sub} for more detail). Our goal in this section is to maximize $F$ under a polytope constraint given by
\[
	P=\{\vx\in \nnRE{n} \mid A\vx\leq \vb, \vx\leq \vu,A\in\nnRE{m\times n}, \vb\in \nnRE{m}\}
\]
for some dimensions $n$ and $m$.

Following Bian et al.~\cite{bian2017nonmonotone}, we set $m = n$, choose the matrix $\mH\in \mathbb{R}^{n\times n}$ to be a randomly generated symmetric matrix whose entries are drawn uniformly at random (and independently) from $[-1,0]$, and choose $\mA\in\mathbb{R}^{m\times n}$ to be a randomly generated matrix whose entries are drawn uniformly at random from $[v,v+1]$ for $v=0.01$ (this choice of $v$ guarantees that the entries of $A$ are strictly positive). We also set $\vb=\vone$ (i.e., $\vb$ is the all ones vector), and $\vu$ to be the upper bound on $P$ given by $u_j=\min_{j\in[m]} b_i / A_{i,j}$ for every $j\in[n]$. Finally, we set $\vh = -\beta\cdot \mH^T\vu$ for a parameter $\beta>0$ (in~\cite{bian2017nonmonotone} $\beta$ was fixed to $0.1$).

The non-positivity of $\mH$ guarantees that $f$ is DR-submodular. To make sure that $f$ is also non-negative, the value of $c$ should be at least $-\min_{\vzero \leq \vx \leq \vu}\frac{1}{2}\vx^T\mH\vx + \vh^T\vx$ (where $\vzero$ is the all zeros vector). This value can be approximately obtained by using \quadprogIP\footnote{We used IBM CPLEX optimization studio \url{https://www.ibm.com/products/ilog-cplex-optimization-studio}.}~\cite{xia2020globally}. Let the value of this minimum be $M$; then we set $c= -M+\alpha|M|$ for some parameter $\alpha > 0$. 

The definition of the monotonicity ratio can be extend to the continuous setting we consider in this section as follows.
\[ m=\inf_{\vzero \leq \vx \leq \vy \leq \vu}\frac{F(\vy)}{F(\vx)}\enspace,\]
where the ratio $F(\vy) / F(\vx)$ should be understood to have a value of $1$ whenever $F(\vx) = 0$. The following theorem analyzes the monotonicity ratio of the function given in Equation~\eqref{exp:quad} based on this definition.
\begin{theorem} \label{thm:quadratic}
For $\beta \in (0, 1/2)$, the objective function $F$ given by Equation~\eqref{exp:quad} is $\frac{(1-2\beta)\cdot\alpha}{1+\alpha}$-monotone. Furthermore, when $\min_{\vzero \leq x \leq \vu}(\frac{1}{2}\vx^T\mH\vx +\vh\vx) \geq 0$, $F$ is even $(1-2\beta)$-monotone.
\end{theorem}
\begin{proof}
Fix two vectors $\vzero \leq \vx \leq \vy \leq \vu$. We begin this proof by providing a lower bound on $F(\vy)$ and an upper bound on $F(\vx)$. The lower bound on $F(\vy)$ is as following.
\begin{align*}
    F(\vy)=\frac{1}{2}\vy^T\mH\vy+\vh^T\vy+c \geq \min_{\vzero \leq \vx \leq \vu}\left(\frac{1}{2}\vx^T\mH\vx+\vh\vx\right)+c\enspace.
\end{align*}
To get the upper bound on $F(\vx)$, we first need to prove an upper bound on $c$.
\begin{align*}
    c\geq -\min_{0 \leq \vx \leq \vu}\left(\frac{1}{2}\vx^T\mH\vx + \vh^T\vx\right)=-\min_{0 \leq \vx \leq \vu}\left(\frac{1}{2}\vx^T\mH\vx-\beta \vu^T\mH \vx\right)\geq -\left(\frac{1}{2}-\beta\right)\vu^T\mH\vu\enspace.
\end{align*}
The promised upper bound on $F(\vx)$ now follows.
\[
    F(\vx)=\frac{1}{2}\vx^T\mH\vx+\vh^T\vx + c\leq \vh^T\vx+c\leq \vh^T\vu+c=
    -\beta \vu^T\mH\vu+c\leq \frac{\beta c}{1/2-\beta}+c = \frac{c}{1-2\beta}\enspace,
\]
where the first inequality holds since $\mH$ is non-positive, and the second inequality holds since $\vh$ is non-negative.

Recall now that $c= -M+\alpha|M|$, which implies
\[
	\min_{0 \leq \vx \leq \vu}\left(\frac{1}{2}\vx^T\mH\vx + \vh^T\vx\right)
	=
	M
	\geq
	-\frac{c}{1+\alpha}
	\enspace,
\]
and therefore,
\[
	F(\vy)
	\geq
	-\frac{c}{1+\alpha} + c
	=
	\frac{c\alpha}{1+\alpha}
	\geq
	\frac{(1 - 2\beta)\alpha}{1+\alpha} \cdot F(\vx)
	\enspace.
\]

It remains to consider the case in which $\min_{\vzero \leq \vx \leq \vu}\left(\frac{1}{2}\vx^T\mH\vx+\vh\vx\right)\geq 0$. In this case
\[
	F(\vy)
	\geq
	c
	\geq
	(1 - 2\beta) \cdot F(\vx)
	\enspace.
	\qedhere
\]
\end{proof}

We applied the Non-monotone Frank-Wolfe algorithm of Bian et al.~\cite{bian2017nonmonotone} to the above defined optimization problem (Non-monotone Frank-Wolfe algorithm is related to the Measured Continuous Greedy algorithm studied in Section~\ref{ssc:continuous_greedy}, and we refer the reader to Appendix~\ref{sec:DR-sub} for further detail about this algorithm and its analysis). Figure~\ref{fig:quad} depicts the results we obtained. Specifically, Figure~\ref{fig:quad2} shows the value of the solution obtained by Non-monotone Frank-Wolfe for $\alpha = 0.3$ and $\beta = 0.2$ as the dimensionality $n$ varies. The shaded area is the area between the previous upper bound on the optimal value (that ignores the monotonicity ratio), and our upper bound that takes advantage of the monotonicity ratio bound given by Theorem~\ref{thm:quadratic}. Figures~\ref{fig:quad1} and~\ref{fig:quad3} are similar, but they fix the dimensionality $n$ to be $4$, and vary $\alpha$ or $\beta$ instead. Let us discuss now some properties of Figure~\ref{fig:quad}.
\begin{itemize}
	\item Each data point in Figure~\ref{fig:quad} corresponds to a single instance drawn from the distribution described above. This implies that the plots in Figure~\ref{fig:quad} vary for different runs of our experiment, but the plots that we give represent a (single) typical run.
	\item The size of the the shaded area depends on $\alpha$ and $\beta$, but also on the sign of $\min_{\vzero \leq \vx \leq \vu}(\frac{1}{2}\vx^T\mH\vx +\vh\vx)$. This is the reason that this size behaves somewhat non-continuously in Figure~\ref{fig:quad3}. Interestedly, the sign of this minimum is mostly a function of $\beta$. In other words, there are values of $\beta$ for which the minimum is non-negative with high probability, and other values for which the minimum is negative with high probability.
	\item One can see that the use of the monotonicity ratio significantly improves the upper bound on the optimal value, especially when the minimum $\min_{\vzero \leq \vx \leq \vu}(\frac{1}{2}\vx^T\mH\vx +\vh\vx)$ happens to be non-negative.
\end{itemize}

\begin{figure}[tb]
\begin{subfigure}[t]{0.32\textwidth}
  \includegraphics[width=\linewidth]{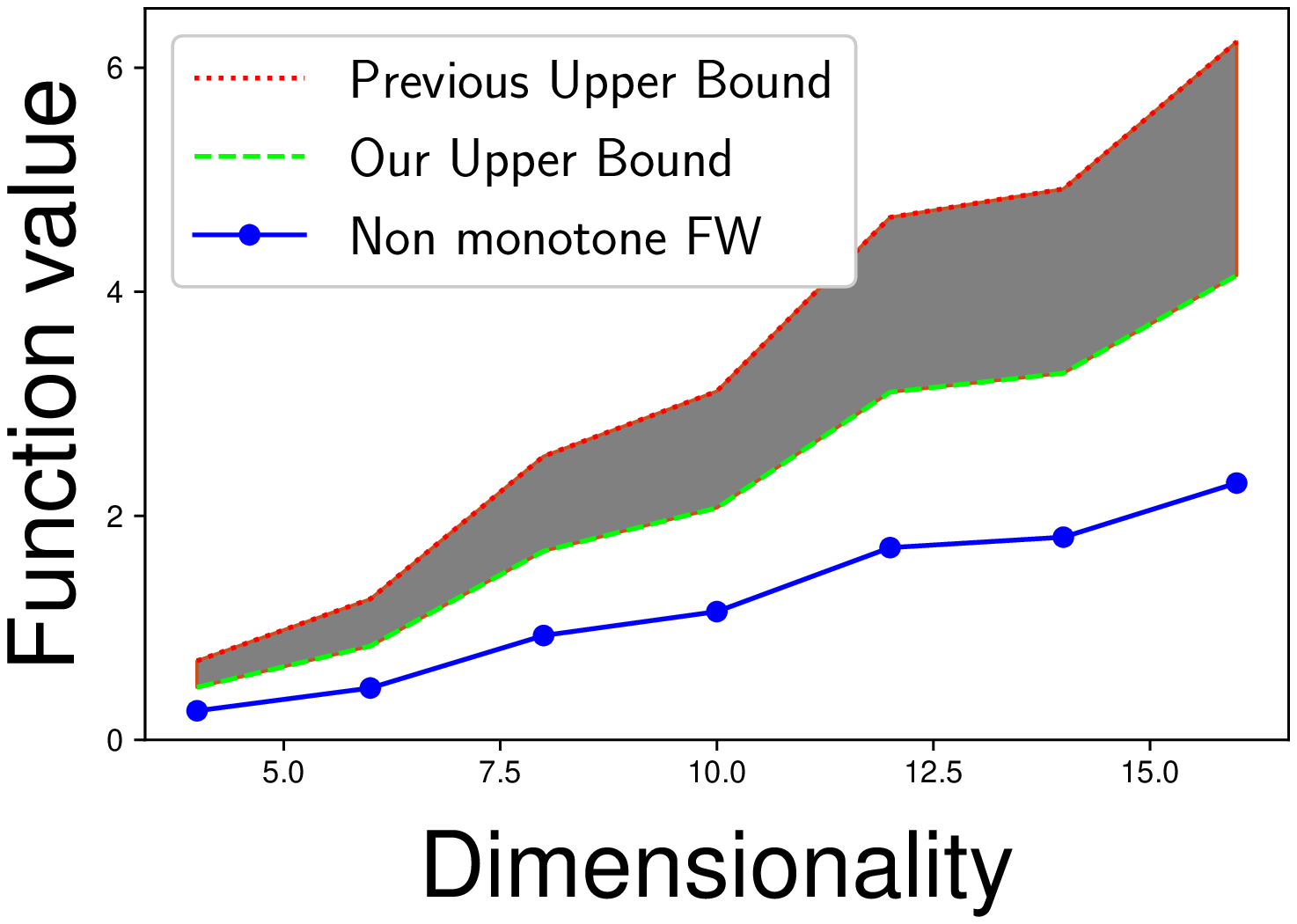}
  \caption{Varying the dimensionality $n$ for fixed $\alpha=0.3$ and $\beta=0.2$.}\label{fig:quad2}
\end{subfigure}\hfill
\begin{subfigure}[t]{0.32\textwidth}
  \includegraphics[width=\linewidth]{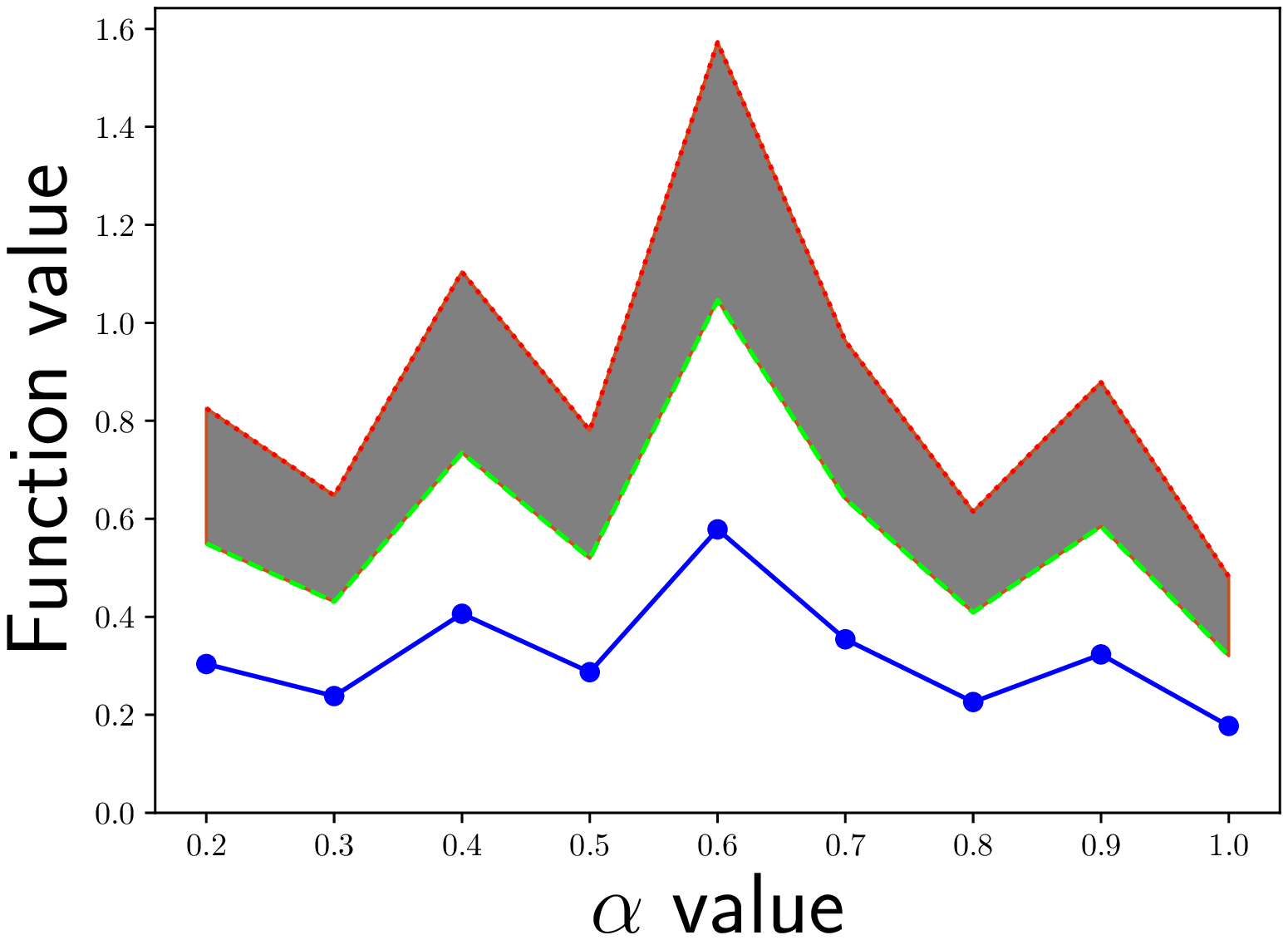}
  \caption{Varying $\alpha$ for fixed $\beta = 0.2$ and $n = 4$.}\label{fig:quad1}
\end{subfigure}\hfill
\begin{subfigure}[t]{0.32\textwidth}
  \includegraphics[width=\linewidth]{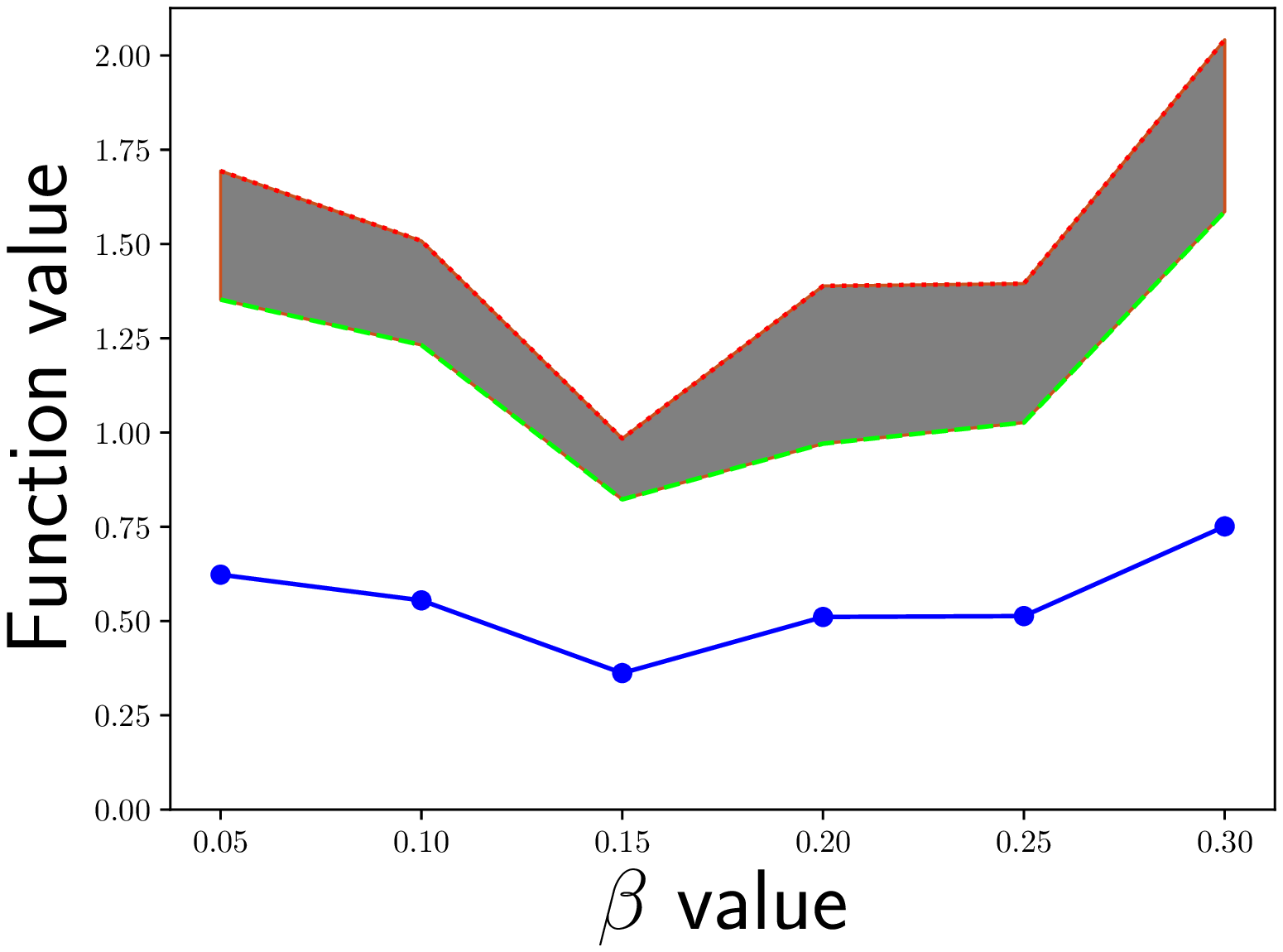}
  \caption{Varying $\beta$ for fixed $\alpha = 0.5$ and $n = 4$.}\label{fig:quad3}
\end{subfigure}
\caption{Results of the Quadratic Programming experiments.} \label{fig:quad}
\end{figure}

%% file: Conclusion.tex
\section{Conclusion}

In this paper we have defined the monotonicity ratio, analyzed how the approximation ratios of standard submodular maximization algorithms depend on this ratio, and then demonstrated that this leads to improved approximation guarantees for the applications of movie recommendation, image summarization and quadratic programming. We believe that the monotonicity ratio is a natural parameter of submodular maximization problems, refining the binary distinction between monotone and non-monotone objective functions and improving the power of submodular maximization tools in machine learning applications. Thus, we hope to see future work towards understanding the optimal dependence on $m$ of the approximation ratios of various submodular maximization problems.

An important take-home message from our work is that, at least in the unconstrained submodular maximization case, the optimal algorithm has an approximation ratio whose dependence on $m$ is non-linear. Such algorithms are rarely obtained using current techniques, which might be one of the reasons why these techniques have so far failed to obtain tight approximation guarantees for constrained non-monotone submodular maximization.

%% file: SymmetryGap.tex
\section{Proof of \texorpdfstring{Theorem~\ref{thm:symmetry_gap}}{Theorem~\ref*{thm:symmetry_gap}}} \label{app:symmetry_gap}

In this section, we show how the proof of the symmetry gap technique due to Vondr\'{a}k~\cite{vondrak2013symmetry} can be adapted to prove Theorem~\ref{thm:symmetry_gap}. Let us begin the section by restating the theorem itself.

\thmSymmetryGap*

The crux of the symmetry gap technique is two lemmata due to~\cite{vondrak2013symmetry} that we restate below. Lemma~\ref{lem:original_continuous_versions} shows that given a non-negative set function $f$, one can obtain from it two continuous versions: a continuous version $\hat{F}$ that resembles $f$ itself, and a continuous version $\hat{G}$ that resembles a symmetrized version of $f$. Distinguishing between $\hat{F}$ and $\hat{G}$ is difficult, however, this does not translate into an hardness for discrete problems since $\hat{F}$ and $\hat{G}$ are continuous. Therefore, Vondr\'{a}k~\cite{vondrak2013symmetry} proved also Lemma~\ref{lem:original_continuous_to_discrete}, which shows how these continuous functions can be translated back into set functions with appropriate properties.

\begin{lemma}[Lemma~3.2 of~\cite{vondrak2013symmetry}] \label{lem:original_continuous_versions}
Consider a function $f\colon 2^\cN \to \nnR$ invariant under a group of
permutations $\cG$ on the ground set $\cN$. Let $F(\vx)$ be the multilinear extension of $F$, define $\bar{x} = \bE_{\sigma \in \cG}[\characteristic_{\sigma(\vx)}]$ and fix any $\eps > 0$. Then, there is $\delta > 0$ and functions $\hat{F}, \hat{G} \colon [0, 1]^{\cN} \to \nnR$ (which are also symmetric with respect to $\cG$), satisfying the following:
\begin{compactenum}
	\item For all $\vx \in [0, 1]^\cN$, $\hat{G}(\vx) = \hat{F}(\bar{\vx})$.
	\item For all $\vx \in [0, 1]^\cN$, $|\hat{F}(\vx) - F(\vx)| \leq \eps$.
	\item Whenever $\|\vx - \bar{\vx}\|_2 \leq \delta$, $\hat{F}(\vx) = \hat{G}(\vx)$ and the value depends only on $\bar{\vx}$.
	\item The first partial derivatives of $\hat{F}$ and $\hat{G}$ are absolutely continuous.
	\item If $f$ is monotone, then, for every element $u \in \cN$, $\frac{\partial\hat{F}}{\partial x_u} \geq 0$ and $\frac{\partial\hat{G}}{\partial x_u} \geq 0$ everywhere.
	\item If $f$ is submodular then, for every two elements $u,v \in \cN$, $\frac{\partial^2\hat{F}}{\partial x_u \partial x_v} \leq 0$ and $\frac{\partial^2\hat{G}}{\partial x_u \partial x_v} \leq 0$ almost everywhere.
\end{compactenum}
\end{lemma}

\begin{lemma}[Lemma~3.1 of~\cite{vondrak2013symmetry}] \label{lem:original_continuous_to_discrete}
Let $n$ be a positive integer, and let $F \colon [0, 1]^\cN \to \bR$ and $X = [n]$. If we define $f\colon 2^{\cN \times X} \to \nnR$ so
that $f(S) = F(\vx)$, where $x_u = \frac{1}{n}|S \cap (\{u\} \times X)|$. Then,
\begin{compactenum}
	\item if $\frac{\partial F}{\partial x_u} \geq 0$ everywhere for each element $u \in \cN$, then $f$ is monotone,
	\item and if the first partial derivatives of $F$ are absolutely continuous and $\frac{\partial^2 F}{\partial x_u \partial x_v} \leq 0$ almost everywhere for all elements $u, v \in \cN$, then $f$ is submodular.
\end{compactenum}
\end{lemma}

One can note that the above lemmata have the property that if the function $f$ plugged into Lemma~\ref{lem:original_continuous_versions} is monotone, then the discrete functions obtained by applying Lemma~\ref{lem:original_continuous_to_discrete} to the functions $\hat{F}$ and $\hat{G}$ are also monotone. This is the reason that the framework of~\cite{vondrak2013symmetry} applies to monotone functions (as well as general, not necessarily monotone, functions). Therefore, to get the proof of~\cite{vondrak2013symmetry} to yield Theorem~\ref{thm:symmetry_gap}, it suffices to prove the following two modified versions of Lemmata~\ref{lem:original_continuous_versions} and~\ref{lem:original_continuous_to_discrete}. These modified versions preserve $m$-monotonicity for any $m \in [0, 1]$, rather than just standard monotonicity.

\begin{lemma}[modified version of Lemma~\ref{lem:original_continuous_versions}] \label{lem:continuous_versions}
Consider a function $f\colon 2^\cN \to \nnR$ that is $m$-monotone and invariant under a group of permutations $\cG$ on the ground set $\cN$. Let $F(\vx)$ be the multilinear extension of $F$, define $\bar{x} = \bE_{\sigma \in \cG}[\characteristic_{\sigma(\vx)}]$ and fix any $\eps > 0$. Then, there is $\delta > 0$ and functions $\hat{F}, \hat{G} \colon [0, 1]^{\cN} \to \nnR$ (which are also symmetric with respect to $\cG$), satisfying the following:
\begin{compactenum}
	\item For all $\vx \in [0, 1]^\cN$, $\hat{G}(\vx) = \hat{F}(\bar{\vx})$.
	\item For all $\vx \in [0, 1]^\cN$, $|\hat{F}(\vx) - F(\vx)| \leq \eps$.
	\item Whenever $\|\vx - \bar{\vx}\|_2 \leq \delta$, $\hat{F}(\vx) = \hat{G}(\vx)$ and the value depends only on $\bar{\vx}$.
	\item For every two vectors $\vx, \vy \in [0, 1]^\cN$ obeying $\vx \leq \vy$, $m \cdot F(\vx) \leq F(\vy)$. \label{item:continuous_monotonicity_ratio}
	\item If $f$ is submodular then, for every two elements $u,v \in \cN$, $\frac{\partial^2\hat{F}}{\partial x_u \partial x_v} \leq 0$ and $\frac{\partial^2\hat{G}}{\partial x_u \partial x_v} \leq 0$ almost everywhere.
\end{compactenum}
\end{lemma}

\begin{lemma}[modified version of Lemma~\ref{lem:original_continuous_to_discrete}] \label{lem:continuous_to_discrete}
Let $n$ be a positive integer, and let $F \colon [0, 1]^\cN \to \bR$ and $X = [n]$. If we define $f\colon 2^{\cN \times X} \to \nnR$ so
that $f(S) = F(\vx)$, where $x_u = \frac{1}{n}|S \cap (\{u\} \times X)|$. Then,
\begin{compactenum}
	\item if for some value $m \in [0, 1]$ the inequality $m \cdot F(\vx) \leq F(\vy)$ holds for any two vectors $\vx, \vy \in [0, 1]^\cN$ that obey $\vx \leq \vy$, then $f$ is $m$-monotone,
	\item and if the first partial derivatives of $F$ are absolutely continuous and $\frac{\partial^2 F}{\partial x_u \partial x_v} \leq 0$ almost everywhere for all elements $u, v \in \cN$, then $f$ is submodular.
\end{compactenum}
\end{lemma}

The proof of Lemma~\ref{lem:continuous_versions} is quite long and appears below. However, before getting to this proof, let first give the much simpler proof of Lemma~\ref{lem:continuous_to_discrete}.

\begin{proof}[Proof of Lemma~\ref{lem:continuous_to_discrete}]
The second point in Lemma~\ref{lem:continuous_to_discrete} follows immediately from Lemma~\ref{lem:original_continuous_to_discrete}, so we concentrate on proving the first point. In other words, we assume that $m \cdot F(\vx) \leq F(\vy)$ for every two vectors $\vx, \vy \in [0, 1]^\cN$ obeying $\vx \leq \vy$, and we need to show that $m \cdot f(S) \leq f(T)$ for every two sets $S \subseteq T \subseteq \cN$.

Let us define two vectors $\vx^{(S)}, \vx^{(T)} \subseteq [0, 1]^\cN$ as follows. For every $u \in \cN$,
\[
	x^{(S)}_u = \frac{1}{n}|S \cap (\{u\} \times X)|
	\quad
	\text{and}
	\quad
	x^{(T)}_u = \frac{1}{n}|T \cap (\{u\} \times X)|
	\enspace.
\]
Since $S \subseteq T$, we get $\vx^{(S)} \leq \vx^{(T)}$, which implies $m \cdot F(\vx^{(S)}) \leq F(\vx^{(T)})$; and the last inequality proves the lemma since $f(S) = F(\vx^{(S)})$ and $f(T) = F(\vx^{(T)})$ by the definition of $f$.
\end{proof}

We now get to the proof of Lemma~\ref{lem:continuous_versions}. We use in this proof functions $\hat{F}$ and $\hat{G}$ that are similar to the ones constructed by Vondr\'{a}k~\cite{vondrak2013symmetry} in the proof of Lemma~\ref{lem:original_continuous_versions}. Specifically, like in the proof of~\cite{vondrak2013symmetry}, we define 
\[
	\hat{G}(\vx) = G(\vx) + 256M |\cN|\alpha J(\vx)
	\enspace,
\]
where $M$ is the maximum value that the function $f$ takes on any set, $G$ is a symmetrized version of the multilinear extension $F$ of $f$ defined as $G(\vx) = F(\bar{\vx})$, $J(\vx) \triangleq |\cN|^2 + 3|\cN| \cdot \sum_{u \in \cN} x_u - \left(\sum_{u \in \cN} x_u\right)^2$, and $\alpha$ is a positive value that is independent of $\vx$. Similarly, the function $\hat{F}$ was defined by Vondr\'{a}k~\cite{vondrak2013symmetry} as
\[
	\hat{F}(\vx) = \tilde{F}(\vx) + 256M |\cN|\alpha J(\vx)
	\enspace,
\]
where the function $\tilde{F}$ interpolates between the multilinear extension $F$ of $f$ and its symmetrized version $G$, and is given by
\[
	\tilde{F}(\vx) = (1 - \phi(D(\vx))) \cdot F(\vx) + \phi(D(\vx)) \cdot G(\vx)
	\enspace.
\]
Here, $D(\vx) \triangleq \|\vx - \bar{\vx}\|_2^2$, and $\phi \colon \nnR \to [0, 1]$ is a function which is defined using the following lemma.
\begin{lemma}[Lemma~3.7 of~\cite{vondrak2013symmetry}] \label{lem:phi_properties}
For any $\alpha, \beta > 0$, there is $\delta > (0, \beta)$ and a function $\phi\colon \nnR \to [0, 1]$ with an absolutely continuous first derivative such that
\begin{compactitem}
	\item For $t \in [0, \delta]$, $\phi(t) = 1$.
	\item For $t \geq \beta$, $\phi(t) < e^{-1/\alpha}$.
	\item For all $t \geq 0$, $|t\phi'(t)| \leq 4\alpha$.
	\item For almost all $t \geq 0$, $|t^2\phi''(t)| \leq 10\alpha$.
\end{compactitem}
\end{lemma}

Vondr\'{a}k~\cite{vondrak2013symmetry} proved that the above functions $\hat{F}$ and $\hat{G}$ have all the properties guaranteed by Lemma~\ref{lem:original_continuous_versions} for the $\delta$ whose existence is guaranteed by Lemma~\ref{lem:phi_properties} when the values of $\alpha$ and $\beta$ are set to be $\alpha = \frac{\eps}{2000M|\cN|^3}$ and $\beta = \frac{\eps}{16M|\cN|}$.
Moreover, the proof of~\cite{vondrak2013symmetry} continues to work as long as $\alpha \leq \frac{\eps}{2000M|\cN|^3}$ and $\beta \leq \frac{\eps}{16M|\cN|}$. Therefore, we assume below that $\alpha = \min\{1, \frac{\eps}{2000M|\cN|^3}\}$ and $\beta = \min\{\alpha^2, \frac{\eps}{16M|\cN|}\}$, and we prove only the part of Lemma~\ref{lem:continuous_versions} that is not stated in the guarantees of Lemma~\ref{lem:original_continuous_versions}, which is Property~\ref{item:continuous_monotonicity_ratio} of the lemma. We begin by showing that the function $\hat{G}$ indeed has this property.
\begin{lemma} \label{lem:G_obeys_property}
For every two vectors $\vx, \vy \in [0, 1]^\cN$ obeying $\vx \leq \vy$, $m \cdot \hat{G}(\vx) \leq \hat{G}(\vy)$.
\end{lemma}
\begin{proof}
Consider the random sets $\RSet(\bar{\vx})$ and $\RSet(\bar{\vy})$. Since $\bar{\vx} \leq \bar{\vy}$, the set $\RSet(\bar{\vy})$ stochastically dominates $\RSet(\bar{\vx})$. In other words, one can correlate the randomness of these sets in a way that does not alter their distributions, but guarantees that the inclusion $\RSet(\bar{\vx}) \subseteq \RSet(\bar{\vy})$ holds deterministically. Assuming this done, we get
\begin{equation} \label{eq:G_inequality}
	m \cdot G(\vx)
	=
	m \cdot F(\bar{\vx})
	=
	m \cdot \bE[f(\RSet(\bar{\vx}))]
	\leq
	\bE[f(\RSet(\bar{\vy}))]
	=
	F(\bar{\vy})
	=
	G(\vy)
	\enspace,
\end{equation}
where the inequality follows from the linearity of the expectation and the $m$-monotonicity of $f$.

Observe now that for every element $u \in \cN$, the partial derivative of $J$ with respect to $z_u$ at any point $\vz \in [0, 1]^\cN$ is
\[
	\frac{\partial J(\vz)}{\partial z_u}
	=
	3|\cN| - 2 \sum_{v \in \cN} z_v
	\geq
	|\cN|
	\geq
	0
	\enspace.
\]
Hence, the inequality $\vx \leq \vy$ implies $m \cdot J(\vx) \leq J(\vx) \leq J(\vy)$. Together with Inequality~\eqref{eq:G_inequality}, this implies the lemma.
\end{proof}

One can observe that the arguments used to prove Inequality\eqref{eq:G_inequality} in the proof of the last lemma also show that $m \cdot F(\vx) \leq F(\vy)$, which is a fact that we use below. However, proving that $\hat{F}$ also has this property (and therefore, obeys Property~\ref{item:continuous_monotonicity_ratio} of Lemma~\ref{lem:continuous_versions}) is more involved. As a first step towards this goal, we bound the gradient of
\[
	\tilde{F}(\vx) - F(\vx)
	=
	\phi(D(\vx)) \cdot [G(\vx) - F(\vx)]
	\enspace.
\]
The following lemma does that in the regime in which $D(\vx)$ is small, and the next lemma handles the other regime.
\begin{lemma} \label{lem:derivative_small_D}
For every element $u \in \cN$ and vector $\vx \in [0, 1]^\cN$ obeying $D(\vx) \leq \beta$, the absolute value of the partial derivative $\frac{\partial \{\phi(D(\vx)) \cdot [G(\vx) - F(\vx)]\}}{\partial x_u}$ is at most $72\sqrt{\beta}M|\cN| \leq 72\alpha M|\cN|$.
\end{lemma}
\begin{proof}
Observe that
\[
	\frac{\partial \{\phi(D(\vx)) \cdot [G(\vx) - F(\vx)]\}}{\partial x_u}
	=
	\phi'(D(\vx)) \cdot \frac{\partial D(\vx)}{\partial x_u} \cdot [G(\vx) - F(\vx)] + \phi(D(\vx)) \cdot \left[\frac{\partial G(\vx)}{\partial x_u} - \frac{\partial F(\vx)}{\partial x_u} \right]
	\enspace.
\]
To use this equation to bound the absolute value of the left hand side, we need to make some observations. First, Lemma~3.6 of~\cite{vondrak2013symmetry} shows that $\|\nabla D(\vx)\|_2 = 2\sqrt{D(\vx)}$, which implies
\[
	\frac{\partial D(\vx)}{\partial x_u}
	\leq
	\|\nabla D(\vx)\|_2
	=
	2\sqrt{D(\vx)}
	\enspace.
\]
Additionally, Lemma~3.5 of~\cite{vondrak2013symmetry} shows that
$
	|G(\vx) - F(\vx)| \leq 8M|\cN| \cdot D(\vx)
$,
and therefore,
\begin{align*}
	\left|\phi'(D(\vx)) \cdot \frac{\partial D(\vx)}{\partial x_u} \cdot [G(\vx) - F(\vx)]\right|
	\leq{} &
	|\phi'(D(\vx))| \cdot \left|\frac{\partial D(\vx)}{\partial x_u}\right| \cdot |G(\vx) - F(\vx)|\\
	\leq{} &
	|\phi'(D(\vx))| \cdot 2\sqrt{D(\vx)} \cdot 8M|\cN| \cdot D(\vx)\\
	={} &
	|D(\vx) \cdot \phi'(D(\vx))| \cdot 16M|\cN| \cdot \sqrt{D(\vx)}
	\leq
	64 \alpha \sqrt{\beta} M |\cN|
	\enspace,
\end{align*}
where the second inequality follows from Lemma~\ref{lem:phi_properties} and our assumption that $D(\vx) \leq \beta$.

We now observe that
\begin{align*}
	\left|\phi(D(\vx)) \cdot \left[\frac{\partial G(\vx)}{\partial x_u} - \frac{\partial F(\vx)}{\partial x_u} \right]\right|
	={} &
	\phi(D(\vx)) \cdot \left|\frac{\partial G(\vx)}{\partial x_u} - \frac{\partial F(\vx)}{\partial x_u} \right|\\
	\leq
	\phi(D(\vx)) \cdot\|\nabla G(\vx) - \nabla F(\vx)\|_2{} &
	\leq
	\phi(D(\vx)) \cdot8M|\cN| \cdot \sqrt{D(\vx)}
	\leq
	8\sqrt{\beta} M|\cN|
	\enspace,
\end{align*}
where the second inequality holds since Lemma~3.5 of~\cite{vondrak2013symmetry} shows that $\|\nabla G(\vx) - F(\vx)\|_2 \leq 8M|\cN| \cdot \sqrt{D(\vx)}$; and the last inequality holds by our assumption that $D(\vx) \leq \beta$ and by recalling that the range of $\phi$ is $[0, 1]$.

Combining all the above yields
\[
	\left|\frac{\partial \{\phi(D(\vx)) \cdot [G(\vx) - F(\vx)]\}}{\partial x_u}\right|
	\leq
	64 \alpha\sqrt{\beta} M |\cN| + 8\sqrt{\beta} M|\cN|
	\leq
	72\sqrt{\beta} M |\cN|
	\enspace,
\]
where the second inequality holds since $\alpha \leq 1$.
\end{proof}

\begin{lemma}
For every element $u \in \cN$ and vector $\vx \in [0, 1]^\cN$ obeying $D(\vx) \geq \beta$, the absolute value of the partial derivative $\frac{\partial \{\phi(D(\vx)) \cdot [G(\vx) - F(\vx)]\}}{\partial x_u}$ is at most $72\alpha M|\cN|^{3/2}$.
\end{lemma}
\begin{proof}
Repeating the proof of Lemma~\ref{lem:derivative_small_D}, except for the use of the inequality $D(\vx) \leq \beta$ (which does not hold in the current lemma) and the inequality $\phi(\vx) \leq 1$ (which too weak for our current purpose), we get
\[
	\left|\phi(D(\vx)) \cdot \left[\frac{\partial G(\vx)}{\partial x_u} - \frac{\partial F(\vx)}{\partial x_u} \right]\right|
	\leq
	64\alpha M|\cN| \cdot \sqrt{D(\vx)} + |\phi(D(\vx))| \cdot8M|\cN| \cdot \sqrt{D(\vx)}
	\enspace.
\]
The expression $\phi(D(\vx))$ can be upper bounded by $e^{-1/\alpha} \leq \alpha$ by Lemma~\ref{lem:phi_properties}. Also, $D(\vx) = \|\vx - \bar{\vx}\|_2^2 \leq |\cN|$. The lemma now follows by plugging these two upper bounds into the previous inequality.
\end{proof}

\begin{corollary} \label{cor:derivatives_bound}
For every element $u \in \cN$ and vector $\vx \in [0, 1]^\cN$, the absolute value of the partial derivative $\frac{\partial \{\phi(D(\vx)) \cdot [G(\vx) - F(\vx)]\}}{\partial x_u}$ is at most $72\alpha M|\cN|^{3/2}$.
\end{corollary}

The last corollary implies that $\tilde{F}$ can be presented as the sum of $F$ and a component that changes slowly. Therefore, if we add to $\tilde{F}$ a function that increases quickly enough (as is done to define $\hat{F}$), then we should get a function that can be represented as $F$ plus a monotone component. This is the intuition formalized in the proof of the next lemma.
\begin{lemma} \label{lem:positive_partial}
The function $\hat{F}(\vx) - F(\vx)$ has non-negative partial derivatives for every $\vx \in [0, 1]^\cN$.
\end{lemma}
\begin{proof}
By the definition of $\hat{F}(\vx)$,
\[
	\hat{F}(\vx) - F(\vx)
	=
	\tilde{F}(\vx) - F(\vx) + 256M |\cN|\alpha J(\vx)
	\enspace.
\]
By Corollary~\ref{cor:derivatives_bound} and the observation that all the partial derivatives of $J(\vx)$ are at least $|\cN|$ (see the proof of Lemma~\ref{lem:G_obeys_property}), the last equality implies, for every element $u \in \cN$,
\[
	\frac{\partial[\hat{F}(\vx) - F(\vx)]}{\partial x_u}
	\geq
	-72\alpha M|\cN|^{3/2} + 256\alpha M|\cN|^2
	\geq
	0
	\enspace.
	\qedhere
\]
\end{proof}

We are now ready to show that $\hat{F}$ obeys Property~\ref{item:continuous_monotonicity_ratio} of Lemma~\ref{lem:continuous_versions}.
\begin{lemma}
For every two vectors $\vx, \vy \in [0, 1]^\cN$ obeying $\vx \leq \vy$, $m \cdot \hat{F}(\vx) \leq \hat{F}(\vy)$.
\end{lemma}
\begin{proof}
Note that $\overline{\characteristic_\varnothing} = \characteristic_\varnothing$, which implies that $G(\characteristic_\varnothing) = F(\characteristic_\varnothing)$, and therefore,
\[
	\tilde{F}(\characteristic_\varnothing) - F(\characteristic_\varnothing)
	=
	\phi(D(\characteristic_\varnothing)) \cdot [G(\characteristic_\varnothing) - F(\characteristic_\varnothing)]
	=
	0
	\enspace.
\]
Plugging this observation into the definition of $\hat{F}$ now gives
\[
	\hat{F}(\characteristic_\varnothing) - F(\characteristic_\varnothing)
	=
	\tilde{F}(\characteristic_\varnothing) - F(\characteristic_\varnothing) + 256M|\cN|\alpha J(\characteristic_\varnothing)
	=
	256M|\cN|\alpha J(\characteristic_\varnothing)
	\enspace.
\]

Since all the first partial derivatives of $\hat{F}(\vz) - F(\vz)$ are non-negative by Lemma~\ref{lem:positive_partial}, the last inequality implies
\[
	\hat{F}(\vy) - F(\vy)
	\geq
	\hat{F}(\vx) - F(\vx)
	\geq
	256M|\cN|\alpha J(\vx)
	\geq
	0
	\enspace.
\]
Hence,
\[
	m \cdot \hat{F}(\vx)
	\leq
	m \cdot [F(\vx) + \hat{F}(\vy) - F(\vy)]
	\leq
	F(\vy) + m \cdot [\hat{F}(\vy) - F(\vy)]
	\leq
	F(\vy) + [\hat{F}(\vy) - F(\vy)]
	=
	\hat{F}(\vy)
	\mspace{-2mu}\enspace,
\]
where the second inequality holds by the discussion immediately after the proof of Lemma~\ref{lem:G_obeys_property}, and the last inequality holds since $m \leq 1$ and $\hat{F}(\vy) - F(\vy) \geq 0$.
\end{proof}

%% file: Experiments_Appendix.tex
\section{Additional Plots for Section~\ref{sec:experiments}}
\label{sec:extra}
As discussed in Section~\ref{sec:experiments}, the various algorithms we use in the context of a cardinality constraint have very similar empirical performance. Figure~\ref{moviebench} presents the performance of all these algorithms in the movie recommendation setting with the number of movies in the summery varying. One can observe that the lines of the three non-scarecrow algorithms almost overlap. Figure~\ref{imagebench} presents the performance of the non-scarecrow algorithms in the image summarization setting. In this figure we had to ignore the scarecrow algorithm Random because otherwise the lines of the three non-scarecrows algorithms are indistinguishable. Furthermore, we had to zoom in on a very small range of $y$-axis values. Despite these steps, the lines of Sample Greedy and Threshold Greedy still completely overlap, but the large zoom allows us to see that Threshold Random Greedy is marginally worse.

\begin{figure}[tb]
  \begin{subfigure}[t]{0.45\textwidth}
    \includegraphics[width=\textwidth]{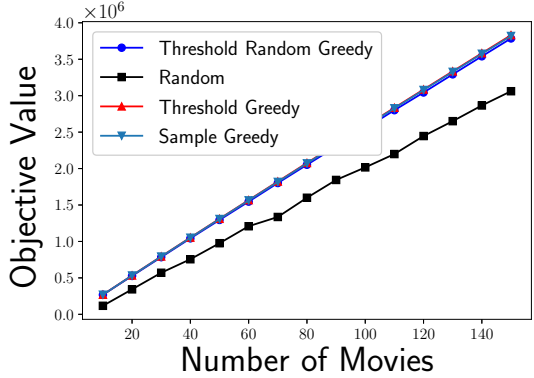}
    \caption{Performance of the various algorithms in the movie recommendation setting (for $\lambda = 0.75$).}
    \label{moviebench}
  \end{subfigure}
	\hfill
  \begin{subfigure}[t]{0.45\textwidth}
    \includegraphics[width=\textwidth]{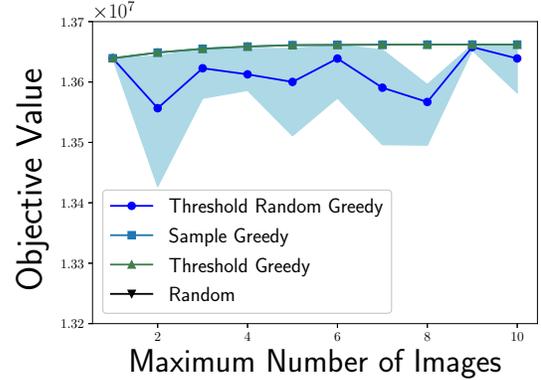}
    \caption{Performance of the non-scarecrow algorithms in the image summarization setting with a cardinality constraint. The shaded area represents the standard deviation of Threshold Random Greedy.}
    \label{imagebench}
  \end{subfigure}
	\caption{Comparing the performance of algorithms for a cardinality constraint in our experiments.}
 \end{figure}

%% file: DR-submodular.tex
\section{Maximizating DR-submodular Functions subject to a Polytope Constraint} \label{sec:DR-sub}

There are (at least) two natural ways in which the notion of submodularity can be extended from set functions to continuous functions. The more restrictive of these is known as DR-submodularity (first defined by~\cite{bian2017guarantees}). Given a domain $\cX = \prod_{i=1}^n\mathcal{X}_i$, where $\mathcal{X}_i$ is a closed range in $\bR$ for every $i \in [n]$, a function $F \colon \mathcal{X}\rightarrow\mathbb{R}$ is \emph{DR-submodular} if for every two vectors $\va,\vb\in \cX$, positive value $k$ and coordinate $i\in[n]$ the inequality
\[
	F(\va+k\ve_i)-F(\va)\geq F(\vb+k\ve_i)-F(\vb)
\]
holds whenever $\va\leq \vb$ and $\vb+k\ve_i\in \mathcal{X}$ (here and throughout the section $\ve_i$ denotes the standard $i$-th basis vector, and comparison between two vectors should be understood to hold coordinate-wise). If $F$ is continuously differentiable, then the above definition of DR-submodulrity is equivalent to $\nabla F$ being an antitone mapping from $\cX$ to $\mathbb{R}^n$ (i.e., $\nabla F(\va)\geq \nabla F(\vb)$ for every two vectors $\va,\vb\in\mathcal{X}$ that obey $\va\leq\vb$). Moreover, when $F$ is twice differentiable, it is DR-submodular if and only if its Hessian is non-positive at every vector $\vx\in\mathcal{X}$.

In this section we consider the problem of maximizing a non-negative DR-submodular function $F\colon 2^\cN \to \nnR$ subject to a solvable down-closed\footnote{In Section~\ref{ssc:continuous_greedy}, down-closeness of was defined for the special case of $P \subseteq [0, 1]^\cN$. More generally, a body $P \subseteq \cX$ is down-closed if $\vb \in P$ implies $\va \in P$ for every vector $\va \in \cX$ obeying $\va \leq \vb$.} convex body $P \subseteq \cX$ (usually polytope) constraint. As is standard when dealing with problems of this kind, we assume that $F$ is $L$-smooth, i.e., for every two vectors $\vx, \vy \in \cX$ it obeys
\[
    \|\nabla F(\vx)-\nabla F(\vy)\|_2 \leq L\|\vx-\vy\|_2
\]
for some non-negative parameter $L$. Additionally, for simplicity, we assume that $\mathcal{X} = [0,1]^n$. This assumption is without loss of generality because the natural mapping from $\mathcal{X}$ to $[0,1]^n$ preserves all our results.

We analyze a variant of the Frank-Wolfe algorithm for the above problem due to~\cite{bian2017nonmonotone} called Non-monotone Frank-Wolfe. This variant was motivated by the Measured Continuous Greedy algorithm studied in Section~\ref{ssc:continuous_greedy}, and its assumes access to the first order derivatives of $F$. The details of the algorithm we consider appear as Algorithm~\ref{alg:FW}. This algorithm gets a quality control parameter $\varepsilon \in (0, 1)$, and it is assumed that $\varepsilon^{-1}$ is an integer (if this is not the case, one can fix that by reducing $\eps$ by at most a factor $2$). Algorithm~\ref{alg:FW} and its analysis also employ the notation defined in Section~\ref{ssc:continuous_greedy}, namely, given two vectors $\vx, \vy$, their coordinate-wise multiplication is denoted by $\vx \odot \vy$. Additionally, we denote by $\vzero$ and $\vone$ the all zeros and all ones vectors, respectively.

\begin{algorithm}
\DontPrintSemicolon
Let  $\vy^{(0)}\leftarrow \vzero$ and $t=0$.\;
\While{$t \leq 1$}{
$\vs^{(t)}\leftarrow\argmax_{\vx\in P} \vx \cdot ((\vone-\vy^{(t)}) \odot \nabla F(\vy^{(t)}))$. \;
$\vy^{(t+\varepsilon)} \leftarrow \vy^{(t)}+\varepsilon\cdot(\bar{1}-\vy^{(t)})\odot\vs^{(t)}$. \;
$t\leftarrow t+\varepsilon.$ \;
}
\Return $\vy^{(1)}$.
\caption{Non-monotone Frank-Wolfe($\varepsilon$)\label{alg:FW}}
\end{algorithm}

To analyze Algorithm~\ref{alg:FW} we need to define two additional parameters. The first parameter is the diameter $D = \max_{x\in P}\|\vx\|_2$ of $P$, which is a standard parameter. The other parameter is the monotonicity ratio of $F$, which can be extended to the continuous setting we study in the following natural way.\footnote{In Section~\ref{ssc:quadratic_programming} we showed how the monotonicity ratio can be extended to the particular continuous setting studied in that section. The definition of Section~\ref{ssc:quadratic_programming} is obtained from the more general definition we give here by setting $\cX = \prod_{i = 1}^n [0, u_i]$.}
\[ m=\inf_{\substack{\vx,\vy\in \cX\\\vx\leq\vy}}\frac{F(\vy)}{F(\vx)}\enspace,\]
where the ratio $F(\vy) / F(\vx)$ should be understood to have a value of $1$ whenever $F(\vx) = 0$. Additionally, let us denote by $\vo$ an arbitrary optimal solution for the problem described above. Using these definitions, we are now ready to state the result that we prove for Algorithm~\ref{alg:FW}.

\begin{restatable}{theorem}{thmFW} \label{thm:FW}
When given a non-negative $m$-monotone DR-submodular function $F\colon \cX \to \nnR$ and a down-closed solvable convex body $P \subseteq \cX$, the Measured Greedy Frank-Wolfe algorithm (Algorithm~\ref{alg:FW}) outputs a solution $\vy \in P$ such that $F(\vy) \geq [m(1 - 1/e) + (1 - m) \cdot (1/e)] \cdot F(\vo) - \varepsilon LD^2$.
\end{restatable}

Our first objective towards proving Theorem~\ref{thm:FW} is to lower bound the expression $F(\vo  + \vy^{(t)} \cdot (\vone - \vo))$, which we do in the next two lemmata.
\begin{lemma}
\label{lemma:FW1}
For every integer $i\in [0,\varepsilon^{-1}]$, $\vy^{(\eps i)} \geq \vzero$ and $\|\vy^{(\eps i)}\|_\infty \leq 1 - (1-\varepsilon)^{-i}$.
\end{lemma}
\begin{proof}
We prove the lemma by induction on $i$. For $i=0$, the lemma follows directly from the initialization $\vy^{(0)}=\vzero$ because $1 - (1-\varepsilon)^{-0} = 0$. Assume now that the lemma holds for $i-1$, and let us prove it for an integer $0 < i\leq 1$.
Observe that, for every $j \in [n]$,
\[
y_j^{\eps i}= y_j^{\eps (i-1)} + \varepsilon\cdot\left(1-y_j^{\eps (i-1)}\right)\cdot s_j^{\eps(i - 1)} \geq y_j^{\eps(i - 1)} \geq 0\enspace,
\]
where the first inequality holds since $y_j^{\eps(i-1)} \leq 1$ by the induction hypothesis and the value of $s_j^{(\eps(i - 1))}$ is non-negative by definition. Moreover,
\begin{align*}
    y_j^{\eps i}&=y_j^{\eps(i-1)}+\varepsilon\cdot\left(1-y_j^{\eps(i-1)}\right)\cdot s_j^{\eps(i - 1)}\leq y_j^{\eps(i-1)}+\varepsilon\cdot\left(1-y_j^{\eps(i-1)}\right)\\
    &=\varepsilon + (1-\varepsilon)\cdot y_j^{\eps(i-1)}\leq\varepsilon+(1-\varepsilon)\cdot\left[1-(1-\varepsilon)^{(i-1)}\right]=1-(1-\varepsilon)^i\enspace,
\end{align*}
where again the first inequality holds since $s^{(\eps (i - 1))} \in \cX$, which implies $s_j^i \leq 1$; and the second inequality holds by the induction hypothesis.
\end{proof}

\begin{lemma}\label{lemma:FW2}
For every integer $i\in [0,\varepsilon^{-1}]$, $F(\vo  + \vy^{(\eps i)} \cdot (\vone - \vo)) \geq\left[(1-(1-m)\left(1-(1-\varepsilon)^{i}\right)\right]\cdot F(\vo) = \left[m+(1-m)(1-\varepsilon)^{i}\right]\cdot F(\vo)$.
\begin{proof}
Observe that
\begin{align*}
    F(\vo  + \vy^{(\eps i)} \cdot (\vone - \vo))&\geq \left(1-\|\vy^{(\eps i)}\|_\infty\right)\cdot F(\vo)+ \|\vy^{(\eps i)}\|_\infty\cdot F\left(\vo +\frac{\vy^{(\eps i)} \cdot (\vone - \vo)}{\|\vy^{(\eps i)}\|_\infty}\right)\\
    &\geq \left(1-\|\vy^{(\eps i)}\|_\infty\right)\cdot F(\vo)+m\cdot \|\vy^{(\eps t)}\|_\infty\cdot F(\vo)\\
    &=\left(1-(1-m)\cdot\|\vy^{(\eps i)}\|_\infty\right)\cdot F(\vo)\enspace,
\end{align*}
where the first inequality holds since the DR-submodularity of $F$ implies that $F$ is concave along positive directions (such as the direction $\vy^{(\eps i)} \cdot (\vone - \vo) / \|\vy^{(\eps i)}\|_\infty$), and the second inequality holds since the monotonicity ratio of $F$ is at least $m$.
Plugging Lemma~\ref{lemma:FW1} into the previous inequality completes the proof of the lemma. 
\end{proof}
\end{lemma}

Using the previous lemma, we can now provide a lower bound on the increase in the value of $\vy^{(t)}$ as a function of $t$.
\begin{lemma}
\label{lemma:FW3}
For every integer $0 \leq i < \varepsilon^{-1}$, $F(\vy^{(\eps(i+1))})-F(\vy^{(\eps i)}) \geq \varepsilon\cdot[(m + (1-m)\cdot(1-\varepsilon)^i)\cdot F(\vo)-F(\vy^{(\eps i)})]-\varepsilon^2LD^2$.
\end{lemma}
\begin{proof}
By the chain rule,
\begin{align*}
    F(\vy^{\eps(i+1)})-F(\vy^{(\eps i)})&= F(\vy^{(\eps i)}+\varepsilon\cdot \vs^{(\eps i)}\odot(\vone-\vy^{(\eps i)}))-F(\vy^{(\eps i)})\\
    &=\int_{0}^{\varepsilon} \nabla F(\vy^{(\eps i)}+r\cdot \vs^{(\eps i)}\odot(\vone-\vy^{(\eps i)})) \cdot (\vs^{(\eps i)}\odot(\vone-\vy^{(\eps i)})) \,dr\\
    &\geq\int_{0}^{\varepsilon}\nabla F(\vy^{(\eps i)})\cdot (\vs^{(\varepsilon i)}\odot(\vone-\vy^{(\eps i)}))\,dr -\varepsilon^2LD^2\\
    &=\varepsilon\cdot\nabla F(\vy^{(\eps i)})\cdot (\vs^{(\eps i)} \odot(1-\vy^{(\eps i)}))-\varepsilon^2LD^2\enspace,
\end{align*}
where the first inequality holds by the $L$-smoothness of $F$. Furthermore,
\begin{align*}
	\nabla F(\vy^{(\eps i)})\cdot (\vs^{(\eps i)} \odot(\vone-\vy^{(\eps i)}))
	={} &
	((\vone-\vy^{(\eps i)}) \odot \nabla F(\vy^{(\eps i)})) \cdot \vs^{(\eps i)}
	\geq
  ((\vone-\vy^{(\eps i)}) \odot \nabla F(\vy^{(\eps i)})) \cdot \vo\\
	={} &
	\nabla F(\vy^{(\eps i)})) \cdot ((\vone-\vy^{(\eps i)}) \odot \vo)
	\geq
	F(\vo + \vy^{(\eps i)}(\vone - \vo))-F(\vy^{(\eps i)})\\
  \geq{} &
	\left[m+(1-m)\cdot(1-\varepsilon)^i\right]\cdot F(\vo)-F(\vy^{(\eps i)})
	\enspace,
\end{align*}
where the first inequality holds by the definition of $\vs^{(\eps i)}$ since $\vo$ is a candidate to be this vector, the second inequality follows from the concavity of $F$ along positive directions, and the last inequality holds by Lemma~\ref{lemma:FW2}. The lemma now follows by combining the two above inequalities.
\end{proof}

We are now ready to prove Theorem~\ref{thm:FW}.
\begin{proof}[Proof of Theorem~\ref{thm:FW}]
Rearranging the guarantee of Lemma~\ref{lemma:FW3}, we get
\[
    F(\vy^{\eps (i + 1)})
		\geq
		(1-\varepsilon)\cdot F(\vy^{(\eps i)})+\varepsilon[m + (1-m)\cdot (1-\varepsilon)^i] \cdot F(\vo) - \eps^2LD^2
		\enspace.
\]
Since this inequality applies for every integer $0 \leq i < \eps^{-1}$, we can use it repeatedly to obtain
\begin{align*}
    F(\vy^{(1)})
		&\geq \varepsilon \cdot \sum_{i=1}^{1 / \eps} (1-\varepsilon)^{1/\varepsilon-i} \cdot \left[(m + (1-m)\cdot(1-\varepsilon)^{i-1}) \cdot F(\vo) - \eps LD^2\right] + (1 - \eps)^{1/\eps} \cdot F(\vzero)\\
    &\geq
		m\varepsilon\cdot\sum_{i=1}^{\nicefrac{1}{\varepsilon}}(1-\varepsilon)^{\frac{1}{\varepsilon}-i}\cdot F(\vo)+\varepsilon(1-m)\cdot\sum_{i=1}^{\nicefrac{1}{\varepsilon}}[(1-\varepsilon)^{\nicefrac{1}{\varepsilon}-1}\cdot F(\vo) - \eps LD^2]\\
    &=m\varepsilon\cdot\frac{1-(1-\varepsilon)^{1 / \varepsilon}}{\varepsilon} \cdot F(\vo)+\varepsilon(1-m)\cdot\frac{(1-\varepsilon)^{1/\varepsilon-1} \cdot F(\vo) - \eps LD^2}{\varepsilon}\\
    &\geq \left[m(1-e^{-1})+(1-m)\cdot e^{-1}\right]\cdot F(\vo) - \eps LD^2
		\enspace,
\end{align*}
where the second inequality follows from the non-negativity of $F$, and the last inequality holds since $(1-\varepsilon)^{1/\eps} \leq e^{-1} \leq (1-\varepsilon)^{1/\eps - 1}$.\qedhere
\end{proof}